%% file: paper.tex
\documentclass{article}

\PassOptionsToPackage{numbers, compress}{natbib}

 \usepackage[final]{neurips_2022}

\usepackage[utf8]{inputenc} 
\usepackage[T1]{fontenc}    
\usepackage{hyperref}       
\usepackage{url}            
\usepackage{booktabs}       
\usepackage{amsfonts}       
\usepackage{nicefrac}       
\usepackage{microtype}      
\usepackage{xcolor}         

\usepackage{multirow}
\usepackage{amsmath}
\usepackage{amssymb}
\usepackage{mathtools}
\usepackage{amsthm}
\usepackage{algorithm,algpseudocode}
\usepackage{tabularx}
\usepackage[export]{adjustbox}
\usepackage{wrapfig}

\theoremstyle{plain}
\newtheorem{theorem}{Theorem}[section]
\newtheorem{proposition}[theorem]{Proposition}
\newtheorem{lemma}[theorem]{Lemma}
\newtheorem{corollary}[theorem]{Corollary}
\newtheorem{definition}[theorem]{Definition}
\newtheorem{assumption}[theorem]{Assumption}
\theoremstyle{remark}

\usepackage[textsize=tiny]{todonotes}

\usepackage{macros}

\title{On Elimination Strategies for Bandit Fixed-Confidence Identification}

\author{%
Andrea Tirinzoni\thanks{Work done while at Inria Lille.} \\
Meta AI\\
Paris, France\\
\texttt{tirinzoni@fb.com} \\
  \And
  R\'{e}my Degenne\\
  Univ. Lille, Inria, CNRS, Centrale Lille, UMR 9189 CRIStAL, F-59000 Lille, France\\
  \texttt{remy.degenne@inria.fr} \\
}

\usepackage{minitoc}

\begin{document}

\maketitle

\doparttoc 
\faketableofcontents 

\begin{abstract}
Elimination algorithms for bandit identification, which prune the plausible correct answers sequentially until only one remains, are computationally convenient since they reduce the problem size over time. However, existing elimination strategies are often not fully adaptive (they update their sampling rule infrequently) and are not easy to extend to combinatorial settings, where the set of answers is exponentially large in the problem dimension. 
On the other hand, most existing fully-adaptive strategies to tackle general identification problems are computationally demanding since they repeatedly test the correctness of every answer, without ever reducing the problem size.
We show that adaptive methods can be modified to use elimination in both their stopping and sampling rules, hence obtaining the best of these two worlds: the algorithms (1) remain fully adaptive, (2) suffer a sample complexity that is never worse of their non-elimination counterpart, and (3) provably eliminate certain wrong answers early. We confirm these benefits experimentally, where elimination improves significantly the computational complexity of adaptive methods on common tasks like best-arm identification in linear bandits.
\end{abstract}

\input{sections/introduction}

\input{sections/elimination}
\input{sections/sampling}

\input{sections/experiments}

\input{sections/conclusion}


\bibliographystyle{unsrt}
\bibliography{references}

\clearpage
\section*{Checklist}


\begin{enumerate}

\item For all authors...
\begin{enumerate}
  \item Do the main claims made in the abstract and introduction accurately reflect the paper's contributions and scope?
    \answerYes{}
  \item Did you describe the limitations of your work?
    \answerYes{}
  \item Did you discuss any potential negative societal impacts of your work?
    \answerNA{}
  \item Have you read the ethics review guidelines and ensured that your paper conforms to them?
    \answerYes{}
\end{enumerate}

\item If you are including theoretical results...
\begin{enumerate}
  \item Did you state the full set of assumptions of all theoretical results?
    \answerYes{}
        \item Did you include complete proofs of all theoretical results?
    \answerYes{In appendix \ref{app:proofs-elim-stopping} and \ref{app:proofs_of_sampling}.}
\end{enumerate}

\item If you ran experiments...
\begin{enumerate}
  \item Did you include the code, data, and instructions needed to reproduce the main experimental results (either in the supplemental material or as a URL)?
    \answerYes{}
  \item Did you specify all the training details (e.g., data splits, hyperparameters, how they were chosen)?
    \answerYes{In appendix \ref{app:experiments}.}
        \item Did you report error bars (e.g., with respect to the random seed after running experiments multiple times)?
    \answerYes{}
        \item Did you include the total amount of compute and the type of resources used (e.g., type of GPUs, internal cluster, or cloud provider)?
    \answerYes{}
\end{enumerate}

\item If you are using existing assets (e.g., code, data, models) or curating/releasing new assets...
\begin{enumerate}
  \item If your work uses existing assets, did you cite the creators?
    \answerNA{}
  \item Did you mention the license of the assets?
    \answerNA{}
  \item Did you include any new assets either in the supplemental material or as a URL?
    \answerNA{}
  \item Did you discuss whether and how consent was obtained from people whose data you're using/curating?
    \answerNA{}
  \item Did you discuss whether the data you are using/curating contains personally identifiable information or offensive content?
    \answerNA{}
\end{enumerate}

\item If you used crowdsourcing or conducted research with human subjects...
\begin{enumerate}
  \item Did you include the full text of instructions given to participants and screenshots, if applicable?
    \answerNA{}
  \item Did you describe any potential participant risks, with links to Institutional Review Board (IRB) approvals, if applicable?
    \answerNA{}
  \item Did you include the estimated hourly wage paid to participants and the total amount spent on participant compensation?
    \answerNA{}
\end{enumerate}

\end{enumerate}

\clearpage


\appendix

\part{Appendix}


\parttoc
\newpage

\input{sections/appendix}

\input{sections/appendix_exp_family}
\input{sections/appendx_elimination}

\input{sections/appendix_sampling}

\input{sections/appendix_verify_assumption}

\input{sections/appendix_experiments}

\input{sections/appendix_revision.tex}

\end{document}

%% file: sections/introduction.tex

\section{Introduction}
\label{sec:introduction}

The multi-armed bandit is a sequential decision-making task which is now extensively studied (see, e.g., \cite{lattimore2020bandit} for a recent review). In this problem, an algorithm interacts with its environment by sequentially ``pulling'' one among $K \in \mathbb{N}$ arms and observing a sample from a corresponding distribution. Among the possible objectives, we focus on \emph{fixed-confidence identification} \citep{even2002pac, jamieson2014best,garivier2016optimal,chen2015optimal}. In this setting, the algorithm successively collects samples until it decides to stop and return an answer to a given query about the distributions. Its task is to return the correct answer with at most a given probability of error $\delta$, and its secondary goal is to do so while stopping as early as possible. This problem is called ``fixed-confidence'' as opposed to ``fixed-budget'', where the goal is to minimize the error probability with at most a given number of samples \citep{bubeck2009pure,audibert2010best,gabillon2012best,carpentier2016tight,abbasi2018best}.


The most studied query is \emph{best arm identification} (BAI), where the aim is to return the arm whose distribution has highest mean. A variant is Top-m identification \citep{kalyanakrishnan2012pac}, where the goal is to find the $m$ arms with highest means. While these are the most common, other queries have been studied, including thresholding bandits \citep{carpentier2016tight}, minimum threshold \citep{kaufmann2018sequential}, and multiple correct answers \citep{degenne2019pure}.

Algorithms for fixed-confidence identification can be generally divided into two classes: those based on \emph{adaptive sampling} and those based on \emph{elimination}. Adaptive algorithms \citep[e.g.,][]{gabillon2012best,kalyanakrishnan2012pac,garivier2016optimal,xu2018fully} update their sampling strategy at each round and typically stop when they can simultaneously assess the correctness of every answer. They often enjoy strong theoretical guarantees.
For instance, some of them \citep{garivier2016optimal,degenne2019non,degenne2020gamification,wang2021fast} have been shown to be optimal as $\delta\rightarrow 0$. However, since they repeatedly test the correctness of every answer, they are often computationally demanding.
Elimination-based strategies \citep[e.g.,][]{even2002pac,kaufmann2013information,soare2014best, fiez2019sequential,tao2018best} maintain a set of ``active'' answers (those that are still likely to be the correct one) and stop when only one remains. They typically update their sampling rules and/or the active answers infrequently. This, together with the fact that eliminations reduce the problem size over time, makes them more computationally efficient but also yields large sample complexity in practice.
Moreover, while adaptive algorithms for general identification problems (i.e., with arbitrary queries) exist \citep{garivier2016optimal,degenne2019non,wang2021fast}, elimination-based strategies are not easy to design at such a level of generality. In particular, they are not easy to extend to structured combinatorial problems (such as Top-m), where the number of answers is exponential in the problem dimension.\footnote{An elimination strategy for specific unstructured combinatorial problems has been introduced in \citep{chen2014combinatorial}.}

In this paper, we design a novel elimination rule for general identification problems which we call \emph{selective} elimination. It can be easily combined with existing adaptive strategies, both in their stopping and sampling rules, making them achieve the best properties of the two classes mentioned above.
In particular, we prove that (1) selective elimination never suffers worse sample complexity than the original algorithm, and hence remain asymptotically optimal whenever the base algorithm is; (2) It provably discards some answers much earlier than the stopping time; (3) It improves the computational complexity of the original algorithm when some answers are eliminated early. Experimentally, we compare several existing algorithms for three identification problems (BAI, Top-m, and thresholding bandits) on two bandit structures (linear and unstructured). We find that, coherently across all experiments, existing adaptive strategies achieve significant gains in computation time and, to a smaller extent, in sample complexity when combined with selective elimination.

\subsection{Bandit fixed-confidence identification}
\label{sub:bandit_identification}

An algorithm interacts with an environment composed of $K>1$ \emph{arms}. At each time $t \in \mathbb{N}$, the algorithm picks an arm $k_t$ and observes $X_t^{k_t} \sim \nu_{k_t}$, where $\nu_{k_t}$ is the distribution of arm $k_t$. 
At a time $\tau$, the algorithm stops and returns an answer $\ihat$ from a finite set $\mathcal I$. 
Formally, let $\cF_t$ be the $\sigma$-algebra generated by the observations up to time $t$. An identification algorithm is composed of
\begin{enumerate}[nosep]
	\item \emph{Sampling rule}: the sequence $(k_t)_{t \in \mathbb{N}}$, where $k_t$ is $\mathcal F_{t-1}$-measurable.
	\item \emph{Stopping rule}: a stopping time $\tau$ with respect to $(\cF_t)_{t\in\mathbb{N}}$ and a random variable $\ihat \in \mathcal I$, i.e., the answer returned when stopping at time $\tau$.
\end{enumerate}
Note that, while it is common to decouple $\tau$ and $\ihat$, we group them to emphasize that the time at which an algorithm stops depends strongly on the answer it plans on returning.

We assume that the arm distributions depend on some unknown parameter $\theta\in\cM$, where $\cM \subseteq \mathbb{R}^d$ is the set of possible parameters, and write $\nu_k(\theta)$ for $k\in[K]$ to make this dependence explicit. For simplicity, we shall use $\theta$ to refer to the bandit problem $(\nu_k(\theta))_{k\in[K]}$.
This assumption allows us to include linear bandits in our analysis.
We let $i^\star: \cM \to \mathcal I$ be the function, known to the algorithm, which returns the unique correct answer for each problem. The algorithm is correct on $\theta$ if $\ihat = i^\star(\theta)$. 

\begin{definition}[$\delta$-correct algorithm]\label{def:delta-correct}
An algorithm is said to be $\delta$-correct on $\mathcal M \subseteq \mathbb{R}^d$ if for all $\theta \in \mathcal M$, $\tau < +\infty$ almost surely and
$
\mathbb{P}_\theta(\ihat \ne i^\star(\theta) )
\le \delta \: .
$
\end{definition}

We want to design algorithms that, given a value $\delta$, are $\delta$-correct on $\mathcal M$ and have minimal expected sample complexity $\mathbb{E}_\theta[\tau]$ for all $\theta\in\cM$. A lower bound on $\mathbb{E}_\theta[\tau]$ was proved in \cite{garivier2016optimal}. In order to present it, we introduce the concept of \emph{alternative} set to an answer $i\in\cI$: $\Lambda(i) := \{\lambda \in \mathcal M \mid i^\star(\lambda) \ne i\}$, the set of parameters for which the correct answer is not $i$. Let us denote by $\KL_k(\theta,\lambda)$ the Kullback-Leibler (KL) divergence between the distribution of arm $k$ under $\theta$ and $\lambda$.
Then the lower bound states that for any algorithm that is $\delta$-correct on $\cM$ and any problem $\theta\in\cM$,
\begin{align}
\mathbb{E}_\theta[\tau]
&\ge \log(1/(2.4\delta))/H^\star(\theta)
\: ,
\text{with }
H^\star(\theta)
:= \max_{\omega\in\Delta_K}\inf_{\lambda \in \Lambda(i^\star(\theta))} \sum_{k\in[K]}\omega^k\KL_k(\theta,\lambda)
\: .\label{eq:lower_bound}
\end{align}


\paragraph{Example: BAI in Gaussian linear bandits}

While our results apply to general queries, we illustrate all statements of this paper on the widely-studied task of BAI in Gaussian linear bandits \citep{soare2014best,xu2018fully,zaki2020explicit,degenne2020gamification,jedra2020optimal}. 
In this setting, each arm $k\in[K]$ has a Gaussian distribution $\mathcal N(\mu_k(\theta), 1)$ with mean $\mu_k(\theta) = \phi_k^\top \theta$, a linear function of the unknown parameter $\theta \in \mathbb{R}^d$ (and $\mathcal M = \mathbb{R}^d$) and of known arm features $\phi_k\in\mathbb{R}^d$. The set of answers is $\mathcal I = [K]$ and the correct answer is $i^\star(\theta) := \argmax_{k\in[K]}\phi_k^\top \theta$.

Finally, for $x \in \mathbb{R}^d$ and $A \in \mathbb{R}^{d \times d}$, we define $\Vert x \Vert_A := \sqrt{x^\top A x}$. For $\omega \in \mathbb{R}^K$, let $V_\omega := \sum_{k=1}^K \omega^k \phi_k \phi_k^\top$. With this notation, we have $\sum_{k\in[K]}\omega^k\KL_k(\theta,\lambda) = \frac{1}{2}\Vert \theta - \lambda \Vert_{V_\omega}^2$. 
%

\subsection{Log-likelihood ratio stopping rules}
\label{sub:stopping_rules}

Most existing adaptive algorithms use a log-likelihood ratio (LLR) test in order to decide when to stop. Informally, they check whether sufficient information has been collected to confidently discard at once all answers except one. Since such LLR tests are crucial for the design of our general elimination rules, we now describe their principle.
 
Given two parameters $\theta,\lambda \in \cM$, the LLR of observations $X_{[t]} = (X_1^{k_1}, \ldots, X_t^{k_t})$ between models $\theta$ and $\lambda$ is 
$
L_t(\theta, \lambda)
:= \log \frac{d \mathbb{P}_\theta}{d \mathbb{P}_\lambda}(X_{[t]})
= \sum_{s=1}^t \log \frac{d \mathbb{P}_\theta}{d \mathbb{P}_\lambda}(X_s^{k_s})
$~. Let $\hat{\theta}_t := \argmax_{\lambda\in\cM} \log  \mathbb{P}_\lambda(X_{[t]})$ be the maximum likelihood estimator of $\theta$ from $t$ observations. In Gaussian linear bandits, we have
$
L_t(\theta, \lambda)
= \frac{1}{2}\Vert \theta - \lambda \Vert^2_{V_{N_t}} + (\theta - \lambda)^\top V_{N_t} (\hat{\theta}_t - \theta)
$~, where $N_t^k := \sum_{s=1}^t \indi{k_s=k}$.
See Appendix~\ref{sec:exponential_families} for more details. $L_t(\theta, \lambda)$ is closely related to $\sum_{k=1}^K N_t^k \KL_k(\theta, \lambda)$, a quantity that appears frequently in our results. Indeed, the difference between these quantities is a martingale, which is a lower order term compared to them. The LLR stopping rule was introduced to the bandit literature in \cite{garivier2016optimal}. At each step $t \in \mathbb{N}$, the algorithm computes the infimum LLR to the alternative set of $i^\star(\hat{\theta}_t)$ and stops if it exceeds a threshold, i.e., if
\begin{align}\label{eq:llr-stop}
\inf_{\lambda \in \Lambda(i^\star(\hat{\theta}_t))} L_t(\hat{\theta}_t, \lambda)
\ge \beta_{t, \delta} \: ,
\end{align}
where the function $\beta_{t, \delta}$ can vary, notably based on the shape of the alternative sets. The recommendation rule is then $\ihat = i^\star(\hat{\theta}_t)$. Informally, the algorithm stops if it has enough information to exclude all points $\lambda$ for which the answer is not $i^\star(\hat{\theta}_t)$.
This stopping rule enforces $\delta$-correctness, provided that the sampling rule ensures $\tau < + \infty$ a.s. and that $\beta_{t,\delta}$ is properly chosen. The most popular choice is to ensure a concentration property of $L_t(\hat{\theta}_t, \theta)$. For example, if for all $\delta$, $\beta_{t,\delta}$ guarantees that
\begin{align}\label{eq:concentration-beta}
\mathbb{P}\left( \exists t \geq 1 : L_t(\hat{\theta}_t, \theta) \geq \beta_{t,\delta}\right) \leq \delta,
\end{align}
LLR stopping with that threshold returns a wrong answer with probability at most $\delta$. Such concentration bounds can be found in \citep{abbasi2011improved,magureanu2014lipschitz} for linear and unstructured bandits, respectively. This LLR stopping rule is used in many algorithms \citep{garivier2016optimal,xu2018fully,degenne2019non,degenne2020gamification,jedra2020optimal,wang2021fast}\footnote{LinGapE \citep{xu2018fully} does not use LLR stopping explicitly, but its stopping rule is equivalent to it. We can write it as: stop if for all points inside a confidence region a gap is small enough, that is if all those points do not belong to the alternative of $i^\star(\hat{\theta}_t)$. The contrapositive of that statement is exactly LLR stopping.}.
Some of them have been proven 
to be \emph{asymptotically optimal}: their sample complexity upper bound matches the lower bound~\eqref{eq:lower_bound} when $\delta \to 0$. However, improvements are still possible: their sample complexity for moderate $\delta$ may not be optimal and their computational complexity may be reduced, as we will see.

%% file: sections/elimination.tex

\section{Elimination stopping rules for adaptive algorithms}
\label{sec:elimination_stopping_rules}

We show how to modify the stopping rule of adaptive algorithms using LLR stopping to perform elimination. We assume that the alternatives sets $\Lambda(i)$ can be decomposed  into a union of sets which we refer to as \emph{alternative pieces} (or simply pieces), with the property that computing the infimum LLR over these sets is computationally easy.

\begin{assumption}\label{ass:union_of_sets}
For all $i \in \cI$, there exist pieces $(\Lambda_p(i))_{p \in \cP(i)}$, where $\cP(i)$ is a finite set of 
piece indexes, such that $\Lambda(i) = \bigcup_{p\in\cP(i)} \Lambda_p(i)$ and $\inf_{\lambda \in \Lambda_p(i)}L_t(\hat{\theta}_t,\lambda)$ can be efficiently computed for all $p \in \cP(i)$ and $t >0$.
\end{assumption}
This assumption is satisfied in many problems of interest, including BAI, Top-$m$ identification, and thresholding bandits (see Appendix \ref{app:problems}).
Indeed, in all applications we consider in this paper, the sets of Assumption~\ref{ass:union_of_sets} are half-spaces. In our linear BAI example, the piece indexes are simply arms. For $i,j\in[K]$ we can define $\Lambda_j(i) = \{\lambda\in\cM \mid \phi_j^\top \lambda > \phi_i^\top \lambda\}$. Then, $\Lambda(i) = \bigcup_{j\in [K]\setminus \{i\}} \Lambda_j(i)$. Moreover, the infimum LLR (and the corresponding minimizer) can be computed in closed form as \citep[e.g.,][]{fiez2019sequential}
$
\inf_{\lambda \in \Lambda_j(i)}L_t(\hat{\theta}_t,\lambda) = \max\{\hat{\theta}_t^T(\phi_i-\phi_j),0\}^2 / \| \phi_i - \phi_j \|_{V_{N_t}^{-1}}^2.
$

\paragraph{Elimination stopping}

The main idea is that it is not necessary to exclude all $\Lambda_p(i)$ for $p\in\cP(i)$ at the \emph{same time}, as LLR stopping \eqref{eq:llr-stop} does\footnote{Under Assumption \ref{ass:union_of_sets}, LLR stopping is written as $\min_{p\in\cP(i)}\inf_{\lambda\in\Lambda_p(i)}L_t(\hat{\theta}_t,\lambda) \geq \beta_{t,\delta}$ for $i=i^\star(\hat{\theta}_t)$, which implies that all alternative pieces of answer $i$ are discarded at once.}, in order to know that the algorithm can stop and return answer $i$. Instead, each piece can be pruned as soon as we have enough information to do so. 

\begin{definition}\label{def:elimination}
A set $S \subseteq \mathbb{R}^d$ is said to be eliminated at time $t$ if, for all $\lambda\in S$, $L_t(\hat{\theta}_t,\lambda) \ge \beta_{t,\delta}$.
\end{definition}
From the concentration property~\eqref{eq:concentration-beta}, we obtain that the probability that $\theta \in S$ and $S$ is eliminated is less than $\delta$.
LLR stopping interrupts the algorithm when the alternative set $\Lambda(i^\star(\hat{\theta}_t))$ can be eliminated. 
In elimination stopping, we eliminate smaller sets gradually, instead of the whole alternative at once.
Formally, let us define, for all $i\in\cI$,
\begin{align}
\label{eq:active-pieces-t-only}
\overline{\cP}_t(i;\beta_{t,\delta}) = \left\{ p \in \cP(i) : \inf_{\lambda \in \Lambda_p(i)} L_t(\hat{\theta}_t,\lambda) < \beta_{t,\delta} \right\}
\end{align}
as the subset of pieces for answer $i\in\cI$ whose infimum LLR at time $t$ is below a threshold $\beta_{t,\delta}$. That is, the indexes of pieces that are \emph{not} eliminated at time $t$. Moreover, we define, for all $i\in\cI$, a set of \emph{active pieces} $\cP_t^{\mathrm{stp}}(i)$ which is initialized as $\cP_0^{\mathrm{stp}}(i) = \cP(i)$ (all piece indexes).

Our \emph{selective elimination} rule updates, at each time $t$, only the active pieces of the empirical answer $i^\star(\hat{\theta}_t)$. That is, for $i=i^\star(\hat{\theta}_t)$, it sets
\begin{align}\label{eq:elimination-sets}
	\cP_t^{\mathrm{stp}}(i) := \cP_{t-1}^{\mathrm{stp}}(i) \cap \overline{\cP}_t(i;\beta_{t,\delta}),
\end{align}
while it sets $\cP_t^{\mathrm{stp}}(i) := \cP_{t-1}^{\mathrm{stp}}(i)$ for all $i\neq i^\star(\hat{\theta}_t)$. One might be wondering why not updating all answers at each round. The main reason is computational: as we better discuss at the end of this section, checking LLR stopping requires one minimization for \emph{each} piece $p\in\cP(i^\star(\hat{\theta}_t))$, while selective elimination requires only one for each \emph{active} piece $p\in\cP_{t-1}^{\mathrm{stp}}(i^\star(\hat{\theta}_t))$. Thus, the latter becomes increasingly more computationally efficient as pieces are eliminated. For completeness, we also analyze the variant, that we call \emph{full elimination}, which updates the active pieces according to \eqref{eq:elimination-sets} for \emph{all} answers $i\in\cI$ at each round. While we establish slightly better theoretical guarantees for this rule, it is computationally demanding and, as we shall see in our experiments, it does not significantly improve sample complexity w.r.t. selective elimination, which remains our recommended choice.

Let $\tau_{\mathrm{s. elim}} = \inf_{t \geq 1}\{t \mid \cP_t^{\mathrm{stp}}(i^\star(\hat{\theta}_t)) = \emptyset\}$ and $\tau_{\mathrm{f. elim}} := \inf_{t \geq 1}\{t \mid \exists i\in\cI : \cP_t^{\mathrm{stp}}(i) = \emptyset\}$ be the stopping times of selective and full elimination, respectively. Intuitively, these two rules stop when one of the updated answers has all its pieces eliminated (and return that answer). We show that, as far as $\beta_{t,\delta}$ is chosen to ensure concentration of $\hat{\theta}_t$ to $\theta$, those two stopping rules are $\delta$-correct.

\begin{lemma}[$\delta$-correctness]\label{lem:delta-correct}
Suppose that $\beta_{t,\delta}$ guarantees \eqref{eq:concentration-beta} and that the algorithm verifies that, whenever it stops, there exists $i_{\emptyset}\in\cI$ such that $\cP_\tau^{\mathrm{stp}}(i_{\emptyset}) = \emptyset$ and $\ihat = i_{\emptyset}$. Then, $\mathbb{P}_\theta(\ihat \ne i^\star(\theta)) \le \delta$.
\end{lemma}
All proofs for this section are in Appendix~\ref{app:proofs-elim-stopping}. If an algorithm verifies the conditions of Lemma~\ref{lem:delta-correct} and has a sampling rule that makes it stop almost surely, then it is $\delta$-correct. Interestingly, we can prove a stronger result than $\delta$-correctness: under the same sampling rule, the elimination stopping rules never trigger later than the LLR one \emph{almost surely}. In other words, any algorithm equipped with elimination stopping suffers a sample complexity that is never worse than the one of the same algorithm equipped with LLR stopping. Let $\tau_{\mathrm{llr}} := \inf_{t \geq 1}\{t \mid \inf_{\lambda \in \Lambda(i^\star(\hat{\theta}_t))} L_t(\hat{\theta}_t,\lambda) \geq \beta_{t,\delta}\}$.

\begin{theorem}\label{th:elim-better-than-llr}
For any sampling rule, almost surely $\tau_{\mathrm{f. elim}} \le \tau_{\mathrm{s. elim}} \le \tau_{\mathrm{llr}}$~.
\end{theorem}

The proof of this theorem is very simple: if $\tau_{\mathrm{llr}} = t$, then at $t$ all pieces $\Lambda_p(i^\star(\hat{\theta}_t))$ for $p\in\cP(i^\star(\hat{\theta}_t))$ can be eliminated, hence $\tau_{\mathrm{s. elim}} \le t$. The proof that $\tau_{\mathrm{f. elim}} \le \tau_{\mathrm{s. elim}}$ follows from the observation that full elimination always has less active pieces than selective elimination. Note that all three stopping rules must use the same threshold $\beta_{t,\delta}$ to be comparable. Although simple, Theorem \ref{th:elim-better-than-llr} has an important implication: we can take any existing algorithm that uses LLR stopping, equip it with elimination stopping instead, and obtain a new strategy that is never worse in terms of sample complexity and for which the original theoretical results on the stopping time still hold.

Finally, it is important to note that, while defining the elimination rule in the general form \eqref{eq:elimination-sets} allows us to unify many settings, storing/iterating over all sets $\cP_t^{\mathrm{stp}}(i)$ would be intractable in problems with large number of answers (e.g., top-m identification or thresholding bandits, where the latter is exponential in $K$).
Fortunately, we show in Appendix \ref{app:problems} that this is not needed and efficient implementations exist for these problems that take only polynomial time and memory.

\subsection{Elimination time of alternative pieces}
\label{sub:changing_the_stopping_rule_only}

We now show that 
elimination stopping can indeed discard certain alternative pieces much earlier that the stopping time. 
While all results so far hold for any distribution and bandit structure, in the remaining we focus on Gaussian linear bandits. Other distribution classes beyond Gaussians could be used with minor modifications (see Appendix \ref{sub:beyond_gaussians}) but the Gaussian case simplifies the exposition. Since most existing adaptive sampling rules target the optimal proportions from the lower bound of \cite{garivier2016optimal}, we unify them under the following assumption.
\begin{assumption}\label{asm:sampling-rule}
Consider the concentration events
\begin{align}\label{eq:Et}
E_t := \left\{ \forall s \leq t: L_s(\hat{\theta}_s,\theta) \leq \beta_{t,1/t^2} \right\} \: .
\end{align}
A sampling rule is said to have low information regret if there exists a problem-dependent function $R(\theta,t)$ which is sub-linear in $t$ such that for each time $t$ where $E_t$ holds,
\begin{align}\label{eq:no-regret-property}
\inf_{\lambda \in \Lambda(i^\star(\theta))} \sum_{k\in[K]}N_t^k\KL_k(\theta,\lambda) \geq t H^\star(\theta) - R(\theta,t).
\end{align}
\end{assumption}
The left-hand side of \eqref{eq:no-regret-property} can be understood as the information collected by the sampling rule at time $t$ to discriminate $\theta$ with all its alternatives. Therefore, Assumption \ref{asm:sampling-rule} requires that information to be comparable (up to a low-order term $R(\theta,t)$) with the maximal one from the lower bound.
In Appendix \ref{app:assumptions}, we show that this is satisfied by both Track-and-Stop \cite{garivier2016optimal} and the approach in \cite{degenne2019non}.

Let $H_p(\omega, \theta) := \inf_{\lambda \in \Lambda_p(i^\star(\theta))} \sum_{k\in[K]}\omega^k\KL_k(\theta,\lambda)$, the information that sampling with proportions $\omega$ brings to discriminate $\theta$ from the alternative piece $\Lambda_p(i^\star(\theta))$. Note that $H^\star(\theta) = \max_{\omega\in\Delta_K}\min_{p\in\cP(i^\star(\theta))}H_p(\omega, \theta)$. For $\epsilon\geq 0$, let $\Omega_\epsilon(\theta) := \{ \omega\in\Delta_K \mid \inf_{\lambda \in \Lambda(i^\star(\theta))}  \sum_k \omega^k\KL_k(\theta,\lambda) \geq H^\star(\theta) - \epsilon\}$ be the set of $\epsilon$-optimal proportions.

\begin{theorem}[Piece elimination]\label{th:piece-elim}
The stopping time of any sampling rule having low information regret, combined with LLR stopping, satisfies $\mathbb{E}[\tau] \leq \bar{t} + 2$, where $\bar{t}$ is the first integer such that
\begin{align}\label{eq:llr-stopping-ineq}
t \geq \left(\left(\sqrt{\beta_{t,\delta}} + \sqrt{\beta_{t,1/t^2}}\right)^2 + R(\theta,t)\right) / H^\star(\theta).
\end{align}
When the same sampling rule is combined with elimination stopping, let $\tau_p$ be the time at which $p \in \cP(i^\star(\theta))$ is eliminated. Then, $\mathbb{E}[\tau_p] \leq \min\{\bar{t}_p, \bar{t} \} + 2$, where $\bar{t}_p$ is the first integer such that
\begin{align}\label{eq:elimination_time}
t \geq \max\left\{\frac{\left(\sqrt{\beta_{t,\delta}} + \sqrt{\beta_{t,1/t^2}}\right)^2}{\min_{\omega \in \Omega_{R(\theta,t)/t}(\theta)}H_p(\omega, \theta)}, G(\theta,t) \right\},
\end{align}
with $G(\theta,t) = 0$ for full elimination and $G(\theta,t) = \frac{ 4\beta_{t,1/t^2} + R(\theta,t)}{H^\star(\theta)}$ for selective elimination.
\end{theorem}

First, the bound we obtain on the elimination time of pieces in $\cP(i^\star(\theta))$ is not worse than the bound we obtain on the stopping time of LLR stopping.
Second, with elimination stopping, such eliminations can actually happen much sooner. Intuitively, sampling rules with low information regret play arms with proportions that are close to the optimal ones. If all of such ``good'' proportions provide large information for eliminating some piece $p\in\cP(i^\star(\theta))$, then $p$ is eliminated much sooner than the actual stopping time (which requires eliminating the worst-case piece in the same set).

While both elimination rules are provably efficient, with full elimination enjoying slighly better guarantees\footnote{Note that $G(\theta,t)$ for selective elimination contributes only a finite (in $\delta$) sample complexity.}, selective elimination provably never worsens (and possibly improves) the computational complexity over LLR stopping. In all applications we consider, implementing LLR stopping requires one minimization for each of the same alternative pieces we use for elimination stopping.
Therefore, the total number of minimizations required by LLR stopping is $\sum_{t=1}^{\tau_{\mathrm{llr}}} |\cP(i^\star(\hat{\theta}_t))|$ versus $\sum_{t=1}^{\tau_{\mathrm{s.elim}}} |\cP_t^{\mathrm{stp}}(i^\star(\hat{\theta}_t))|$ for selective elimination.
The second is never larger since $\tau_{\mathrm{s.elim}} \leq \tau_{\mathrm{llr}}$ by Theorem \ref{th:elim-better-than-llr} and $\cP_t^{\mathrm{stp}}(i^\star(\hat{\theta}_t)) \subseteq \cP(i^\star(\hat{\theta}_t))$ for all $t$, and much smaller if eliminations happen early, as we shall verify in experiments.
In our linear BAI example we need to perform $(K-1)$ minimizations at each step, one for each sub-optimal arm, in order to implement LLR stopping.
On the other hand, we need only $|\cP_t^{\mathrm{stp}}(i^\star(\hat{\theta}_t))|$ minimizations with selective elimination, one for each active sub-optimal arm, while full elimination takes $\sum_{i\in[K]}|\cP_t^{\mathrm{stp}}(i)|$ to update all the sets.

Note that Theorem \ref{th:piece-elim} does not provide a better bound on $\mathbb{E}[\tau]$ for elimination stopping than for LLR stopping. In fact, when evaluating the bound on $\mathbb{E}[\tau_p]$ for the worst-case piece in $p\in\cP(i^\star(\theta))$, we recover the one on $\mathbb{E}[\tau]$. This is intuitive since the sampling rule is playing proportions that try to eliminate all alternative pieces at once. The following result formalizes this intuition. 
\begin{theorem}\label{th:elim-vs-llr-fixed-sampling}
Suppose that we can write $\beta_{t, \delta} = \log\frac{1}{\delta} + \xi(t, \delta)$ with $\lim_{\delta \to 0}\xi(t, \delta)/\log(1/\delta) = 0$. Then for any sampling rule that satisfies Assumption \ref{asm:sampling-rule},
\begin{align*}
\mathbb{E}[\tau_{\mathrm{llr}}] \le \mathbb{E}[\tau_{\mathrm{elim}}] + f(\theta, \delta) \: .
\end{align*}
with $\lim_{\delta \to 0} f(\theta, \delta)/\log(1/\delta) = 0$. Here $\tau_{\mathrm{elim}}$ can stand for either full or selective elimination.
\end{theorem}
See Appendix~\ref{subapp:proof_elim_vs_llr} for $f$. This result shows that when the sampling rule is tailored to the LLR stopping rule, the expected LLR and elimination stopping times differ by at most low-order (in $\log(1/\delta)$) terms. As $\delta\rightarrow 0$ the two expected stopping times converge to the same value $H^\star(\theta)^{-1}\log(1/\delta)$, which is the asymptotically-optimal sample complexity prescribed by the lower bound~\eqref{eq:lower_bound}.

We showed that, for both elimination rules, some pieces of the alternative are discarded sooner than the stopping time, and that the overall sample complexity of the method can only improve over LLR stopping.
However, since the sampling rule of the algorithm was not changed, elimination does not change the computational cost of each sampling step, only the cost of checking the stopping rule.

\subsection{An example}
\label{sub:an_example}

We compare LLR and elimination stopping on a simple example so as to better quantify the elimination times of Theorem~\ref{th:piece-elim} and their computational impact (see Appendix~\ref{app:example} for a full discussion).

\begin{wrapfigure}{r}{4.5cm}
    \includegraphics[width=4.5cm]{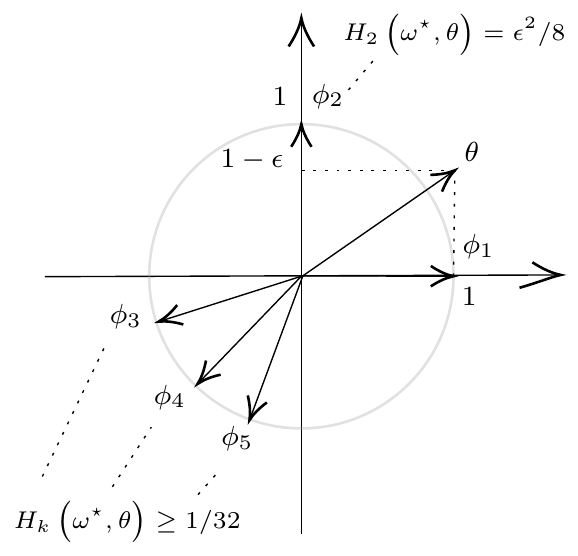}
    \caption{Example of BAI instance with $d=2$ and $K=5$.}\label{fig:example-main}
    \vspace{-0.5cm}
\end{wrapfigure} 

Consider BAI in a Gaussian linear bandit instance with unit variance, $d=2$, and arbitrary number of arms $K \geq 3$ (see Figure \ref{fig:example-main}). The arm features are $\phi_1 = (1,0)^T$, $\phi_2 = (0,1)^T$, and, for all $i = 3,\dots,K$, $\phi_i = (a_i,b_i)^T$ with $a_i,b_i$ arbitrary values in $(-1,0)$ such that $\|\phi_i\|_2=1$. The true parameter is $\theta = (1,1-\varepsilon)^T$, for $\varepsilon \in (0,1/2)$ a possibly very small value. Arm 1 is optimal with mean $\mu_1(\theta) = 1$, while arm 2 is sub-optimal with mean $\mu_2(\theta) = 1-\varepsilon$. For all other arms $i=3,\dots,K$, $\mu_i(\theta) \leq 0$.

Let $\omega \in \Delta_K$ be any allocation. Recall that in BAI each piece index is simply an arm, and $\cP(i^\star(\theta)) = \cP(1) = \{2,\dots,K\}$. Let $k\in\cP(1)$ be any sub-optimal arm. The distance to the $k$-th alternative piece $H_k(\omega, \theta)$ can be computed in closed form as
$H_k(\omega, \theta) = ((\phi_1 - \phi_k)^T \theta)^2/(2\|\phi_1 - \phi_k\|_{V_\omega^{-1}}^2)$.
The optimal allocation is $\omega^\star = \argmax_\omega \min_k H_k(\omega, \theta) = (1/2,1/2,0,\dots,0)^\top$.

The intuition why this example is interesting is as follows. Any correct strategy is required to discriminate between arm 1 and 2 (i.e., to figure out that arm 1 is optimal), which requires roughly $O(1/\varepsilon^2)$ samples from both. An optimal strategy plays these two arms nearly with the same proportions. Since $\phi_1$ and $\phi_2$ form the canonical basis of $\mathbb{R}^2$, the samples collected by this strategy are informative for estimating the mean reward of \emph{every} arm, even those than are not played. Then, since arms $3,\dots,K$ have at least a sub-optimality gap of $1$, an elimination-based strategy discards them with a number of samples not scaling with $1/\varepsilon^2$. This means that a non-elimination strategy runs for $O(1/\varepsilon^2)$ steps over the original problem with $K$ arms, while an elimination-based one quickly reduces the problem to one with only 2 arms. The main impact is computational: since most algorithms need to compute some statistics for each active arm at each round (e.g., closest alternatives, confidence intervals, etc.), the computational complexity of a non-elimination algorithm is at least $O(K/\varepsilon^2)$, while the one of an elimination-based variant is roughly $O(K + 1/\varepsilon^2)$, a potentially very large improvement.

We now quantify the elimination times and computational complexity on this example. Since such quantities depend on the specific sampling rule, we do it for an oracle strategy that samples according to $\omega^\star$. Similar results can be derived for any low information regret sampling rule (see Appendix~\ref{app:example}).
\begin{proposition}
    For any $K \geq 3$ and $\varepsilon \in (0,1/2)$, for any $\delta \in (0,1)$, the oracle strategy combined with LLR stopping satisfies on the example instance
    \begin{align*}
        \mathbb{E}[\tau] \geq \Omega\left(\frac{\log(1/\delta)}{\varepsilon^2}\right).
    \end{align*}
    On the same instance, for the oracle strategy with elimination at stopping and a threshold $\beta_{t,\delta} = \log(1/\delta) + O(\log(t))$, the expected elimination time of any piece (i.e., arm) $k\geq 3$ is
    \begin{align*}
        \mathbb{E}[\tau_k] &\leq \widetilde{O}(\log(1/\delta)) & \text{for full elimination},\\
        \mathbb{E}[\tau_k] &\leq \widetilde{O}\left(\log(1/\delta) + \frac{1}{\varepsilon^2}\right) & \text{for selective elimination}.
    \end{align*}
    Moreover, the expected per-round computation time of the oracle strategy with LLR stopping is $\Omega(K)$, while it is at most $O(K^2\varepsilon^2)$ for full elimination and $O(K\varepsilon^2 + K/\log(1/\delta))$ for selective elimination.
\end{proposition}

%% file: sections/sampling.tex

\section{Elimination at sampling}
\label{sec:elimination_at_sampling}

We show how to adapt sampling rules in order to accommodate piece elimination. There are two reasons for doing this: first, adapting the sampling to ignore pieces that have been discarded could reduce the sample complexity; second, the amount of computations needed to update the sampling strategy is often proportional to the number of pieces and decreasing it can reduce the overall time.

We start from an algorithm using LLR stopping, for which we change the stopping rule as above. The sampling strategies that we can adapt are those that aggregate information from each alternative piece.
For example, in linear BAI, methods that mimic the lower bound allocation~\eqref{eq:lower_bound}, like Track-and-Stop \citep{garivier2016optimal}, LinGame \citep{degenne2020gamification}, or FWS \citep{wang2021fast}, and even LinGapE \citep{xu2018fully}, all compute distances or closest points to each piece in the decomposition $\{\lambda \mid \phi_j^\top \lambda \ge \phi_{i^\star(\hat{\theta}_t)}^\top \lambda\}$.
Eliminating pieces at sampling simply means omitting from such computations the arms that were deemed sub-optimal. Algorithm \ref{alg:tas} shows how Track-and-Stop \cite{garivier2016optimal} can be modified to incorporate elimination at sampling and stopping.

\begin{algorithm}
	\small
	\begin{tabularx}{\textwidth}{*{2}{>{\centering\arraybackslash}X}}
	\begin{algorithmic}
	\vspace{-0.8em}
	\While{not stopped}
		\State \vphantom{Set ${\cP}_{t}^{\mathrm{stp}}(i^\star(\hat{\theta}_{t})) = {\cP}_{t-1}^{\mathrm{stp}}(i^\star(\hat{\theta}_{t}))$}
		\For{\textcolor{carmine}{$p \in \cP(i^\star(\hat{\theta}_{t}))$}} \Comment{stopping}
			\State $L_{p,t} = \inf_{\lambda \in \Lambda_p(i^\star(\hat{\theta}_{t}))} L_t(\hat{\theta}_{t}, \lambda)$
			\State \vphantom{\textbf{if} $L_{p,t} > \beta_{t,\delta}$ \textbf{delete} $p$ from ${\cP}_{t}^{\mathrm{stp}}(i^\star(\hat{\theta}_{t}))$}
		\EndFor
		\State \textbf{if} $\textcolor{carmine}{\forall p \in \cP(i^\star(\hat{\theta}_{t}))} : L_{p,t} > \beta_{t,\delta}$ \textbf{then} STOP
		\State $w_t = \argmax_{\omega} \min_{\textcolor{carmine}{p \in \cP(i^\star(\hat{\theta}_{t}))}} H_p(\omega,\hat{\theta}_t)$
		\State \textbf{if} $ \exists k : N_t^k < \sqrt{t}$ pull $k_{t+1} = \argmin_k N_t^k$
		\State \textbf{else} pull $k_{t+1} = \argmin_k (N_t^k - t w_t^k)$
		\State \vphantom{Update $\cP_{t+1}^{\mathrm{smp}}(i^\star(\hat{\theta}_{t}))$ (Algorithm \ref{alg:update-pieces})}
	\EndWhile
	\vspace{-1em}
	\end{algorithmic}
	&
	\begin{algorithmic}
		\vspace{-0.8em}
		\While{not stopped}
		\State Set ${\cP}_{t}^{\mathrm{stp}}(i^\star(\hat{\theta}_{t})) = {\cP}_{t-1}^{\mathrm{stp}}(i^\star(\hat{\theta}_{t}))$
		\For{\textcolor{carmine}{$p \in {\cP}_{t-1}^{\mathrm{stp}}(i^\star(\hat{\theta}_{t}))$}} \Comment{stopping}
			\State $L_{p,t} = \inf_{\lambda \in \Lambda_p(i^\star(\hat{\theta}_{t}))} L_t(\hat{\theta}_{t}, \lambda)$
			\State \textbf{if} $L_{p,t} > \beta_{t,\delta}$ \textbf{delete} $p$ from ${\cP}_{t}^{\mathrm{stp}}(i^\star(\hat{\theta}_{t}))$
		\EndFor
		\State \textbf{if} \textcolor{carmine}{${\cP}_{t}^{\mathrm{stp}}(i^\star(\hat{\theta}_{t})) = \emptyset$} \textbf{then} STOP
		\State $w_t = \argmax_{\omega} \min_{\textcolor{carmine}{p \in \cP_t^{\mathrm{smp}}(i^\star(\hat{\theta}_{t}))}} H_p(\omega,\hat{\theta}_t)$
		\State \textbf{if} $ \exists k : N_t^k < \sqrt{t}$ pull $k_{t+1} = \argmin_k N_t^k$
		\State \textbf{else} pull $k_{t+1} = \argmin_k (N_t^k - t w_t^k)$
		\State Update $\cP_{t+1}^{\mathrm{smp}}(i^\star(\hat{\theta}_{t}))$ (Algorithm \ref{alg:update-pieces})
	\EndWhile
	\vspace{-1em}
	\end{algorithmic}
	\end{tabularx}
	\caption{Track-and-Stop \cite{garivier2016optimal}: vanilla (left) and with selective elimination (right)}\label{alg:tas}
	\end{algorithm}


Similarly to elimination stopping, the idea is to maintain sets of active pieces at sampling $\cP_{t}^{\mathrm{smp}}(i)$ for each $i\in\cI$. Note that these are different from the ones introduced in Section \ref{sec:elimination_stopping_rules} for the stopping rule. The set is updated at each step like we did for the stopping sets, but with a different threshold $\alpha_{t,\delta}$ (see Appendix~\ref{app:proofs_of_sampling} for details). Additionally, we reset it very infrequently at steps $t\in\{\bar{t}_0^{2^j}\}_{j\geq 0}$, where $\bar{t}_0 \ge 2$.
Formally, let us define the helper sets $\tilde{\cP}_{t}^{\mathrm{smp}}(i)$ as $\tilde{\cP}_{0}^{\mathrm{smp}}(i) := \cP(i)$ and
\begin{align*}
\tilde{\cP}_t^{\mathrm{smp}}(i) := \begin{cases}
\tilde{\cP}_{t-1}^{\mathrm{smp}}(i) \cap \overline{\cP}_t(i;\alpha_{t,\delta}) & \text{if } t\notin\{\bar{t}_0^{2^j}\}_{j\geq 0}
\\ \overline{\cP}_t(i;\alpha_{t,\delta}) & \text{otherwise},
\end{cases}
\end{align*}
where $\overline{\cP}_t$ was defined in \eqref{eq:active-pieces-t-only}.
Let $\overline{t}_j := \bar{t}_0^{2^j}$ be the time step at which the $j$-th reset is performed and $j(t) := \lfloor \log_2\log_{\bar{t}_0} t \rfloor$ be the index of the last reset before $t$.
We define ${\cP}_t^{\mathrm{smp}}(i) := \tilde{\cP}_t^{\mathrm{smp}}(i) \cap \tilde{\cP}_{\overline{t}_{j(t)} - 1}^{\mathrm{smp}}(i)$, such that ${\cP}_t^{\mathrm{smp}}(i)$ is the intersection of all active pieces from the second-last reset up to $t$, i.e.,
$
{\cP}_t^{\mathrm{smp}}(i) = \bigcap_{s=\overline{t}_{j(t)-1}}^t \overline{\cP}_s(i;\alpha_{s,\delta}).
$
Since the resets are very infrequent, this definition only drops a small number of rounds from the intersection (less than $\sqrt{t}$). The detailed procedure to update these sets is summarized in Algorithm \ref{alg:update-pieces}. As before, we can instantiate both selective and full elimination.
%

The reason for the resets is two-fold. First, they ensure that the algorithm stops almost surely as required by Definition \ref{def:delta-correct}. In fact, without resets, it might happen with some small (less than $\delta$) probability that pieces containing the true parameter are eliminated, in which case the sampling rule could diverge. Second, they guarantee that the thresholds $(\alpha_{s,\delta})_{s=\overline{t}_{j(t)-1}}^t$ used in ${\cP}_t^{\mathrm{smp}}(i)$ are within a constant factor of each other. This is crucial to relate the LLR of active pieces at different times.

\subsection{Properties}
\label{sub:properties}

We consider a counterpart of Assumption \ref{asm:sampling-rule} for sampling rules combined with piece elimination.

\begin{assumption}\label{asm:sampling-rule-v2}
There exists a sub-linear (in $t$) problem-dependent function $R(\theta,t)$ such that, for each time $t$ where $E_t$ (defined in Equation \ref{eq:Et}) holds,
\begin{align*}
\min_{p\in\cP_{t}^{\mathrm{smp}}(i^\star(\theta))}\! \! \! \! \! H_p(N_t, \theta) 
\geq \max_{\omega\in\Delta_K} \sum_{s=1}^t \min_{p\in\cP_{s-1}^{\mathrm{smp}}(i^\star(\theta))}\! \! \! \! \! H_p(\omega, \theta) {-} R(\theta,t).
\end{align*}
\end{assumption}

Intuitively, the sampling rule maximizes the information for discriminating $\theta$ with all its alternatives from the sequence of active pieces $(\cP_{s-1}^{\mathrm{smp}}(i^\star(\theta)))_{s=1}^t$. We prove in Appendix \ref{app:assumptions} that the algorithms for which we proved Assumption \ref{asm:sampling-rule} also satisfy Assumption \ref{asm:sampling-rule-v2} when their sampling rules are combined with either full or selective elimination.

\begin{theorem}\label{th:sampling-rule-with-elim}
Consider a sampling rule that verifies Assumption \ref{asm:sampling-rule-v2} and uses either full or selective elimination with the sets ${\cP}_t^{\mathrm{smp}}$. Then, Assumption \ref{asm:sampling-rule} holds as well. Moreover, when using the same elimination rule at stopping, such a sampling rule verifies Theorem \ref{th:piece-elim}, i.e., it enjoys the same guarantees as without elimination at sampling.
\end{theorem}

The proof is in Appendix~\ref{app:proofs_of_sampling}. Theorem \ref{th:sampling-rule-with-elim} shows that for an algorithm using elimination at sampling and stopping, we get bounds on the times at which pieces of $\Lambda(i^\star(\theta))$ are discarded from the stopping rule which are not worse than those we obtained for the same algorithm without elimination at sampling.
This result is non-trivial. We know that the sampling rule collects information to discriminate $\theta$ with its closest alternatives, and eliminating a piece cannot make the resulting ``optimal'' proportions worse at this task. However, it could make them worse at discriminating $\theta$ with alternatives that are not the closest.
This would imply that the elimination times for certain pieces could actually increase w.r.t. not eliminating at sampling.
Theorem \ref{th:sampling-rule-with-elim} guarantees that this does not happen: eliminating pieces at sampling cannot worsen our guarantees.
We shall see in our experiments that eliminating pieces in both the sampling and stopping rules often yields improved sample complexity. 

%% file: sections/experiments.tex
\section{Experiments}
\label{sec:experiments}

\begin{figure*}[t]
\centering
	\includegraphics[scale=0.21]{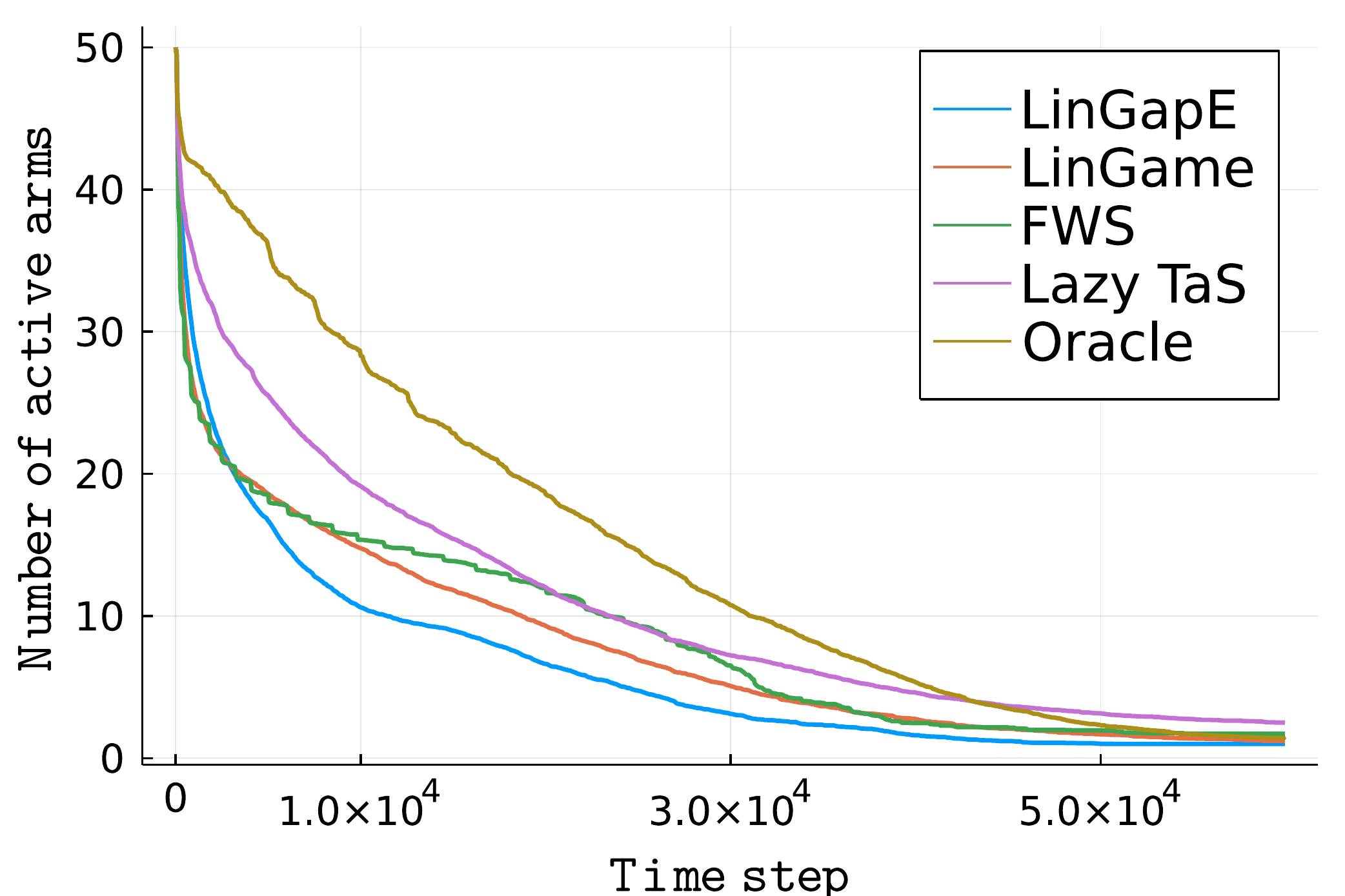}
	\includegraphics[scale=0.21]{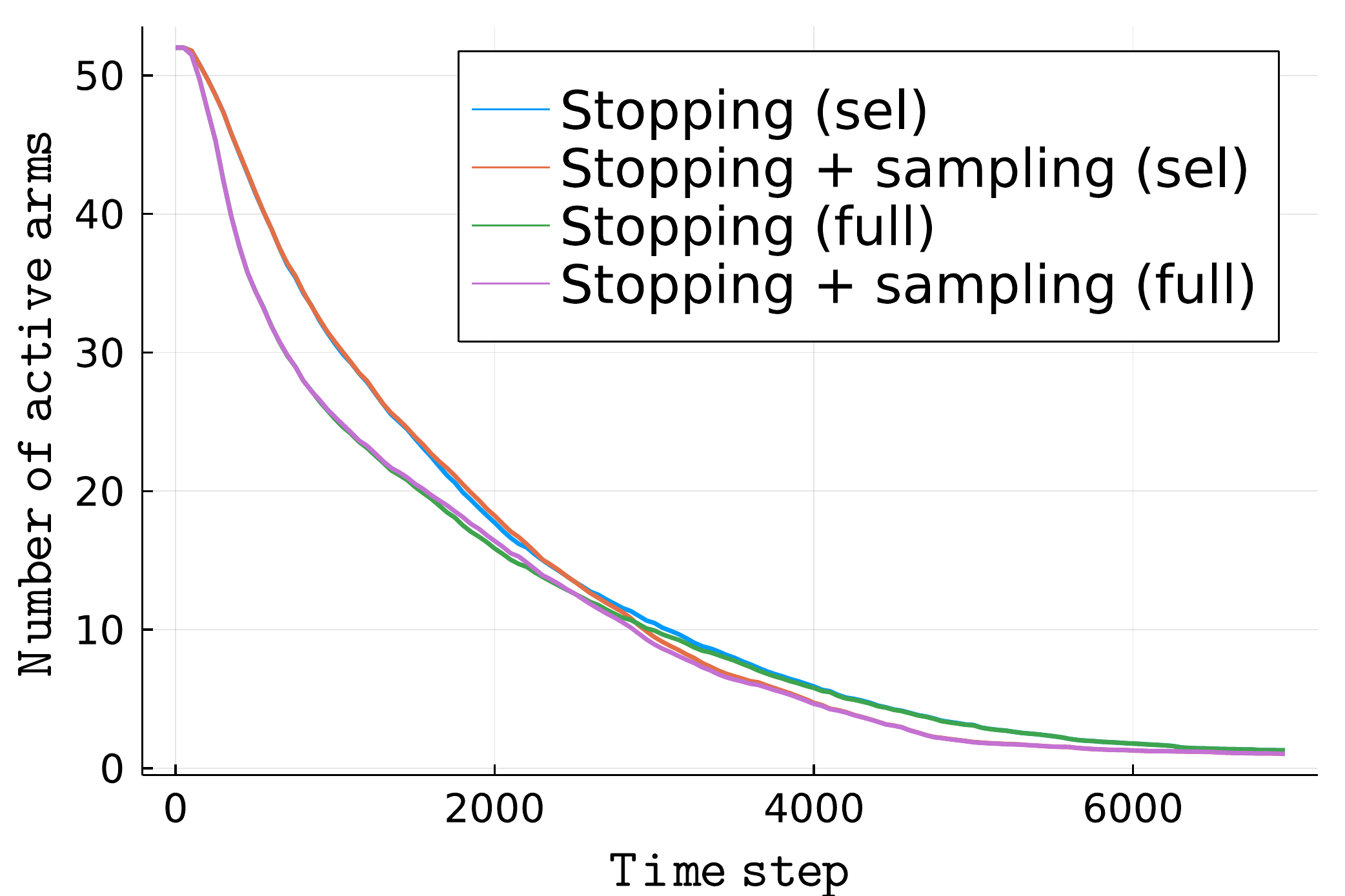}
	\includegraphics[scale=0.21]{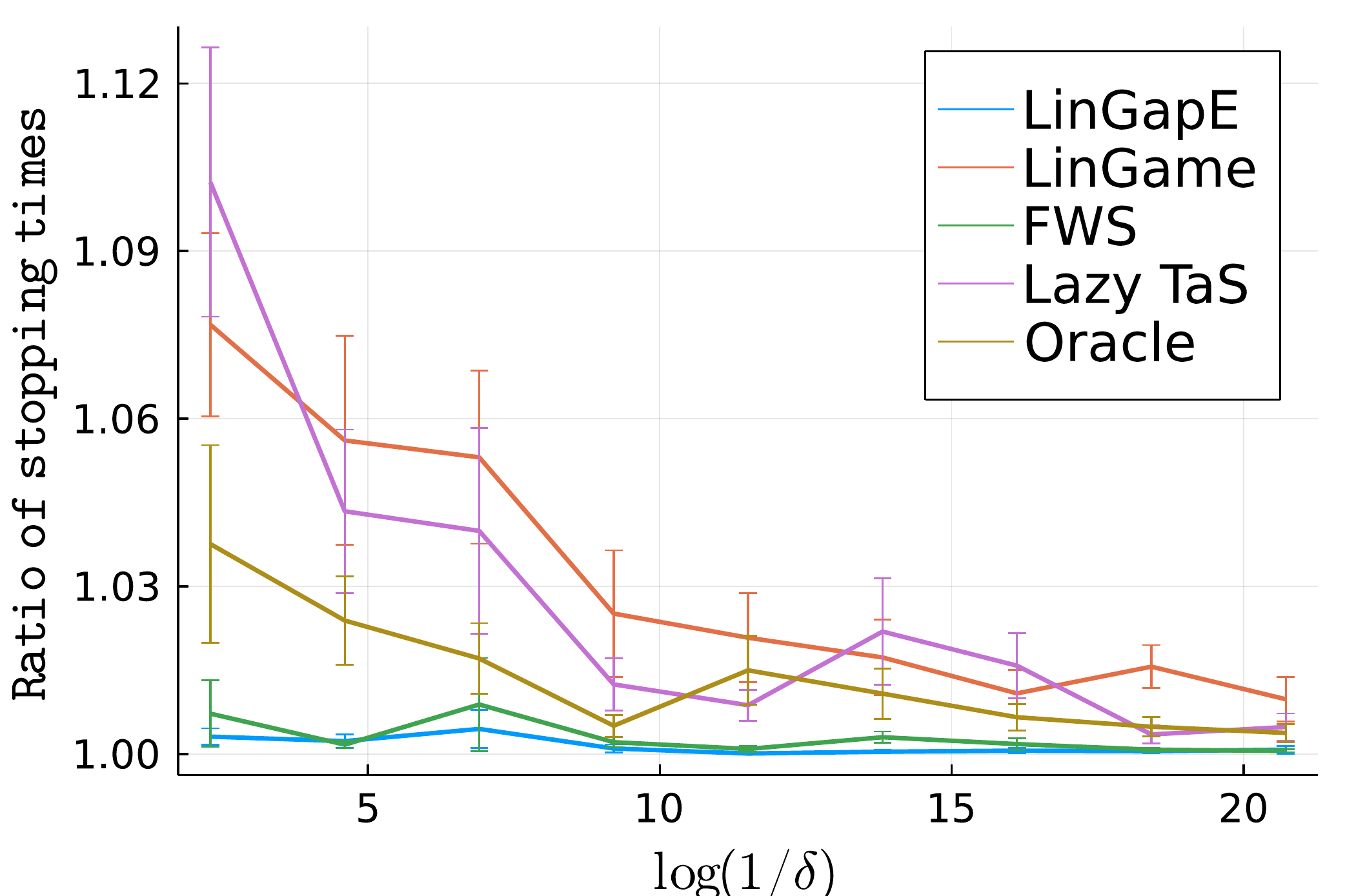}
	\caption{Experiments on linear instances with $K=50$, $d=10$, averaged over 100 runs, with the right plot showing standard deviations. (left) How different adaptive algorithms eliminate arms in BAI when using elimination stopping. (middle) LinGame on BAI when combined with full and selective elimination rules, either only at stopping or both at stopping and at sampling. (right) Ratio between the LLR and elimination stopping times of different algorithms as a function of $\log(1/\delta)$.}
	\label{fig:all}
\end{figure*}

Our experiments aim at addressing the following questions: (1) how do existing adaptive strategies behave when combined with elimination at stopping and (when possible) at sampling? How do they compare with native elimination-based methods? (2) What is the difference between selective and full elimination? (3) How do LLR and elimination stopping compare as a function of $\delta$?\footnote{Our code is available at \url{https://github.com/AndreaTirinzoni/bandit-elimination}.}



We ran experiments on two bandit structures: linear (where $d < K$) and unstructured (where $K=d$ and the arms are the canonical basis of $\mathbb{R}^d$). For each of them, we considered 3 pure exploration problems: BAI, Top-m, and online sign identification (OSI) \citep{carpentier2016tight,ouhamma2021online}, also called thresholding bandits. All experiments use $\delta=0.01$ and are averaged over 100 runs.

We combined adaptive algorithms which are natively based on LLR stopping with our elimination stopping rules and, whenever possible, we extended their sampling rule to use elimination. The selected baselines are the following. For linear BAI, LinGapE \citep{xu2018fully}, LinGame \citep{degenne2020gamification}, Frank-Wolfe Sampling (FWS) \citep{wang2021fast}, Lazy Track-and-Stop (TaS) \citep{jedra2020optimal}, XY-Adaptive \citep{soare2014best}, and RAGE \citep{fiez2019sequential} (the latter two are natively elimination based). For linear Top-m, m-LinGapE \citep{reda2021top}, MisLid \citep{reda2021dealing}, FWS, Lazy TaS\footnote{Lazy TaS, while analyzed only for BAI, can be applied to any problem since it is a variant of Track-and-Stop.}, and LinGIFA \citep{reda2021top}. For linear OSI, LinGapE\footnote{LinGapE was originally proposed only for BAI in \cite{xu2018fully}, but its extension to OSI is trivial.}, LinGame, FWS, and Lazy TaS.
For unstructured instances linear algorithms are still applicable, and we further implemented LUCB \citep{kalyanakrishnan2012pac}, UGapE \citep{gabillon2012best}, and the Racing algorithm \citep{kaufmann2013information} for BAI and Top-m. We also tested an ``oracle'' sampling rule which uses the optimal proportions from the lower bound. Due to space constraints, we present only the results on linear structures. Those on unstructured problems can be found in Appendix~\ref{app:experiments}. 
%
%
The first experiments use randomly generated instances with $K=50$ arms and dimension $d=10$.

\textbf{Comparison of elimination times.} We analyze how different adaptive algorithms eliminate pieces when combined with selective elimination at stopping. To this purpose we focus on BAI, where the sets of pieces can be conveniently reduced to a set of active arms, those that are still likely to be the optimal one.
Figure \ref{fig:all}\emph{(left)} shows how the set of active arms evolves over time for the 5 adaptive baselines.
Notably, many arms are eliminated very quickly, with most baselines able to halve the set of active arms in the first 3000 steps.
The problem size is quickly reduced over time. As we shall see in the last experiment, this will yield significant computational gains.
We further note that the ``oracle'' strategy, which plays fixed proportions, seems the slowest at eliminating arms.
The reason is that the optimal proportions from the lower bound focus on discriminating the ``hardest'' arms, while the extra randomization in adaptive rules might indeed eliminate certain ``easier'' arms sooner.

\begin{table*}[t!]
\centering
\small
\begin{tabular}{@{}clcccccc@{}} 
\toprule
 & & \multicolumn{2}{c}{No elimination (LLR)} & \multicolumn{2}{c}{Elim. stopping} & \multicolumn{2}{c}{Elim. stopping + sampling} \\
\cmidrule(r){3-8}
& Algorithm & Samples & Time & Samples & Time & Samples & Time \\
\cmidrule{1-8}
\multirow{7}{*}{\rotatebox[origin=c]{90}{BAI}} 
& LinGapE & $33.19 \pm 8.7$ & $0.23$ & $33.11 \pm 8.7$ & $0.2$ & $29.89 \pm 8.6$ & $0.18 (-22\%)$ \\
& LinGame & $45.34 \pm 14.2$ & $0.23$ & $43.67 \pm 13.4$ & $0.21$ & $32.49 \pm 8.1$ & $0.18 (-22\%)$ \\
& FWS & $42.26 \pm 60.1$ & $0.73$ & $42.25 \pm 60.1$ & $0.7$ & $32.62 \pm 18.0$ & $0.45 (-38\%)$ \\
& Lazy TaS & $76.33 \pm 65.8$ & $0.15$ & $74.08 \pm 65.8$ & $0.13$ & $64.48 \pm 81.8$ & $0.12 (-20\%)$ \\
& Oracle & $56.36 \pm 9.1$ & $0.05$ & $55.36 \pm 9.3$ & $0.02$ & & \\
& XY-Adaptive & & & & & $87.08 \pm 29.1$ & $0.44$ \\
& RAGE & & & & & $106.87 \pm 30.7$ & $0.02$ \\
\cmidrule{1-8}
\multirow{6}{*}{\rotatebox[origin=c]{90}{Top-m ($m=5$)}} 
& m-LinGapE & $63.69 \pm 11.1$ & $0.56$ & $63.48 \pm 11.0$ & $0.41$ & $59.57 \pm 9.4$ & $0.24 (-57\%)$ \\
& MisLid & $87.77 \pm 20.4$ & $0.55$ & $85.95 \pm 20.5$ & $0.4$ & $69.58 \pm 16.0$ & $0.25 (-55\%)$ \\
& FWS & $78.28 \pm 65.0$ & $3.0$ & $78.23 \pm 65.0$ & $2.85$ & $77.79 \pm 65.0$ & $0.97 (-67\%)$ \\
& Lazy TaS & $161.43 \pm 96.9$ & $0.57$ & $159.86 \pm 96.9$ & $0.43$ & $146.06 \pm 82.6$ & $0.36 (-36\%)$ \\
& Oracle & $102.45 \pm 16.1$ & $0.2$ & $101.53 \pm 16.4$ & $0.08$ & & \\
& LinGIFA & $58.31 \pm 10.8$ & $2.46$ & $58.31 \pm 10.8$ & $2.33$ & & \\
\cmidrule{1-8}
\multirow{5}{*}{\rotatebox[origin=c]{90}{OSI}}
& LinGapE & $17.31 \pm 2.3$ & $0.22$ & $17.29 \pm 2.2$ & $0.19$ & $14.71 \pm 2.0$ & $0.17 (-23\%)$ \\
& LinGame & $23.77 \pm 4.1$ & $0.25$ & $23.05 \pm 3.9$ & $0.21$ & $14.87 \pm 2.0$ & $0.19 (-24\%)$ \\
& FWS & $15.26 \pm 2.0$ & $0.83$ & $15.24 \pm 2.0$ & $0.81$ & $14.99 \pm 2.1$ & $0.56 (-32\%)$ \\
& Lazy TaS & $35.11 \pm 10.2$ & $0.32$ & $33.98 \pm 9.7$ & $0.3$ & $23.51 \pm 5.6$ & $0.24 (-25\%)$ \\
& Oracle & $29.1 \pm 4.8$ & $0.06$ & $28.65 \pm 5.0$ & $0.03$ & & \\
\bottomrule
\end{tabular}
\caption{Experiments on linear instances with $K=50$, $d=20$. The "Time" columns report average times per iteration in milliseconds. The percentage in the last column is the change w.r.t. the time without elimination. Each entry reports the mean across $100$ runs plus/minus standard deviation (which is omitted for compute times due to space constraints). Algorithms for which the third column is missing cannot be combined with elimination at sampling, while algorithms for which the first two columns are missing are natively elimination-based. Samples are scaled down by a factor $10^3$.}\label{tab:lin_all}
\end{table*}

\textbf{Full versus selective elimination.} We 
combine the different algorithms with full and selective elimination, both at sampling and stopping.
Due to space constrains, Figure \ref{fig:all}\emph{(middle)} shows the results only for LinGame (see Appendix~\ref{app:experiments} for the others).
We note that full elimination seems faster at discarding arms in earlier steps, as we would expect theoretically.
However, it never stops earlier than its selective counterpart. Moreover, its computational overhead is not advantageous. Overall, we concluded that our selective elimination rule is the best choice and we shall thus focus on it in the remaining. Finally, we remark that combining the sampling rule with elimination (no matter of what type) seems to discard arms faster in later steps, and could eventually make the algorithm stop sooner. 

\textbf{LLR versus elimination stopping.} We now compare LLR and elimination stopping as a function of $\delta$.
We know from theory that both stopping rules allow to achieve asymptotic optimality. Hence for asymptotically optimal sampling rules the resulting stopping times with LLR and elimination should tend to the same quantity as $\delta\rightarrow 0$.
Figure \ref{fig:all}\emph{(right)}, where we report the ratio between the LLR stopping time and the elimination one for different algorithms, confirms that this is the case.
Some algorithms (LinGapE and FWS) seem to benefit less from elimination stopping than the others, i.e., they achieve smaller ratios of stopping times.
We believe this to be a consequence of their mostly ``greedy'' nature, while the extra randomization of the other algorithms might help in this aspect.

\textbf{Sample complexities and computation times.} We finally compare our baselines in all three exploration tasks, in terms of sample complexity and computation time.
For this experiment, we selected a larger linear instance with $K=50$ and $d=20$, randomly generated (see the protocol in Appendix~\ref{app:experiments}).
From the results in Table \ref{tab:lin_all}, we highlight three points.
(1) The computation times of all adaptive algorithms decrease when using selective elimination stopping instead of LLR and further decrease when also using elimination at sampling. In the case of Top-m (i.e., the hardest combinatorial problem), most adaptive algorithms become at least twice faster with elimination at stopping and sampling instead of LLR.
(2) Elimination at sampling improves the sample complexity of all algorithms.
(3) For BAI, the natively elimination-based algorithm RAGE, which updates its strategy infrequently, is the fastest in terms of computation time but the slowest in terms of samples. Adaptive algorithms using elimination achieve run times that are within an order of magnitude of those of RAGE, while outperforming it in terms of sample complexity by a factor 2 to 3.

%% file: sections/conclusion.tex
\section{Conclusion}
\label{sec:conclusion}

We proposed a selective elimination rule, which successively prunes the pieces of the empirical answer, that can be easily combined with existing adaptive algorithms for general identification problems. We proved that it reduces their computational complexity, it never worsens their sample complexity guarantees, and it provably discards certain answers early. Our experiments on different pure exploration problems and bandit structures show that existing adaptive algorithms often benefit from a reduced sample complexity when combined with selective elimination, while achieving significant gains in computation time. Moreover, they show that selective elimination is overall better (in terms of samples vs time) than its full variant which repeatedly updates the pieces of all answers.



Interesting directions for future work include investigating whether better guarantees on the stopping time can be derived for algorithms combined with elimination as compared to their LLR counterparts, and designing adaptive algorithms which are specifically tailored for elimination.

%% file: sections/appendix.tex

\section{Notation}\label{app:notation}

\begin{table*}[h]
\centering
\begin{tabular}{@{}ll@{}} 
\toprule
Symbol & Meaning \\
\cmidrule{1-2}
$[K] = \{1,2,\dots,K\}$ & Set of $K$ arms\\
$\Delta_K$ & $K$-dimensional simplex\\
$d\in\mathbb{N}_{>0}$ & Dimension of parameter space\\
$\cM \subseteq \mathbb{R}^d$ & Set of possible reward parameters\\
$\mathbb{P}_\theta$ & Distribution of observations in bandit $\theta\in\cM$\\
$\nu_k(\theta)$ & Reward distribution of arm $k$ in bandit $\theta\in\cM$\\ 
$\mu_k(\theta) := \mathbb{E}_{x\sim \nu_k(\theta)}[x]$ & Mean reward of arm $k$ in bandit $\theta\in\cM$\\ 
$\cI$ & Set of answers\\
$i^\star(\theta)$ & Correct answer for bandit problem $\theta\in\cM$\\
$\Lambda(i) := \{\lambda\in\cM : i^\star(\lambda) \neq i\}$ & Set of alternatives to answer $i\in\cI$\\
$\cP(i)$ & Set of alternative piece indexes for answer $i\in\cI$\\
$P_i := |\cP(i)|$ & Number of pieces for answer $i\in\cI$\\
$\Lambda_p(i)$ & Piece $p\in\cP(i)$ for answer $i\in\cI$\\
$X_{[t]} := (X_1^{k_1}, \ldots, X_t^{k_t})$  & Vector of $t$ observations\\
$L_t(\theta, \lambda)
:= \log \frac{d \mathbb{P}_\theta}{d \mathbb{P}_\lambda}(X_{[t]})$ & LLR of $t$ observations between $\theta$ and $\lambda$ \\
$\KL_k(\theta,\lambda) := \KL(\nu_k(\theta),\nu_k(\lambda))$ & KL divergence between $\nu_k(\theta)$ and $\nu_k(\lambda)$\\
$\hat{\theta}_t := \argmax_{\lambda\in\cM}d \mathbb{P}_\lambda (X_{[t]})$ & Maximum likelihood estimator for $\theta$\\
$\hat{\mu}_t^k := \frac{1}{N_t^k}\sum_{s=1}^t X_s^{k_s}\indi{k_s=k}$ & Empirical mean of arm $k$ (different from $\mu_k(\hat{\theta}_t))$\\
$H_p(\omega, \theta) := \inf_{\lambda \in \Lambda_p(i^\star)} \sum_{k\in[K]}\omega^k\KL_k(\theta,\lambda)$ & Information of $\omega\in\Delta_K$ for piece $\Lambda_p(i^\star)$\\
$H^\star(\theta) := \max_{\omega\in\Delta_K}\min_{p\in\cP(i^\star)}H_p(\omega, \theta)$ & Optimal constant from the lower bound \eqref{eq:lower_bound}\\
$\Omega_\epsilon(\theta)$ & Set of $\epsilon$-optimal proportions\\
\bottomrule
\end{tabular}
\caption{The notation adopted in this paper.}
\label{tab:notation}
\end{table*}

\section{Identification Problems}\label{app:problems}

In this section, we show that popular identification problems satisfy Assumption \ref{ass:union_of_sets} and are thus suitable for elimination-based algorithms. We shall focus on \emph{Gaussian linear bandits}, where, for any $\theta\in\cM$ and $k\in[K]$, $\nu_k(\theta)$ is Gaussian with unit variance and linear mean $\mu_k(\theta) = \theta^T\phi_k$. For each identification problem, we first show how to decompose the sets of alternatives into pieces for which the closest alternatives can be found efficiently. Moreover, we report the closed-form equations for computing such closest alternatives in Gaussian linear bandits and in the special case of unstructured bandits (where $\cM = \mathbb{R}^K$ and $\phi_k = e_k$, the canonical basis of $\mathbb{R}^k$). Finally, we show how to efficiently implement elimination strategies in each of these identification problems even when enumerating over all possible answers is intractable (e.g., for problems where the number of answers is exponential in the problem dimension).

\paragraph{The LLR in Gaussian linear bandits}

For all identification problems presented later, we need to show that $\inf_{\lambda \in \Lambda_p(i)}L_t(\hat{\theta},\lambda)$ can be computed efficiently for any piece. In Gaussian linear bandits, such a log-likelihood ratio is actually equivalent to a KL divergence (see Corollary \ref{cor:llr-lin-gauss}),
\begin{align*}
L_t(\hat{\theta},\lambda) = \sum_{k\in[K]}N_t^k\KL_k(\hat{\theta}_t,\lambda) = \frac{1}{2}\|\hat{\theta}_t-\lambda\|_{V_t}^2,
\end{align*}
which in turn is a quadratic form weighted by the design matrix $V_t := \sum_{s=1}^t \phi_{k_s}\phi_{k_s}^T$. Therefore, with greater generality, in rest of this section we shall focus on showing that $\inf_{\lambda \in \Lambda_p(i)}\|\theta-\lambda\|_{V_N}^2$ can be computed efficiently for any $\theta,\lambda\in\cM$, piece $\Lambda_p(i)$, and (positive-definite) matrix $V_N := \sum_{k\in[K]} N^k \phi_k\phi_k^T$ with $N\in\mathbb{R}^K_{\geq 0}$. In all cases, this will require minimizing quadratic forms over half-spaces. 

\subsection{Best-arm Identification}

In BAI, the goal is to find the arm with largest mean. The set of answers is therefore $\cI = [K]$ and the correct answer of $\theta\in\cM$ is $i^\star(\theta) = \argmax_{k\in[K]}\theta^T\phi_k$. 

\paragraph{Decomposition into pieces}

For each $i\in\cI$, the set of alternatives $\Lambda(i)$ can be decomposed into half-spaces,
\begin{align*}
\Lambda(i) = \bigcup_{k\in[K], k\neq i} \left\{ \lambda \in \cM : \lambda^T\phi_k > \lambda^T\phi_i\right\}.
\end{align*}
Therefore, we can take $\cP(i) = [K] \setminus \{i\}$ with $P_i = K-1$ and $\Lambda_p(i) = \{\lambda\in\cM : \lambda^T\phi_p > \lambda^T\phi_i\}$ for $p \in [K]\setminus\{i\}$.

\paragraph{Closest alternatives}

%

Take any $j,k\in[K]$ with $j\neq k$. For linear problems, for any $\theta\in\mathbb{R}^d$,
\begin{align*}
\inf_{\lambda \in \Lambda_j(k)}\Vert {\theta} - \lambda \Vert_N^2 = \begin{cases}
	\frac{(\theta^T(\phi_k-\phi_j))^2}{\| \phi_k - \phi_j \|_{V_N^{-1}}^2} &\text{if } \theta^T(\phi_k-\phi_j) \geq 0, \\
	0 &\text{otherwise}.
\end{cases}
\end{align*}
For the special case of unstructured problem, for any $\theta\in\mathbb{R}^K$,
\begin{align*}
\inf_{\lambda \in \Lambda_j(k)}\Vert {\theta} - \lambda \Vert_{V_N}^2 = \begin{cases}
	\frac{N_jN_k}{N_j + N_k}(\theta^T(\phi_k-\phi_j))^2 &\text{if } \theta^T(\phi_k-\phi_j) \geq 0\ \text{and } N_j + N_k > 0, \\
	0 &\text{otherwise}.
\end{cases}
\end{align*}

\paragraph{Efficient implementation}

In BAI, for each answer (i.e., arm) $i\in\cI$, the piece indexes $p\in\cP(i)$ are themselves answers (different than $i$). Implementing the elimination stopping rule in its general form requires storing and iterating over $K(K-1)$ elements (all items in $\cP(i)$ for each $i\in\cI$). However, much better implementations exist that require storing at most $K$ elements (one for each answer). Here we propose two such implementations: the first, for full elimination, is more statistically-efficient, while the second one (for selective elimination) is more computationally-efficient.

\paragraph{Full elimination (statistically-efficient implementation)}

Due to the structure of the problem, whenever we eliminate one piece $\Lambda_p(i)$, we actually know that the mean reward of arm $p$ cannot be better than that of arm $i$. In other words, $p$ cannot be the right answer. Therefore, we can maintain a list of active arms $\cI_t$ which is initialized as $\cI_0 = [K]$ and updated as
\begin{align*}
\mathcal{I}_t := \mathcal{I}_{t-1} \setminus \left\{ j\in\cI_{t-1} \big| \max_{i\neq j} \inf_{\lambda \in \Lambda_j(i)} L_t(\hat{\theta}_t,\lambda) \geq \beta_{t,\delta}\right\}.
\end{align*}
Then, we stop whenever $|\cI_t| = 1$ and return the single arm left active. Due to the inner maximization, this implementation requires performing $|\cI_{t-1}|(K-1)$ minimizations over half-spaces at each step to check elimination.

\begin{proposition}
An algorithm using the statistically-efficient implementation above never discards pieces later than (and thus never stops later than) the full elimination rule of \eqref{eq:elimination-sets} almost surely. 
\end{proposition}
\begin{proof}
If the algorithm eliminates a piece $\Lambda_p(i)$ at time $t$ with \eqref{eq:elimination-sets},
\begin{align*}
\beta_{t,\delta} \leq \inf_{\lambda \in \Lambda_p(i)} L_t(\hat{\theta}_t,\lambda) \leq \max_{j\neq p} \inf_{\lambda \in \Lambda_p(j)} L_t(\hat{\theta}_t,\lambda).
\end{align*}
This implies that the elimination condition in the statistically efficient implementation triggers as well.
\end{proof}

\paragraph{Selective elimination (computationally-efficient implementation)} 

An even simpler implementation is to check elimination, at each time $t$, only for pieces related to the empirical optimal arm $i^\star(\hat{\theta}_t)$. That is, we update $\cI_t$ as
\begin{align*}
\mathcal{I}_t := \mathcal{I}_{t-1} \setminus \left\{ j\in\cI_{t-1} \big| \inf_{\lambda \in \Lambda_j(i^\star(\hat{\theta}_t))} L_t(\hat{\theta}_t,\lambda) \geq \beta_{t,\delta}\right\}.
\end{align*}
This requires linear (in $K$) per-round memory and time complexity. Moreover, checking this stopping rule is more time-efficient than the LLR one. The latter requires to perform $K-1$ tests at each step, while the elimination one only performs $O(|\cI_{t-1}|)$ tests at each round $t$, thus becoming faster as arms are eliminated.

\begin{proposition}
An algorithm using the computationally-efficient implementation above never discards pieces later than (and thus never stops later than) the selective elimination rule almost surely. 
\end{proposition}
\begin{proof}
If the algorithm eliminates a piece of the empirical optimal arm at time $t$ using the selective elimination rule, the arm corresponding to that piece is also eliminated from the set $\cI_t$ above.
\end{proof}


\subsection{Top-m Identification}

In top-$m$ identification, the goal is to find the $m > 0$ arms with largest mean. BAI is therefore a special case of this problem when $m=1$. The set of answers is $\cI = \{\cS \subseteq [K] : |\cS|=m\}$ with size $|\cI| = {K \choose m}$ and the correct answer of $\theta\in\cM$ is $i^\star(\theta) = \argmax_{k\in[K]}^m\theta^T\phi_k$, where we use $\argmax^m : \mathbb{R}^K \mapsto {K \choose m}$ to denote the function returning the set of $m$ largest values. 

\paragraph{Decomposition into pieces}

Let us denote each $i\in\cI$ as a tuple $i = (k_1,\dots, k_m)$ of $m$ arms. Similarly to BAI, it is known that the set of alternatives $\Lambda(i)$ can be decomposed into half-spaces \cite{reda2021dealing},
\begin{align*}
\Lambda(i) = \bigcup_{j\in i, k \in [K] \setminus i} \left\{ \lambda \in \cM : \lambda^T\phi_k > \lambda^T\phi_j\right\}.
\end{align*}
Therefore, we have $\cP(i) = i \times ([K] \setminus i)$ with $P_i = m(K-m)$ and $\Lambda_p(i) = \{\lambda\in\cM : \lambda^T\phi_k > \lambda^T\phi_j\}$ when $p = (j,k)$.

\paragraph{Closest alternatives}

Note that each set $\Lambda_p(i)$ is still a half-space of the same form as the one we have for BAI. Hence, the same closed form expression for the closest alternative derived for BAI can be adopted for top-$m$ identification (see the closed-form expressions in the previous section).

\paragraph{Efficient implementation}

Note that, differently from BAI, here the set of answers is of combinatorial size. It is therefore intractable to store and enumerate all sets of active pieces $\cP_t(i)$. However, thanks to the structure of the problem, this is not necessary and there exists an efficient implementation for the elimination stopping rule. First note that, while there are $m(K-m)$ pieces for each of $K \choose m$ possible answers, the total number of half-spaces is only $K(K-1)$, one for each couple of different arms. With some abuse of notation, let us denote by $\Lambda_{k,j} := \left\{ \lambda \in \cM : \lambda^T\phi_k > \lambda^T\phi_j\right\}$ the half-space associated with arms $k$ and $j$. The elimination stopping rule, which checks whether all the pieces in $\cP(i)$ for some answer $i$ have been discarded, is equivalent to checking whether there exist $m$ arms $k_1,\dots,k_m$ such that $\Lambda_{k,k_l}$ has been eliminated for all $k\notin \{k_1,\dots,k_m\}$ and $l\in[m]$. Therefore, in our implementations we will only store whether each half-space $\Lambda_{k,k_l}$ has been eliminated or not. As before, we now see two possible implementations, one more computationally efficient and the other more statistically efficient.

\paragraph{Full elimination (statistically-efficient implementation)}

The idea is to check, at each time step $t$, the elimination condition for all half-spaces which have not been previously discarded. In particular, for each arm $j\in[K]$ we keep a set $\cS_t(j)$ storing those arms which are ``worse'' than $j$. Formally, we initially set $\cS_0(j) = \emptyset$ and update it as
\begin{align*}
\cS_t(j) := \begin{cases}
	\cS_{t-1}(j) \cup \left\{ k\notin  \cS_{t-1}(j)\cup\{j\} \big| \inf_{\lambda \in \Lambda_{k,j}} L_t(\hat{\theta}_t,\lambda) \geq \beta_{t,\delta}\right\} &\text{if } |\cS_{t-1}(j)| < K-m, \\
	\cS_{t-1}(j) &\text{otherwise}.
\end{cases}
\end{align*}
That is, when a half-space $\Lambda_{k,j}$ is eliminated, we conclude that arm $k$ is ``worse'' than arm $j$ and thus add the former to $\cS_t(j)$. In order to decide when to stop, we use the following intuition: whenever we find that $|\cS_t(j)| \geq K-m$ for some arm $j\in[K]$, then we know that $j$ must be in the top-m arms of $\theta$ and we can thus stop updating the set $\cS_t(j)$. Therefore, we can stop whenever there exist $m$ arms satisfying this property. This can be checked efficiently by keeping track of how many arms reach the condition $|\cS_t(j)| \geq K-m$ and stopping when the number of such arms reaches $m$. This approach takes $O(K(K-m))$ memory in the worst-case to store the sets $\cS_t(j)$. At each step, it performs exactly $\sum_{j: |\cS_{t-1}(j)| < K-m} (K - |\cS_{t-1}(j)| - 1)$ minimizations over half-spaces to check the elimination conditions, which gives $O(K(K-1))$ time complexity in the worst-case.

\begin{proposition}
An algorithm using the statistically-efficient implementation above never discards pieces later than (and thus never stops later than) the full elimination rule of \eqref{eq:elimination-sets} almost surely.
\end{proposition}
\begin{proof}
Note that the elimination condition for single half-spaces is exactly the same in the general elimination rule of \eqref{eq:elimination-sets} and in its implementation above. If \eqref{eq:elimination-sets} eliminates a piece $\Lambda_p(i)$ at time $t$, this implies some half-space $\Lambda_{k,j}$ is eliminated. Then, we have two possible cases: if $|\cS_{t-1}(j)| < K-m$, then we have $k\in\cS_t(j)$ by the condition above, i.e., $k$ is detected as ``worse'' than $j$ and it will be never checked again. On the other hand, if $|\cS_{t-1}(j)| \geq K-m$, then $j$ has already been labeled as belonging to the final answers. Thus, no minimization over its corresponding half-spaces (including the one for $k$) will be checked anymore, which is the same as saying that $\Lambda_{k,j}$ has already been eliminated.
\end{proof}

\paragraph{Selective elimination (computationally-efficient implementation)} 

Similarly to what we did for BAI, the most computationally-efficient implementation consists in checking the elimination condition only for the alternative pieces (i.e., the half-spaces) of the empirical correct answer at each step. We modify the update rule of the statistically-efficient implementation as
\begin{align*}
	\cS_t(j) := \cS_{t-1}(j) \cup \left\{ k\notin i^\star(\hat{\theta}_t) \cup \cS_{t-1}(j) \big| \inf_{\lambda \in \Lambda_{k,j}} L_t(\hat{\theta}_t,\lambda) \geq \beta_{t,\delta}\right\}
\end{align*}
if $j \in i^\star(\hat{\theta}_t) \text{ and } |\cS_{t-1}(j)| < K-m$, and $S_t(j) := \cS_{t-1}(j)$.
That is, at each step we only check elimination for half-spaces associated with the top-m arms of $\hat{\theta}_t$, excluding those that have already been eliminated and those that have already reached the threshold for being among the final answer. Note that this implementation performs $\sum_{j\in i^\star(\hat{\theta}_t),|\cS_{t-1}(j)| < K-m}(K - |i^\star(\hat{\theta}_t) \cup \cS_{t-1}(j)|) \leq m(K-m)$ minimizations over half-spaces at each step $t$. In constrast, the LLR stopping rule always performs $m(K-m)$ minimizations and is thus less efficient.

\begin{proposition}
An algorithm using the computationally-efficient implementation above never discards pieces later than (and thus never stops later than) the selective elimination rule almost surely.
\end{proposition}
\begin{proof}
The proof is the same as for BAI: if an arm is discarded by the selective elimination rule, then it is also discarded from the sets above.
\end{proof}

\subsection{Thresholding Bandits}

In the thresholding bandit problem, the goal is to learn whether the mean of each arm is above or below some given threshold. As usual, without loss of generality, we shall take zero as our threshold, for which the problem reduces to learning the sign of the mean reward of each arm. Let $\sign(x) := \indi{x \geq 0}$. Then, the set of answers is $\cI = \{0,1\}^K$ with size $|\cI| = 2^K$. The correct answer of problem $\theta\in\cM$ is $i^\star(\theta) = (\sign(\theta^T\phi_k))_{k\in[K]}$. 

\paragraph{Decomposition into pieces}

For each $i\in \{0,1\}^K$ (represented as a $K$-dimensional binary vector), the set of alternatives $\Lambda(i)$ can be decomposed into pieces as
\begin{align*}
\Lambda(i) = \bigcup_{k\in[K]} \left\{\lambda\in\cM : \sign(\lambda^T\phi_k) \neq i^k \right\}.
\end{align*}
Therefore, we have $\cP(i) = [K]$ with $P_i = K$ and $\Lambda_p(i) = \{\lambda\in\cM : \sign(\lambda^T\phi_p) \neq i^p\}$. As for BAI, the computation of the closest alternative over such pieces can be performed efficiently.

\paragraph{Closest alternatives}

Let $\theta\in\cM$ and $N\in\mathbb{R}_{\geq 0}^K$. The computation of the closest alternatives over pieces $\Lambda_p(i)$ can be reduced to the following optimization problem. For any arm $k\in[K]$ and any $b\in\{0,1\}$, we need to find
\begin{align*}
\inf_{\lambda \in \cM : \sign(\lambda^T\phi_k) \neq b} \Vert {\theta} - \lambda \Vert_{V_N}^2.
\end{align*}
It is easy to see that this is zero when $b\neq\sign(\theta^T\phi_k)$ (since $\theta$ itself is feasible). In case $b=\sign(\theta^T\phi_k)$, for unstructured problems ($\cM=\mathbb{R}^K$), the solution is to take $\lambda$ equal to $\theta$ at all components except the $k$-th one, where it is set to zero. This gives
\begin{align*}
\inf_{\lambda \in \mathbb{R}^K : \sign(\lambda^T\phi_k) \neq b} \Vert {\theta} - \lambda \Vert_{V_N}^2 = \begin{cases}
	N^k(\theta^T\phi_k)^2 &\text{if } \sign(\theta^T\phi_k) = b, \\
	0 &\text{otherwise}.
\end{cases}
\end{align*}
In the linear case, again under the assumption that $V_N$ is positive definite, we get
\begin{align*}
\inf_{\lambda \in \mathbb{R}^d : \sign(\lambda^T\phi_k) \neq b} \Vert {\theta} - \lambda \Vert_{V_N}^2 = \begin{cases}
	\frac{(\theta^T\phi_k)^2}{\| \phi_k \|_{V_N^{-1}}^2} &\text{if } \sign(\theta^T\phi_k) = b, \\
	0 &\text{otherwise}.
\end{cases}
\end{align*}

\paragraph{Efficient implementation}

For this problem, implementing the general elimination stopping rule would require storing and iterating over $2^K K$ pieces, which is clearly intractable. However, this problem introduces a high redundancy in the alternative pieces that we can exploit for an efficient implementation which takes only linear (in $K$) time and space. Differently from BAI and top-m identificaiton, the procedure highlighted below is \emph{exactly} an implementation of the ``theoretical'' elimination rules presented in the main paper, with the full and selective elimination rules reducing to the same thing.

Note that, for any $p\in[K]$ and $i,j\in\{0,1\}^K$ such that $i^p = j^p$, we have $\Lambda_p(i) = \Lambda_p(j)$. That is, whenever we eliminate some piece $\Lambda_p(i)$ for $p\in[K]$ and $i\in\{0,1\}^K$, we actually eliminate all problems in $\cM$ whose sign of the $p$-th mean reward is different from $i^p$. In other words, we learn that the $p$-th position of the correct answer for $\theta$ is indeed $i^p$. Therefore, an efficient implementation is as follows: we keep a set $\cA_t$ of active arms (those for which we still have to learn the corresponding component in the correct answer). This set is initialized as $\cA_0 = [K]$ and updated as
\begin{align*}
\mathcal{A}_t := \mathcal{A}_{t-1} \setminus \left\{ j\in\cA_{t-1} \big| \max_{i\in\{0,1\}^K} \inf_{\lambda \in \Lambda_j(i)} L_t(\hat{\theta}_t,\lambda) \geq \beta_{t,\delta}\right\}.
\end{align*}
While the maximization over $2^K$ elements might appear intractable, the structure of the problem allows us to entirely avoid it. Note that, for fixed $j$, the sets $\Lambda_j(i)$ are fully specified by the $j$-th component of $i$. Moreover, the inf is zero whenever $i^j \neq \sign(\hat{\theta}_t^T\phi_j)$ since that would imply $\hat{\theta}_t \in \Lambda_j(i)$. Therefore, the elimination condition can be equivalently rewritten in the convenient form
\begin{align*}
\mathcal{A}_t := \mathcal{A}_{t-1} \setminus \left\{ j\in\cA_{t-1} \big| \inf_{\lambda \in \cM : \sign(\lambda^T\phi_j) \neq \sign(\hat{\theta}_t^T\phi_j) } L_t(\hat{\theta}_t,\lambda) \geq \beta_{t,\delta}\right\}.
\end{align*}
Moreover, whenever the elimination condition above triggers for some arm $j\in[K]$, we set a variable $S_j := \sign(\hat{\theta}_t^T\phi_j)$ with the correct sign for the $j$-th component. We stop whenever $\cA_t = \emptyset$ (i.e., when all signs have been learned) and return $\hat{i} := (S_1,S_2,\dots, S_K)$. Similarly to BAI, this requires to perform only $|\cA_{t-1}|$ tests at each step $t$. On the other hand, the LLR stopping rule would perform $K$ tests at each step.

%% file: sections/appendix_exp_family.tex
\section{Log-likelihood ratio in exponential families}
\label{sec:exponential_families}

We suppose in this section that all arms have distributions in a one-parameter exponential family (the same for all arms, for simpler notations). An arm distribution can thus be described by any one of three parameters: the arm feature vector $\phi_k \in \mathbb{R}^d$, the arm mean $\mu_k(\theta) = \phi_k^\top \theta$ and its natural parameter $\eta_k(\theta)$. These two last are functions of the model $\theta$.
For two models $\theta$ and $\lambda$, let $f$ be a function such that the KL between the arm distributions with those parameters is $d_f(\eta_k(\lambda), \eta_k(\theta))$, where $d_f$ is the Bregman divergence associated to $f$. Let $f^*$ be the convex conjugate of $f$. The Kullback-Leibler divergence between the arm distributions under models $\theta$ and $\lambda$ is also equal to $d_{f^*}(\mu_k(\theta), \mu_k(\lambda))$.

If $\eta_k(\theta)$ is the natural parameter of that arm $k$, we have $\mu_k(\theta) = f'(\eta_k(\theta))$, and since $(f^*)' = (f')^{-1}$ we have $\eta_k(\theta) = (f^*)'(\mu_k(\theta))$.

\begin{lemma}
For all $\theta, \lambda \in \mathcal M$, the quantity $L_t(\theta, \lambda) - \sum_{k=1}^K N_t^k \KL_k(\theta, \lambda)$ is a martingale if the observations come from the model $\theta$. This does not depend on the hypothesis that the distributions belong to an exponential family but only requires $\mathbb{P}_\theta \ll \mathbb{P}_\lambda$.
\end{lemma}
\begin{proof}
We can expand the LLR to obtain a sum over times and write the KL as an expected log-likelihood ratio,
\begin{align*}
L_t(\theta, \lambda) - \sum_{k=1}^K N_t^k \KL_k(\theta, \lambda)
&= \sum_{s=1}^t \left(\log \frac{d \mathbb{P}_\theta}{d \mathbb{P}_\lambda}(X_s^{k_s}) - \mathbb{E}_{X \sim \nu_{k_s}(\theta)}\left[ \log \frac{d \mathbb{P}_\theta}{d \mathbb{P}_\lambda}(X) \right] \right)
\: .
\end{align*}
The martingale property is then immediate.
\end{proof}

Many of our proofs depend on the informal statement that the martingale $L_t(\theta, \lambda) - \sum_{k=1}^K N_t^k \KL_k(\theta, \lambda)$ concentrates, and is thus a lower order term which is negligible for $t$ large enough.

\begin{lemma}\label{lem:LLR_exp_family}
For all $\theta, \lambda$ and all $X_{[t]}$,
\begin{align*}
L_t(\theta, \lambda)
&= \sum_{k=1}^K N_t^k \KL_k(\theta, \lambda) - \sum_{k=1}^K N_t^k (\eta_k(\lambda) - \eta_k(\theta)) (\hat{\mu}_{t,k} - \mu_k(\theta))
\: .
\end{align*}
\end{lemma}
\begin{proof}
Write $\log \frac{d \mathbb{P}_\theta}{d \mathbb{P}_\lambda}(X) = \eta(\theta) X - f(\eta(\theta)) - (\eta(\lambda) X - f(\eta(\lambda)))$ and develop the Bregman divergence (the KL) on the right.
\end{proof}

\begin{lemma}\label{lem:sub_gaussian_KL_bounds}
If the distribution of arm $k$ for model $\theta$ is $\sigma^2$-sub-Gaussian, then for all $\lambda$,
\begin{align*}
\frac{1}{2 \sigma^2} (\mu_k(\lambda) - \mu_k(\theta))^2
&\le \KL_k(\lambda, \theta)
\: , \\
\frac{1}{2 \sigma^2} \sum_{k=1}^K N_t^k (\mu_k(\hat{\theta}_t) - \mu_k(\theta))^2
&\le L_t(\hat{\theta}_t, \theta)
\end{align*}
\end{lemma}

\begin{proof}
We first prove that the sub-Gaussian hypothesis is equivalent to both these inequalities:
\begin{align*}
\forall \lambda, d_f(\eta_k(\lambda), \eta_k(\theta)) \le \frac{1}{2}\sigma^2 (\eta_k(\lambda) - \eta_k(\theta))^2
\: , \\
\forall \lambda, d_{f^*}(\mu_k(\lambda), \mu_k(\theta)) \ge \frac{1}{2 \sigma^2} (\mu_k(\lambda) - \mu_k(\theta))^2
\: .
\end{align*}
The first result is then a simple consequence of that second inequality and the equality $\KL(\lambda, \theta) = d_{f^*}(\mu_k(\lambda), \mu_k(\theta))$. The second result of the lemma can be obtained by applying the first one to $\lambda = \hat{\theta}_t$ for all arms, then summing over arms.

The cumulant generating function at parameter $\eta_k(\theta)$ is $\xi \mapsto d_f(\eta_k(\theta) + \xi, \eta_k(\theta))$. The sub-Gaussian hypothesis is that this function is lower than $\frac{1}{2}\sigma^2 \xi^2$. For the second inequality, we first remark that the convex conjugate of $\xi \mapsto d_f(\eta_k(\theta) + \xi, \eta_k(\theta))$ is $x \mapsto d_{f^*}(\mu_k(\theta) + x, \mu_k(\theta))$, and write
\begin{align*}
d_{f^*}(\mu_k(\lambda), \mu_k(\theta))
&= \sup_\xi \xi (\mu_k(\lambda) - \mu_k(\theta)) - d_f(\eta_k(\theta) + \xi, \eta_k(\theta))
\\
&\ge \sup_\xi \xi (\mu_k(\lambda) - \mu_k(\theta)) - \frac{1}{2}\sigma^2 \xi^2
\\
&= \frac{1}{2 \sigma^2} (\mu_k(\lambda) - \mu_k(\theta))^2 \: .
\end{align*}
\end{proof}

\begin{corollary}\label{cor:martingale_bound_exp_family}
If the distribution of arm $k$ for model $\theta$ is $\sigma^2$-sub-Gaussian, then for all $\lambda$,
\begin{align*}
\left\vert \sum_{k=1}^K N_t^k (\eta_k(\lambda) - \eta_k(\theta)) (\mu_k(\hat{\theta}_t) - \mu_k(\theta)) \right\vert
&\le 2 \sqrt{L_t(\hat{\theta}_t, \theta)} \sqrt{\sum_{k=1}^K N_t^k \frac{1}{2} \sigma^2 (\eta_k(\lambda) - \eta_k(\theta))^2 }
\: .
\end{align*}
\end{corollary}

Remark: for Gaussians with variance $\sigma^2$, we also have $\sum_{k=1}^K N_t^k \frac{1}{2} \sigma^2 (\eta_k(\lambda) - \eta_k(\theta))^2 = \sum_{k=1}^K N_t^k \KL_k(\theta, \lambda)$. But that is not the case in general, and the sub-Gaussian assumption tells us that the sum of squares is larger than a KL, while we would like the reverse inequality.

\begin{proof}
Apply the Cauchy-Schwarz inequality, then Lemma~\ref{lem:sub_gaussian_KL_bounds}:
\begin{align*}
\left\vert \sum_{k=1}^K N_t^k (\eta_k(\lambda) - \eta_k(\theta)) (\mu_k(\hat{\theta}_t) - \mu_k(\theta)) \right\vert
&\le \sqrt{ \sum_{k=1}^K N_t^k (\eta_k(\lambda) - \eta_k(\theta))^2 \sum_{k=1}^K N_t^k (\mu_k(\hat{\theta}_t)- \mu_k(\theta))^2 }
\\
&\le 2 \sqrt{L_t(\hat{\theta}_t, \theta)} \sqrt{\sum_{k=1}^K N_t^k \frac{1}{2} \sigma^2 (\eta_k(\lambda) - \eta_k(\theta))^2 }
\: .
\end{align*}
\end{proof}

%

\subsection{Log-likelihood ratio for Gaussian linear models}
\label{sec:log_likelihood_ratio}

\begin{lemma}\label{lem:llr-lin-gauss}
For any $\theta\in\cM$ and $k\in[K]$, let $\nu_k(\theta)$ be Gaussian with unit variance and linear mean $\mu_k(\theta) = \theta^T\phi_k$. Then for any $\theta,\lambda\in\cM$, any $t>0$ and any sequence of observations,
\begin{align*}
L_t(\theta,\lambda) &= \sum_{k\in[K]}N_t^k\KL_k(\theta,\lambda) - (\lambda - \theta)^\top V_t (\hat{\theta}_t - \theta)
\: .
\end{align*}
\end{lemma}
\begin{proof}
Apply Lemma~\ref{lem:LLR_exp_family} to the Gaussian case, where $\eta(\theta) = \mu(\theta) = \phi_k^\top \theta$.
\begin{align*}
L_t({\theta},\lambda)
&= \sum_{k\in[K]} N_t^k\KL_k(\theta, \lambda) - \sum_{k=1}^K N_t^k (\phi_k^\top\lambda - \phi_k^\top\theta) (\hat{\mu}_{t,k} - \phi_k^\top \theta)
\end{align*}
For Gaussian linear models, we have
\begin{align*}
\sum_{k=1}^K N_t^k (\phi_k^\top\lambda - \phi_k^\top\theta) \phi_k^\top \hat{\theta}_t
&= (\lambda - \theta)^\top V_{N_t} \hat{\theta}_t
= (\lambda - \theta)^\top \sum_{k=1}^K N_t^k \hat{\mu}_{t,k} \phi_k
= \sum_{k=1}^K N_t^k (\phi_k^\top \lambda - \phi_k^\top \theta) \hat{\mu}_{t,k}
\: .
\end{align*}
We can use this to replace the sum involving $\hat{\mu}_{t,k}$ in the expression of $L_t(\theta, \lambda)$ by one involving $\hat{\theta}_t$.
\begin{align*}
L_t({\theta},\lambda)
&= \sum_{k\in[K]} N_t^k\KL_k(\theta, \lambda) - (\lambda - \theta)^\top (\sum_{k=1}^K N_t^k \phi_k \phi_k^\top) (\hat{\theta}_{t} - \theta)
\\
&= \sum_{k\in[K]} N_t^k\KL_k(\theta, \lambda) - (\lambda - \theta)^\top V_t (\hat{\theta}_{t} - \theta)
\: .
\end{align*}
\end{proof}

\begin{corollary}\label{cor:llr-lin-gauss}
For the linear Gaussian model of Lemma \ref{lem:llr-lin-gauss}, for any $\lambda\in\cM$,
\begin{align*}
L_t(\hat{\theta}_t,\lambda) = \sum_{k\in[K]}N_t^k\KL_k(\hat{\theta}_t,\lambda).
\end{align*}
\end{corollary}

\begin{lemma}\label{lem:llr-to-kl-lin-gauss}
For any $\theta\in\cM$ and $k\in[K]$, let $\nu_k(\theta)$ be Gaussian with unit variance and linear mean $\mu_k(\theta) = \theta^T\phi_k$. Then, for any $\lambda\in\cM$ and $t>0$,
\begin{align*}
\left(\sqrt{\sum_{k\in[K]}N_t^k\KL_k(\theta,\lambda)} - \sqrt{L_t(\hat{\theta}_t,\theta)}\right)^2 \leq L_t(\hat{\theta}_t,\lambda) \leq \left(\sqrt{\sum_{k\in[K]}N_t^k\KL_k(\theta,\lambda)} + \sqrt{L_t(\hat{\theta}_t,\theta)}\right)^2.
\end{align*}
\end{lemma}

\begin{proof}
We decompose the LLR as
\begin{align*}
L_t(\hat{\theta}_t,\lambda)
&= L_t(\theta, \lambda) + L_t(\hat{\theta}_t, \theta)
\\
&= \sum_{k\in[K]}N_t^k\KL_k(\theta,\lambda) - (\lambda - \theta)^\top V_t (\hat{\theta}_t - \theta) + L_t(\hat{\theta}_t, \theta)
\end{align*}
The second term is bounded by Corollary~\ref{cor:martingale_bound_exp_family} (with $\sigma^2 = 1$) and the remark below it. We get
\begin{align*}
L_t(\hat{\theta}_t, \lambda)
&\le \sum_{k\in[K]}N_t^k\KL_k(\theta,\lambda) + 2 \sqrt{L_t(\hat{\theta}_t, \theta)} \sqrt{\sum_{k=1}^K N_t^k \KL_k(\theta, \lambda) } + L_t(\hat{\theta}_t, \theta)
\\
&= \left( \sqrt{\sum_{k=1}^K N_t^k \KL_k(\theta, \lambda) } + \sqrt{L_t(\hat{\theta}_t, \theta)} \right)^2
\: .
\end{align*}

The proof of the lower bound is similar.
\end{proof}

\subsection{Beyond Gaussians}
\label{sub:beyond_gaussians}

While the proofs in the next two sections are specialized to Gaussian rewards, it is possible to extend them to more exponential families under slight assumptions, similarly to what was done in \cite{degenne2019non}. If the arm distributions are known to belong to a $\sigma^2$-sub-Gaussian exponential family, with the additional restriction that the distribution parameters should belong to a compact subset of the open interval on which the family is defined, then there exists a constant $c$ such that 
\begin{align*}
\frac{1}{\sigma^2}\sum_{k=1}^K N_t^k \KL_k(\theta, \lambda)
\le \sum_{k=1}^K N_t^k \frac{1}{2} (\eta_k(\lambda) - \eta_k(\theta))^2
\le c\sum_{k=1}^K N_t^k \KL_k(\theta, \lambda) \: .
\end{align*}
And $\sum_{k=1}^K N_t^k \KL_k(\theta, \lambda)$ is also close to $\Vert \lambda - \theta \Vert_{V_t}^2$, up to constant factors.

We can then recover bounds on the likelihood ratio of the same shape as in Lemma~\ref{lem:llr-to-kl-lin-gauss}, up to constant factors depending on $c$ and $\sigma^2$. The proofs of Appendix~\ref{app:proofs-elim-stopping} and Appendix~\ref{app:proofs_of_sampling} then proceed similarly, up to the additional constants.

%% file: sections/appendx_elimination.tex

\section{Proofs of Section \ref{sec:elimination_stopping_rules}}\label{app:proofs-elim-stopping}

\subsection{Proof of Lemma \ref{lem:delta-correct}}

\begin{proof}
Fix any $\theta\in\mathcal{M}$ and let $\hat{i}$ be the answer returned by the algorithm at the stopping time $\tau$. Then,
\begin{align*}
\mathbb{P}\left(\hat{i} \neq i^\star(\theta)\right) 
&\stackrel{(a)}{\leq} \mathbb{P}\left(\exists i \neq i^\star(\theta) : \cP_\tau(i) = \emptyset\right)
\\
&\stackrel{(b)}{\leq} \mathbb{P}\left(\exists i \neq i^\star(\theta),\forall p\in\cP(i), \exists t \leq \tau : \inf_{\lambda \in \Lambda_{p}(i)} L_t(\hat{\theta}_t, \lambda) \geq \beta_{t,\delta}\right)
\\
&\stackrel{(c)}{\leq} \mathbb{P}\left(\exists i \neq i^\star(\theta),\forall p\in\cP(i), \exists t \geq 1 : \inf_{\lambda \in \Lambda_{p}(i)} L_t(\hat{\theta}_t, \lambda) \geq \beta_{t,\delta}\right)
\end{align*}
where (a) is from the definition of stopping rule, (b) is from the fact that, if $\cP_\tau(i)$ is empty, then all the pieces of $i$ have been eliminated at some times before $\tau$, and (c) follows trivially by relaxing the condition $t\leq \tau$ to $t\geq 1$. Take any $t\geq 1$. Now note that, for any wrong answer $i\neq i^\star(\theta)$, $\theta \in \Lambda(i)$. By definition of the decomposition into pieces of $\Lambda(i)$, this means that there exists $\bar{p}_i \in \cP(i)$ such that $\theta \in \Lambda_{\bar{p}_i}(i)$. Therefore, continuing the chain of inequalities above, we get that
\begin{align*}
\mathbb{P}\left(\hat{i} \neq i^\star(\theta)\right) 
&\stackrel{(d)}{\leq} \mathbb{P}\left(\exists i \neq i^\star(\theta), \exists t \geq 1 : \inf_{\lambda \in \Lambda_{\bar{p}_i}(i)} L_t(\hat{\theta}_t, \lambda) \geq \beta_{t,\delta}\right)
\\ &\stackrel{(e)}{\leq} \mathbb{P}\left(\exists t \geq 1 : L_t(\hat{\theta}_t, \theta) \geq \beta_{t,\delta}\right) 
\stackrel{(f)}{\leq} \delta,
\end{align*}
where (d) holds since the event under which all pieces for $i$ have been eliminated implies that $\bar{p}_i$ has been eliminated as well, (e) holds since $\theta \in \Lambda_{\bar{p}_i}(i)$, and (f) is from the assumption on threshold $\beta_{t,\delta}$. This concludes the proof.
\end{proof}

\subsection{Proof of Theorem \ref{th:elim-better-than-llr}}\label{app:monotonicity}

Theorem~\ref{th:elim-better-than-llr} was proved in the main text. We prove here a more general result about the monotonicity of the stopping time with respect to the piece decomposition.
We show that if two algorithms use the same sampling rules and use elimination stopping based on different piece decompositions, the algorithm using the finer decomposition (in the sense defined below) will stop earlier.

\begin{definition}
We say that a piece decomposition $(\Lambda_p(i))_{i \in \mathcal I,p \in \mathcal P(i)}$ is finer than another one $(\tilde{\Lambda}_p(i))_{i \in \mathcal I,p \in \tilde{\mathcal P}(i)}$ if for all $i \in \mathcal I, p \in \tilde{\mathcal P}(i)$, there exists a set $S \subseteq \mathcal P(i)$ such that $\tilde{\Lambda}_p(i) = \bigcup_{q \in S}\Lambda_q(i)$.
\end{definition}

\begin{theorem}\label{thm:elimination_mono_decomposition}
Let $\mathcal D = (\Lambda_p(i))_{i \in \mathcal I,p \in \mathcal P(i)}$ be a finer decomposition than $\tilde{\mathcal D} = (\tilde{\Lambda}_p(i))_{i \in \mathcal I,p \in \tilde{\mathcal P}(i)}$. For $i \in \mathcal I,\tilde{p} \in \tilde{\mathcal P}(i)$, let $\tau_{\tilde{p}}$ and $\tilde{\tau}_{\tilde{p}}$ be the times at which $\tilde{\Lambda}_{\tilde{p}}(i)$ is eliminated by the two corresponding algorithms (in the sense of Definition~\ref{def:elimination}). Then almost surely $\tau_{\tilde{p}} \le \tilde{\tau}_{\tilde{p}}$.
\end{theorem}
Roughly, if a piece in the tilde decomposition corresponds to several pieces in the other, then it is faster to eliminate it as several pieces than as one piece.

\begin{proof}
It is enough to prove that whenever the elimination stopping rule of $\tilde{\Lambda}_{\tilde{p}}(i)$ triggers for $\tilde{\mathcal D}$, it triggers for $\mathcal D$ too. Then, let $t\geq 1$ and suppose that
\begin{align*}
\inf_{\lambda \in \Lambda_{\tilde{p}}(i)} L_t(\hat{\theta}_t, \lambda)  \geq \beta_{t,\delta}.
\end{align*}
Let $S_{\tilde{p}} \subseteq \cP(i)$ be the set corresponding to $\tilde{p}$ in the definition of ``finer decomposition''. We first argue that
\begin{align*}
\inf_{\lambda \in \Lambda_p(i)} L_t(\hat{\theta}_t, \lambda)
&= \inf_{q\in S_p}\inf_{\lambda \in \Lambda_q(i)} L_t(\hat{\theta}_t, \lambda) \: .
\end{align*}
Indeed, this is simply writing the infimum over a union as the infimum of an infimum. Hence if the $\tilde{p}$ piece is eliminated in $\tilde{\mathcal D}$, we have
\begin{align*}
\inf_{q\in S_p}\inf_{\lambda \in \Lambda_q(i)} L_t(\hat{\theta}_t, \lambda)
\ge \beta_{t, \delta}
\end{align*}
and every piece $\Lambda_p(i)$ for $p \in S_{\tilde{p}}$ is eliminated as well in $\mathcal D$. We get that the set $\tilde{\Lambda}_{\tilde{p}}$ is also eliminated in $\mathcal D$.
\end{proof}

Corollary: the finest possible decomposition is the one in which $\mathcal P(i) = \Lambda(i)$ and $\Lambda_p(i) = \{p\}$, i.e. every point of $\Lambda(i)$ is its own piece. This is not a computationally usable decomposition, but its the theoretically best one for the sample complexity metric (for a fixed sampling rule).

Theorem~\ref{th:elim-better-than-llr} for full elimination compared to LLR stopping follows from Theorem~\ref{thm:elimination_mono_decomposition} by setting $\tilde{\mathcal P}(i) = \{0\}$ and $\tilde{\Lambda}_p(i) = \Lambda(i)$. Then the elimination stopping rule uses a finer decomposition than the LLR stopping rule.

%

\subsection{Proof of Theorem \ref{th:piece-elim}}

We first present two important lemmas which will be used to prove the main statement for full elimination and selective elimination, respectively. For full elimination, the following lemma shows that if a piece has not been eliminated the information collected by the algorithm about it must be small.

%
%
%
%
%

\begin{lemma}\label{lemma:upper-bound-sum-inf}
Consider an algorithm that uses the full elimination stopping rule \eqref{eq:elimination-sets}. Let $p\in \cP(i^\star(\theta))$, then, for each time $t$ such that $E_t$ (Equation \ref{eq:Et}) holds and $p\in \cP_t(i^\star(\theta))$,
\begin{align*}
\sqrt{\inf_{\lambda \in \Lambda_p(i^{\star})}\sum_{k\in[K]}N_t^k\KL_k(\theta,\lambda)} < \sqrt{\beta_{t,\delta}} + \sqrt{\beta_{t,1/t^2}}.
\end{align*}
On the other hand, if the algorithm uses the LLR stopping rule, for each time $t$ such that $E_t$ holds and the algorithm did not stop,
\begin{align*}
\sqrt{\inf_{\lambda \in \Lambda(i^{\star})}\sum_{k\in[K]}N_t^k\KL_k(\theta,\lambda)} < \sqrt{\beta_{t,\delta}} + \sqrt{\beta_{t,1/t^2}}.
\end{align*}
\end{lemma}
\begin{proof}
Since $p\in \cP_t(i^\star(\theta))$ (i.e., $p$ has not been eliminated at time $t$), we have from \eqref{eq:elimination-sets} and Lemma \ref{lem:llr-to-kl-lin-gauss} that
\begin{align*}
\beta_{t,\delta}
 &> \inf_{\lambda \in \Lambda_p(i^{\star})} L_t(\hat{\theta}_t, \lambda) \geq \inf_{\lambda \in \Lambda_p(i^{\star})}\left(\sqrt{\sum_{k\in[K]}N_t^k\KL_k(\theta,\lambda)} - \sqrt{L_t(\hat{\theta}_t,\theta)}\right)^2.
\end{align*}
By definition of event $E_t$, this implies
\begin{align*}
\sqrt{\inf_{\lambda \in \Lambda_p(i^{\star})}\sum_{k\in[K]}N_t^k\KL_k(\theta,\lambda)} < \sqrt{\beta_{t,\delta}} + \sqrt{\beta_{t,1/t^2}}.
\end{align*} 
This yields the first statement. The second result can be shown analogously by using the full set of alternatives to $\theta$.

\end{proof}

Selective elimination updates, at each step $t$, only the set of active pieces of the empirical answer $i^\star(\hat{\theta}_t)$. Therefore, to bound the elimination times of the pieces $\Lambda_p(i^\star)$ for $p\in\cP(i^\star)$ we need to show that $i^\star(\hat{\theta}_t) \neq i^\star$ after a certain time. We show that Assumption \ref{asm:sampling-rule} alone is sufficient to guarantee this.

\begin{lemma}\label{lem:empirical-vs-true-answer}
Consider a sampling rule satisfying Assumption \ref{asm:sampling-rule}. Under event $E_t$, a sufficient condition for $i^\star(\hat{\theta}_t) = i^\star$ is
\begin{align}
t \geq \frac{ 4\beta_{t,1/t^2} + R(\theta,t)}{H^\star(\theta)}.
\end{align}
\end{lemma}
\begin{proof}
First note that, under $E_t$, if
\begin{align*}
\inf_{\lambda \in \Lambda(i^{\star}(\hat{\theta}_t))} L_t(\hat{\theta}_t, \lambda) > \beta_{t,1/t^2},
\end{align*}
then $i^\star(\hat{\theta}_t) = i^\star$. In fact, if this was not the case, we would have $\theta \in \Lambda(i^{\star}(\hat{\theta}_t))$ and thus $L_t(\hat{\theta}_t, \theta) > \beta_{t,1/t^2}$, which is a contradiction with event $E_t$ itself. Let us now look for a sufficient condition on $t$ to satisfy this inequality. Take $t$ and suppose it does not satisfy it. Then,
\begin{align*}
\sqrt{\beta_{t,1/t^2}}
 \geq \sqrt{\inf_{\lambda \in \Lambda(i^{\star}(\hat{\theta}_t))} L_t(\hat{\theta}_t, \lambda)} 
 &\stackrel{(a)}{\geq} \sqrt{\inf_{\lambda \in \Lambda(i^{\star}(\hat{\theta}_t))}\sum_{k\in[K]}N_t^k\KL_k(\theta,\lambda)} - \sqrt{L_t(\hat{\theta}_t,\theta)}
 \\ &\stackrel{(b)}{\geq} \sqrt{\inf_{\lambda \in \Lambda(i^{\star})}\sum_{k\in[K]}N_t^k\KL_k(\theta,\lambda)} - \sqrt{\beta_{t,1/t^2}}
 \\ &\stackrel{(c)}{\geq} \sqrt{t H^\star(\theta) - R(\theta,t)} - \sqrt{\beta_{t,1/t^2}},
\end{align*}
where (a) is from Lemma \ref{lem:llr-to-kl-lin-gauss}, (b) from $E_t$ and the fact that either $i^\star(\hat{\theta}_t) = i^\star$ or the infimum is zero, and (c) from Assumption \ref{asm:sampling-rule}. Rearranging this inequality yields the desired condition on $t$.
\end{proof}

\begin{proof}[Proof of Theorem \ref{th:piece-elim}]
We start by proving the sample complexity bound for LLR stopping. Since it is a bound on the elimination time of the whole alternative $\Lambda(i^\star)$, by definition it also holds for the elimination times of its pieces obtained with either full or selective elimination. We then move to full elimination stopping and selective elimination stopping. Finally, we prove that the left-hand side of the maximum in \eqref{eq:elimination_time} is never larger than the bound for LLR stopping.

\paragraph{LLR stopping}

Let $t > 0$ such that $E_t$ holds and the algorithm did not stop. The second statement in Lemma \ref{lemma:upper-bound-sum-inf} yields
\begin{align*}
\sqrt{\inf_{\lambda \in \Lambda(i^{\star})}\sum_{k\in[K]}N_t^k\KL_k(\theta,\lambda)} < \sqrt{\beta_{t,\delta}} + \sqrt{\beta_{t,1/t^2}}.
\end{align*}
Moreover, from Assumption \ref{asm:sampling-rule}
\begin{align*}
\inf_{\lambda \in \Lambda(i^\star)}\sum_{k\in[K]}N_t^k \KL_k(\theta,\lambda) \geq t H^\star(\theta) - R(\theta,t).
\end{align*}
The combination of these two inequalities directly yields the stated inequality on $t$. The result in expectation is obtained by applying Lemma 19 in \cite{reda2021dealing} together with $\mathbb{P}(\neg E_t) \leq 1/t^2$.

\paragraph{Full elimination stopping}

Let $t > 0$ such that $E_t$ holds. By Assumption \ref{asm:sampling-rule},
\begin{align*}
H^\star(\theta) \leq \inf_{\lambda \in \Lambda(i^\star)} \sum_{k\in[K]}\frac{N_t^k}{t} \KL_k(\theta,\lambda) + \frac{R(\theta,t)}{t}
\end{align*}
This implies that $\frac{N_t}{t}\in \Omega_{R(\theta,t)/t}$.

Now fix a piece $p\in \cP(i^\star)$ such that $p\in \cP_t(i^\star)$. Under event $E_t$, we know from Lemma \ref{lemma:upper-bound-sum-inf} that
\begin{align*}
\sqrt{\inf_{\lambda \in \Lambda_p(i^{\star})}\sum_{k\in[K]}N_t^k\KL_k(\theta,\lambda)}
< \sqrt{\beta_{t,\delta}} + \sqrt{\beta_{t,1/t^2}}.
\end{align*}
Since $\frac{N_t}{t}\in \Omega_{R(\theta,t)/t}$,
\begin{align*}
\inf_{\lambda \in \Lambda_p(i^{\star})}\sum_{k\in[K]}N_t^k\KL_k(\theta,\lambda)
  \geq t \min_{\omega \in \Omega_{R(\theta,t)/t}}\inf_{\lambda \in \Lambda_p(i^\star)} \sum_{k\in[K]}\omega^k\KL_k(\theta,\lambda).
\end{align*}
Combining the last two displays, we obtain that, if $j$ is not eliminated at time $t$ and $E_t$ holds, $t$ itself must satisfy
\begin{align*}
t < \frac{\left(\sqrt{\beta_{t,\delta}} + \sqrt{\beta_{t,1/t^2}}\right)^2}{\min_{\omega \in \Omega_{R(\theta,t)/t}}\inf_{\lambda \in \Lambda_p(i^\star)} \sum_{k\in[K]}\omega^k\KL_k(\theta,\lambda)}.
\end{align*}
The result in expectation is obtained by applying Lemma 19 in \cite{reda2021dealing}  together with $\mathbb{P}(\neg E_t) \leq 1/t^2$.

\paragraph{Selective elimination stopping}

Given Lemma \ref{lem:empirical-vs-true-answer}, the proof of Theorem \ref{th:piece-elim} for elimination stopping is very simple. Simply take a time $t$ such that $E_t$ holds and which verifies the condition in Lemma \ref{lem:empirical-vs-true-answer}. Then, for such a $t$ the first claim of Lemma \ref{lemma:upper-bound-sum-inf} can be verified analogously since the empirical answer (the one for which the set of active pieces is updated) is exactly the correct answer. Given Lemma \ref{lemma:upper-bound-sum-inf}, the same derivation as in the proof for full elimination can be carried out. This yields the following sufficient condition on the time $t$ to eliminate a piece $p\in\cP(i^\star)$:
\begin{align*}
t \geq \max\left\{\frac{\left(\sqrt{\beta_{t,\delta}} + \sqrt{\beta_{t,1/t^2}}\right)^2}{\min_{\omega \in \Omega_{R(\theta,t)/t}}\inf_{\lambda \in \Lambda_p(i^\star)} \sum_{k\in[K]}\omega^k\KL_k(\theta,\lambda)} , \frac{ 4\beta_{t,1/t^2} + R(\theta,t)}{H^\star(\theta)}\right\}.
\end{align*}

\paragraph{Comparison of the bounds for full elimination and LLR stopping}

We finally prove that a sufficient condition for $t \geq \bar{t}_p$, with $\bar{t}_p$ associated to full elimination, is $t\geq\bar{t}$.

Take any $p\in \cP(i^\star)$. By definition of the set $\Omega_\epsilon$, we have that, for any $\omega\in\Omega_\epsilon$,
\begin{align*}
\inf_{\lambda \in \Lambda_p(i^\star)} \sum_{k\in[K]}\omega^k\KL_k(\theta,\lambda) \geq \inf_{\lambda \in \Lambda(i^\star)} \sum_{k\in[K]}\omega^k\KL_k(\theta,\lambda) \geq H^\star(\theta) - \epsilon.
\end{align*}
This implies that
\begin{align*}
\min_{\omega\in\Omega_\epsilon}\inf_{\lambda \in \Lambda_p(i^\star)} \sum_{k\in[K]}\omega^k\KL_k(\theta,\lambda) \geq \max_{\omega\in\Delta_K}\inf_{\lambda \in \Lambda(i^\star)} \sum_{k\in[K]}\omega^k\KL_k(\theta,\lambda) - \epsilon.
\end{align*}
A sufficient condition for satisfying the inequality for piece elimination is thus
\begin{align*}
t \geq \frac{\left(\sqrt{\beta_{t,\delta}} + \sqrt{\beta_{t,1/t^2}}\right)^2}{\max_{\omega\in\Delta_K}\inf_{\lambda \in \Lambda(i^\star)} \sum_{k\in[K]}\omega^k\KL_k(\theta,\lambda)- R(\theta,t)/t}.
\end{align*}
Let $\gamma \in (0,1)$. For $R(\theta,t)/t \leq \gamma H^\star(\theta)$ we have that
\begin{align*}
t \geq \frac{\left(\sqrt{\beta_{t,\delta}} + \sqrt{\beta_{t,1/t^2}}\right)^2}{(1-\gamma)\max_{\omega\in\Delta_K}\inf_{\lambda \in \Lambda(i^\star)} \sum_{k\in[K]}\omega^k\KL_k(\theta,\lambda)}
\end{align*}
suffices. Therefore, taking the maximum between the condition above and $t \geq \frac{R(\theta,t)}{\gamma H^\star(\theta)}$ and optimizing over $\gamma \in (0,1)$, the inequality for piece elimination is verified if
\begin{align*}
t \geq \frac{1}{H^\star(\theta)}\inf_{\gamma \in (0,1)}\max\left\{\frac{\left(\sqrt{\beta_{t,\delta}} + \sqrt{\beta_{t,1/t^2}}\right)^2}{(1-\gamma)},  \frac{R(\theta,t)}{\gamma } \right\}
\end{align*}
Optimizing over $\gamma$, which amounts to setting $\gamma = \frac{R(\theta,t)}{\left(\sqrt{\beta_{t,\delta}} + \sqrt{\beta_{t,1/t^2}}\right)^2 + R(\theta,t)}$, yields the desired statement.

\end{proof}

\subsection{Proof of Theorem \ref{th:elim-vs-llr-fixed-sampling}}\label{subapp:proof_elim_vs_llr}

\begin{proof}
Take any time $t$ and suppose that $E_t$ holds while the algorithm did not stop with $\tau_{\mathrm{llr}}$ yet. Using the first result of Theorem \ref{th:piece-elim} yields
\begin{align*}
t < H^\star(\theta)^{-1}\left( \left(\sqrt{\beta_{t,\delta}} + \sqrt{\beta_{t,1/t^2}}\right)^2 + R(\theta,t) \right)
\: .
\end{align*}
From Lemma 19 in \cite{reda2021dealing}, we get that $\mathbb{E}[\tau_{\mathrm{llr}}] \leq \bar{t} + 2$, where $\bar{t}$ is the first time that does not satisfy the inequality above (it is a function of $\theta$ and $\delta$). This means that
\begin{align*}
\mathbb{E}[\tau_{\mathrm{llr}}] \leq H^\star(\theta)^{-1}\left( \left(\sqrt{\beta_{t,\delta}} + \sqrt{\beta_{t,1/t^2}}\right)^2 + R(\theta,t) \right) + 3.
\end{align*}

Now let us take the same algorithm with $\tau_{\mathrm{elim}}$. We know it is $\delta$-correct, so the standard lower bound states that
\begin{align*}
\mathbb{E}[\tau_{\mathrm{elim}}] \geq H^\star(\theta)^{-1}\log(1/2.4\delta).
\end{align*}
We get that
\begin{align*}
&H^\star(\theta)\mathbb{E}[\tau_{\mathrm{llr}}]
\\
&\le H^\star(\theta)\mathbb{E}[\tau_{\mathrm{elim}}] + 3H^\star(\theta) + \left(\sqrt{\beta_{\bar{t},\delta}} + \sqrt{\beta_{\bar{t},1/\bar{t}^2}}\right)^2 + R(\theta,\bar{t}) - \log \frac{1}{2.4 \delta}
\\
&= H^\star(\theta)\mathbb{E}[\tau_{\mathrm{elim}}]+ 3H^\star(\theta) + \beta_{\bar{t},\delta} - \log\frac{1}{\delta} + 2 \sqrt{\beta_{\bar{t},\delta}\beta_{\bar{t},1/\bar{t}^2}} + \beta_{\bar{t},1/\bar{t}^2} + R(\theta,\bar{t}) + \log (2.4) \: .
\end{align*}
We now use the hypothesis $\beta_{\bar{t}, \delta} = \log\frac{1}{\delta} + \xi(\bar{t}, \delta)$:
\begin{align*}
H^\star(\theta)\mathbb{E}[\tau_{\mathrm{llr}}]
&\le H^\star(\theta)\mathbb{E}[\tau_{\mathrm{elim}}] + 3H^\star(\theta) + \xi(\bar{t}, \delta) + 2 \sqrt{\beta_{\bar{t},\delta}\beta_{\bar{t},1/\bar{t}^2}} + \beta_{\bar{t},1/\bar{t}^2} + R(\theta,\bar{t}) + \log (2.4) \: .
\end{align*}
\end{proof}

%% file: sections/appendix_sampling.tex
\section{Proofs of Section \ref{sec:elimination_at_sampling}}
\label{app:proofs_of_sampling}

In order to tune our resets, we need an hypothesis on the shape of the threshold, used to compare thresholds at different times.

\begin{assumption}\label{ass:beta_bounds}
There exist two positive reals $c_1$ and $c_2$  that may depend on the parameters of the problem known to the algorithm ($d$, $K$, etc.) and a function $\overline{\log}$ such that the threshold $\beta_{t,\delta}$ verifies
\begin{align*}
\overline{\log}(1/\delta) + c_1 \log(t)
\le \beta_{t, \delta}
\le \overline{\log}(1/\delta) + c_2 \log(t)
\: .
\end{align*}
Furthermore, $\overline{\log}$ verifies that there exists $x_0$ such that for all $a \ge 2$ and $x \ge x_0$, $\overline{\log}(x^a) \le a \overline{\log}(x)$. 
\end{assumption}
The function $\overline{\log}$ represents an almost logarithmic function and corresponds to the function of order $\log + \sqrt{\log}$ in \cite{abbasi2011improved} or $\log + d\log\log$ in \cite{tirinzoni2020asymptotically,reda2021dealing}.
Other thresholds have a $O(\log\log(t))$ dependence instead of $O(\log(t))$, and our analysis could be extended to them in a similar way.

With that assumption, we set $\bar{t}_0 = \max\{2, \sqrt{x_0}\}$ and
$
\alpha_{{t},\delta} := \left(\sqrt{\beta_{t,\delta} + (4 c_2 - c_1)\log(t)} + 4\sqrt{\frac{c_2}{c_1}\beta_{t,1/t^2}}\right)^2
\: 
$.

The reset times we defined are taylored to a threshold $\beta_{t,\delta}$ of order $\log(1/\delta) + O(\log(t))$ as in Assumption~\ref{ass:beta_bounds}.
They are chosen such that $\beta_{t,1/t^2}$ roughly doubles from one reset to the next.
If the thresold has a different dependence in $t$ (for example $O(\log\log t)$) then the reset times can be adapted.

\subsection{On Assumption \ref{ass:beta_bounds}}

We illustrate first how Assumption~\ref{ass:beta_bounds} covers the threshold obtained by the most common technique: first get a concentration bound valid for one time of the form $f(1/\delta)$, where usually $f(x) = \log(x) + C\log\log(x)$ (for $C$ a dimension-dependent constant), then obtain a time-uniform concentration threshold by setting for example $\beta_{t, \delta} = f(\frac{\pi^2 t^2}{6 \delta})$.

Let us suppose then that $\beta_{t, \delta} = f(\frac{C' t^2}{\delta})$ for $f(x) = \log(x) + C\log\log(x)$ and $C' \ge 1$. We will now prove that if verifies Assumption~\ref{ass:beta_bounds}.
\begin{align*}
\overline{\log}(C' t^2/\delta)
&= \log \frac{C' t^2}{\delta} + C\log(\log\frac{C' }{\delta} + \log t^2)
\\
&\ge \log \frac{C' }{\delta} + C\log\log\frac{C' }{\delta} + 2 \log(t)
\\
&= \overline{\log}\frac{1}{\delta} + 2 \log(t)
\: .
\end{align*}
where $\overline{\log}(x) = \log(C' x) + C\log\log(C' x)$. Using the concavity of $\log$, we also have an upper bound
\begin{align*}
\overline{\log}(C' t^2/ \delta)
&= \log \frac{C' t^2}{\delta} + C\log(\log\frac{C' }{\delta} + \log t^2)
\\
&= \log \frac{C' }{\delta} + C\log\log\frac{C' }{\delta} + 2 \log(t) + C(\log(\log\frac{C' }{\delta} + \log t^2) - \log\log\frac{C' }{\delta})
\\
&\le \overline{\log}\frac{1}{\delta} + 2 (C+1)\log(t)
\: .
\end{align*}
We have found a function $\overline{\log}$ such that for $c_1 = 2$ and $c_2=2(C+1)$,
\begin{align*}
\overline{\log}(\frac{1}{\delta}) + c_1 \log(t)
\le \beta_{t, \delta}
\le \overline{\log}(\frac{1}{\delta}) + c_2 \log(t)
\: .
\end{align*}
It remains to show the condition on $\overline{\log}$, i.e. find $x_0$ such that for $a \ge 2$ and $x \ge x_0$, $\overline{\log}(x^a) \le a \overline{\log}(x)$.
\begin{align*}
\overline{\log}(x^a)
= \log(C' x^a) + C \log\log(C' x^a)
&\le \log((C' x)^a) + C \log\log((C' x)^a)
\\
&= a \log (C' x) + C \log\log(C' x) + C \log a
\end{align*}
It remains to find $x_0$ such that for $x \ge x_0$, $\log\log(C' x) + \log a \le a \log\log(C' x)$. We find that $x_0 = \exp(a^{1/(a-1)})/C'$ is suitable. Since $a \ge 2$, we have $x_0 \le e^2/C'$.

\begin{lemma}\label{lem:new_beta-diff-between-phases}
Under Assumption~\ref{ass:beta_bounds}, for any time $t$ and $0 \leq j \leq j(t)$, 
\begin{align*}
\beta_{t,\delta}
&\leq \beta_{\overline{t}_{j},\delta} + (2^{j(t)+1-j} c_2 - c_1)\log(\overline{t}_{j}),
\: , \\
\beta_{t,1/t^2}
&\leq \frac{c_2}{c_1} 2^{j(t)-j+1} \beta_{\overline{t}_j,1/\overline{t}_j^2}.
\end{align*}
\end{lemma}
\begin{proof}
We have $\bar{t}_0 = \max\{2, \sqrt{x_0}\}$. Note that for all $j$, $\bar{t}_j = \bar{t}_0^{2^j}$.
By using first the upper bound of Assumption~\ref{ass:beta_bounds}, then $t \leq \bar{t}_0^{2^{j(t)+1}}$ and at the end the lower bound of Assumption~\ref{ass:beta_bounds},
\begin{align*}
\beta_{t, \delta}
\le \overline{\log}(\frac{1}{\delta}) + c_2 \log(t)
\le \overline{\log}(\frac{1}{\delta}) + c_2 \log(\bar{t}_0^{2^{j(t)+1}})
&= \overline{\log}(\frac{1}{\delta}) + c_2 \log((\bar{t}_0^{2^j})^{2^{j(t)+1-j}})
\\
&= \overline{\log}(\frac{1}{\delta}) + c_2 2^{j(t)+1-j} \log(\bar{t}_0^{2^j})
\\
&\le \beta_{\overline{t}_j, \delta} + \log(\overline{t}_j)(2^{j(t)+1-j} c_2 - c_1)
\: .
\end{align*}

We have $\overline{t}_j \ge \bar{t}_0 \ge \sqrt{x_0}$, hence the inequality $\overline{\log}(x^a) \le a\overline{\log}(x)$ can be used. Then
\begin{align*}
\beta_{t, 1/t^2}
\le \overline{\log}(t^2) + c_2 \log(t)
&\le \overline{\log}(\bar{t}_0^{2^{j(t)+2}}) + c_2 \log(\bar{t}_0^{2^{j(t)+1}})
\\
&= \overline{\log}((((\bar{t}_0^{2^j})^2)^{2^{j(t)+1-j}}) + c_2 \log((\bar{t}_0^{2^j})^{2^{j(t)+1-j}})
\\
&= \overline{\log}(((\overline{t}_j^2)^{2^{j(t)+1-j}}) + c_2 2^{j(t)+1-j} \log(\overline{t}_j)
\\
&\le 2^{j(t)+1-j}\overline{\log}((\overline{t}_j^2) + c_2 2^{j(t)+1-j} \log(\overline{t}_j)
\\
&\le \frac{c_2}{c_1}\left( 2^{j(t)+1-j}\overline{\log}((\overline{t}_j^2) + c_1 2^{j(t)+1-j} \log(\overline{t}_j) \right)
\\
&\le \frac{c_2}{c_1} 2^{j(t)+1-j} \beta_{\overline{t}_j, 1/\overline{t}_j^2}
\: .
\end{align*}

\end{proof}


\subsection{Proof of Theorem \ref{th:sampling-rule-with-elim}}

We now derive an important result for sampling rules combined with elimination (either full or selective). It essentially shows that we cannot eliminate the closest alternative piece to $\theta$ from the sampling rule without making the algorithm stop.

\begin{lemma}\label{lem:closest-alternatives-not-eliminated-in-sampling}
Let $t \ge \bar{t}_1$ be any time step at which the algorithm did not stop. Suppose that some piece index $p\in\cP(i^\star)$ of the true correct answer $i^\star$ has been eliminated from the sampling rule (i.e., $p\notin\cP_{t}^{\mathrm{smp}}(i^\star)$). Then, under event $E_t$ (see Equation \ref{eq:Et}),
\begin{align*}
\inf_{\lambda \in \Lambda_p(i^{\star})}\sum_{k\in[K]}N_t^k\KL_k(\theta,\lambda) > \inf_{\lambda \in \Lambda(i^{\star})}\sum_{k\in[K]}N_t^k\KL_k(\theta,\lambda).
\end{align*}
\end{lemma}
\begin{proof}
Let us proceed by contradiction: suppose that $p\notin\cP_{t}^{\mathrm{smp}}(i^\star)$ while $\inf_{\lambda \in \Lambda_p(i^{\star})}\sum_{k\in[K]}N_t^k\KL_k(\theta,\lambda) = \inf_{\lambda \in \Lambda(i^{\star})}\sum_{k\in[K]}N_t^k\KL_k(\theta,\lambda)$ and $E_t$ holds.

Since $\cP_{t}^{\mathrm{smp}}(i^\star)$ is the intersection of all active sets from $\overline{t}_{j(t)-1}$ to $t$, if $p\notin\cP_{t}^{\mathrm{smp}}(i^\star)$, then there exists $s$ with $\overline{t}_{j(t)-1} \leq s \leq t$ such that
\begin{align*}
\inf_{\lambda \in \Lambda_p(i)} L_s(\hat{\theta}_s,\lambda) \geq \alpha_{s,\delta}.
\end{align*}
Therefore,
\begin{align*}
\sqrt{\alpha_{s,\delta}}
 \stackrel{(a)}{\leq} \sqrt{\inf_{\lambda \in \Lambda_p(i^\star)} L_{s}(\hat{\theta}_{s},\lambda)}
 &\stackrel{(b)}{\leq}  \sqrt{\inf_{\lambda \in \Lambda_p(i^\star)}\sum_{k\in[K]}N_{s}^k\KL_k(\theta,\lambda)} + \sqrt{L_{s}(\hat{\theta}_{s},\theta)}
 \\ &\stackrel{(c)}{\leq}  \sqrt{\inf_{\lambda \in \Lambda_p(i^\star)}\sum_{k\in[K]}N_{s}^k\KL_k(\theta,\lambda)} + \sqrt{\beta_{t,1/t^2}}
 \\ &\stackrel{(d)}{\leq}  \sqrt{\inf_{\lambda \in \Lambda_p(i^\star)}\sum_{k\in[K]}N_{t}^k\KL_k(\theta,\lambda)} + \sqrt{\beta_{t,1/t^2}}
 \\ & \stackrel{(e)}{=}  \sqrt{\inf_{\lambda \in \Lambda(i^\star)}\sum_{k\in[K]}N_{t}^k\KL_k(\theta,\lambda)} + \sqrt{\beta_{t,1/t^2}},
\end{align*}
where (a) is from the elimination condition, (b) uses Lemma \ref{lem:llr-to-kl-lin-gauss}, (c) uses that $E_t$ holds, (d) uses that the number of pulls of each arm is non-decreasing in time, and (e) holds by our assumption. Recall that, for any $\theta,\lambda\in\mathbb{R}^d$, $\sum_{k\in[K]}N_t^k\KL_k(\theta,\lambda) = \frac{1}{2}\|\theta-\lambda\|_{V_t}^2$. Therefore, by the triangle inequality,
\begin{align*}
\sqrt{\sum_{k\in[K]}N_{t}^k\KL_k(\theta,\lambda)}
= \frac{1}{\sqrt{2}}\|\theta-\lambda\|_{V_t}
&\leq \frac{1}{\sqrt{2}}\|\theta-\hat{\theta}_t\|_{V_t} + \frac{1}{\sqrt{2}}\|\hat{\theta}_t-\lambda\|_{V_t}
\\
&= \sqrt{L_t(\hat{\theta}_t, \lambda)} + \sqrt{L_t(\hat{\theta}_t, \theta)}.
\end{align*}
Combining this with the previous chain of inequalities,
\begin{align*}
\sqrt{\alpha_{s,\delta}} \leq \sqrt{\inf_{\lambda \in \Lambda(i^\star)}L_t(\hat{\theta}_t, \lambda)} + \sqrt{L_t(\hat{\theta}_t, \theta)} + \sqrt{\beta_{t,1/t^2}}
 &\leq \sqrt{\inf_{\lambda \in \Lambda(i^\star)}L_t(\hat{\theta}_t, \lambda)} + 2\sqrt{\beta_{{t},1/{t}^2}}
 \\ &\leq \sqrt{\inf_{\lambda \in \Lambda(i^\star(\hat{\theta}_t))}L_t(\hat{\theta}_t, \lambda)} + 2\sqrt{\beta_{{t},1/{t}^2}},
\end{align*}
where we used again that $E_t$ holds to concentrate the LLR between $\hat{\theta}_t$ and $\theta$. The last inequality is easy to check since, if $i^\star \neq i^\star(\hat{\theta}_t)$, then $\inf_{\lambda \in \Lambda(i^\star)}L_t(\hat{\theta}_t, \lambda) = 0$. Finally, since the algorithm did not stop at $t$, it must be that $\inf_{\lambda \in \Lambda(i^\star)}L_t(\hat{\theta}_t, \lambda) < \beta_{t,\delta}$ (otherwise all alternative pieces of $i^\star(\hat{\theta}_t)$ would be eliminated at $t$). Therefore,
\begin{align*}
\sqrt{\alpha_{s,\delta}} < \sqrt{\beta_{t,\delta}} + 2\sqrt{\beta_{{t},1/{t}^2}}.
\end{align*}
Since $s\geq \overline{t}_{j(t)-1}$, from Lemma \ref{lem:new_beta-diff-between-phases} we have that
\begin{align*}
\beta_{t,\delta} &\leq \beta_{\overline{t}_{j(t)-1},\delta} + (4 c_2 - c_1)\log(\overline{t}_{j(t)-1}) \leq \beta_{s,\delta} + (4 c_2 - c_1)\log(s),
\\ \beta_{t,1/t^2} &\leq 4 \frac{c_2}{c_1} \beta_{\overline{t}_{j(t)-1},1/\overline{t}_{j(t)-1}^2} \leq 4 \frac{c_2}{c_1} \beta_{s,1/s^2}.
\end{align*}
Plugging this into our previous bound, we conclude that
\begin{align*}
\sqrt{\alpha_{s,\delta}} < \sqrt{\beta_{s,\delta} + (4 c_2 - c_1)\log(s)} + 4\sqrt{\frac{c_2}{c_1}\beta_{s,1/s^2}}.
\end{align*}
This is clearly a contradiction w.r.t. our definition of $\alpha_{t,\delta}$.
\end{proof}

\begin{proof}[Proof of Theorem \ref{th:sampling-rule-with-elim}]
Take any step $t \ge \bar{t}_1$ where $E_t$ holds. Lemma \ref{lem:closest-alternatives-not-eliminated-in-sampling} ensures that all pieces of $i^\star$ which are at minimal distance from $\theta$, i.e., such that
\begin{align*}
\inf_{\lambda \in \Lambda_p(i^{\star})}\sum_{k\in[K]}N_t^k\KL_k(\theta,\lambda) = \inf_{\lambda \in \Lambda(i^{\star})}\sum_{k\in[K]}N_t^k\KL_k(\theta,\lambda)
\end{align*}
are not eliminated for the sampling rule at time $t$, i.e., $p\in\cP_t^{\mathrm{smp}}(i^\star)$. This implies that
\begin{align*}
\min_{p\in\cP_{t}^{\mathrm{smp}}(i^\star)}\inf_{\lambda \in \Lambda_p(i^\star)} \sum_{k\in[K]}N_t^k\KL_k(\theta,\lambda) = \inf_{\lambda \in \Lambda(i^\star)} \sum_{k\in[K]}N_t^k\KL_k(\theta,\lambda).
\end{align*}
Thus, if Assumption \ref{asm:sampling-rule-v2} holds, it must be that
\begin{align*}
\inf_{\lambda \in \Lambda(i^\star)} \sum_{k\in[K]}N_t^k\KL_k(\theta,\lambda)
 &\geq \max_{\omega\in\Delta_K}\sum_{s=1}^t \min_{p\in\cP_{s-1}^{\mathrm{smp}}(i^\star)}\inf_{\lambda \in \Lambda_p(i^\star)} \sum_{k\in[K]}\omega^k\KL_k(\theta,\lambda) - R(\theta,t)
 \\ &\geq \max_{\omega\in\Delta_K}\sum_{s=1}^t \inf_{\lambda \in \Lambda(i^\star)} \sum_{k\in[K]}\omega^k\KL_k(\theta,\lambda) - R(\theta,t) = tH^\star(\theta) - R(\theta,t),
\end{align*}
where the second inequality is trivial from $\cP_{s-1}^{\mathrm{smp}}(i^\star) \subseteq \cP(i^\star)$ for all $s\geq 1$. Therefore, we proved that the condition of Assumption \ref{asm:sampling-rule} holds as well for all $t \ge \bar{t}_1$. We can now have it for all $t$ by adding $\bar{t}_1 H^\star(\theta)$ to $R(\theta, t)$ to obtain another regret function which verifies the condition of Assumption~\ref{asm:sampling-rule}. The second statement is a direct consequence of the fact that Theorem \ref{th:piece-elim} holds for any sampling rule that satisifes the latter assumption.
\end{proof}

%% file: sections/appendix_verify_assumption.tex

\section{Existing Algorithms Satisfy Assumption \ref{asm:sampling-rule} and \ref{asm:sampling-rule-v2}}\label{app:assumptions}

In this section, we show that existing sampling rules that target the optimal value of the max-inf game from the general lower bound of \cite{kaufmann2016complexity} satisfy Assumption \ref{asm:sampling-rule}. Moreover, we show that, when such sampling rules are combined with elimination (as in Section \ref{sec:elimination_at_sampling}), they also satisfy Assumption \ref{asm:sampling-rule-v2}. We do that explicitly for two algorithms: the game-theoretic approach based on no-regret learners of \cite{degenne2019non} and an optimistic variant of the Track-and-Stop algorithm of \cite{kaufmann2016complexity}. These two algorithms are sufficiently general to be representative of many existing approaches for pure exploration: the latter represents those that repeatedly solve the optimization problem from the lower bound to get optimal allocations, while the former represents those that solve such problem incrementally. We now describe these two algorithmic techniques. Then, we show, through a unified proof, that they satisfy Assumption \ref{asm:sampling-rule} and \ref{asm:sampling-rule-v2}.

\subsection{Sampling Rules}

We describe the sampling rules of interest while trying to keep some of their design choices (e.g., confidence intervals, tracking, optimization, etc.) as general as possible. We do this because these sampling rules have been adapted to different pure exploration problems and bandit structures in the literature, for which such components would be different. This will allow us to have unified proofs that are actually agnostic to the specific setting.

\subsection{Common Assumptions}

We first state some assumptions on the main common components of these algorithms.

\paragraph{Bounded closest alternatives}

First, we shall make a mild regularity assumption about the considered identification problem: the distance between any closest alternative and $\theta$ is bounded.
\begin{assumption}\label{asm:boundedness}
There exists a constant $B > 0$ such that, for any $\omega \in \mathbb{R}_{\geq 0}$ and any subset $\Lambda\subseteq\cM$ which is exactly the union of arbitrary pieces, there exists a closest alternative
\begin{align*}
\lambda_\omega \in \argmin_{\lambda\in\Lambda}\sum_{k=1}^K \omega^K \KL_k(\theta,\lambda)
\end{align*}
such that
\begin{align*}
\max_{k\in[K]} \KL_k(\theta,\lambda_\omega) \leq B.
\end{align*}
\end{assumption}
We note that this assumption is satisfied in the identification problems we consider (see Appendix \ref{app:problems}), for Gaussian rewards where pieces are half-spaces.

\paragraph{Tracking}

All the sampling rules we consider choose a sequence of proportions $(\omega_t)_{t\geq 1}$, where $\omega_t\in\Delta_K$, and use tracking to select the next arm to play based on these. Formally, the arm played at time $t$ is
\begin{align}\label{eq:tracking}
k_t := \mathrm{Track}\left(\sum_{s=1}^t\omega_s, N_{t-1}\right),
\end{align}
where $\mathrm{Track}: \mathbb{R}_+^K \times \mathbb{R}_+^K \rightarrow [K]$ is some tracking function. To gain generality, we shall keep the tracking function implicit in the remainder, while only requiring the following assumption (which is actually guaranteed by existing methods).
\begin{assumption}\label{asm:tracking}
There exists a constant $C_{\mathrm{track}} > 0$ such that
\begin{align*}
\forall t > 0, k\in[K] : N_t^k \geq \sum_{s=1}^t \omega_s^k -C_{\mathrm{track}}.
\end{align*}
For instance, the widely-adopted cumulative tracking,
\begin{align*}
\mathrm{Track}\left(\sum_{s=1}^t\omega_s, N_{t-1}\right) = \argmin_{k\in[K]} \left( N_{t-1}^k - \sum_{s=1}^t\omega_s^k \right),
\end{align*}
satisfies this assumption.
\end{assumption}

\paragraph{Confidence intervals}

These sampling rules maintain confidence intervals $(c_t^k)_{t\geq 1, k\in[K]}$ about the expected return of each arm. We shall also keep them implicit as their specific form depends on the bandit structure under consideration (e.g., linear vs unstructured). We will only require the following assumption, which is satisfied by common choices as described below.
\begin{assumption}\label{asm:confidence intervals}
Under event $E_t$,
\begin{align*}
\forall s \leq t: \KL_k(\hat{\theta}_{s},\lambda) - c_{s}^k \leq \KL_k(\theta,\lambda) \leq \KL_k(\hat{\theta}_{s},\lambda) + c_{s}^k.
\end{align*}
Moreover, there exists a sub-linear (in $t$) function $C_{\mathrm{conf}}(t)$ such that
\begin{align*}
\sum_{s=1}^t\sum_{k\in[K]}\omega_s^kc_{s-1}^k \leq C_{\mathrm{conf}}(t).
\end{align*}
\end{assumption}
In bandits with linear structure, the confidence intervals typically take the form $c_t^k \propto \|\phi_k\|_{V_t^{-1}}$. In this case finding an upper bound on $\sum_{s=1}^t\sum_{k\in[K]}\omega_s^kc_{s-1}^k$ would reduce to applying an elliptical potential lemma \cite{abbasi2011improved} plus the tracking property, and the resulting upper bound would be $C_{\mathrm{conf}}(t) \propto \sqrt{dt}$. In unstructured bandits we can take $c_t^k \propto 1/\sqrt{N_t^k}$ and finding $C_{\mathrm{conf}}(t)$ would require applying the standard pigeon-hole principle plus the tracking property, for which one obtains $C_{\mathrm{conf}}(t) \propto \sqrt{Kt}$.

\subsubsection{Optimistic Track-and-Stop}\label{sec:optimistic-ts}

Here we describe an optimistic variant of the Track-and-Stop algorithm by \cite{kaufmann2016complexity}. It was originally introduced by \cite{degenne2019non} in order to get rid of forced exploration, one of the main causes behind the poor empirical performance of Track-and-Stop. The idea is to solve, at each time step $t$, an optimistic variant of the optimization problem from the lower bound,
\begin{align}\label{eq:optimistic-ts-optimization}
\omega_t := \argmax_{\omega\in\Delta_K}\inf_{\lambda \in \Lambda(i^\star(\hat{\theta}_{t-1}))} \sum_{k\in[K]}\omega^k \left(\KL_k(\hat{\theta}_{t-1},\lambda) + c_{t-1}^k\right),
\end{align}
where $c_{t-1}^k$ is a per-arm confidence interval that satisfies Assumption \ref{asm:confidence intervals}, thus ensuring optimism $\KL_k(\hat{\theta}_{t-1},\lambda) + c_{t-1}^k \geq \KL_k(\theta,\lambda)$ with high probability. The solution to this optimization problem yields proportions $\omega_t$ which are tracked by the sampling rule. Then, the arm played at time $t$ is given by the tracking rule \eqref{eq:tracking}.

\paragraph{Combining with elimination}

In order to combine this sampling rule with elimination, we simply redefine the optimization problem as
\begin{align}\label{eq:optimistic-ts-optimization-elim}
\omega_t := \argmax_{\omega\in\Delta_K}\min_{p\in{\cP}_{t-1}^{\mathrm{smp}}(i^\star(\hat{\theta}_{t-1}))}\inf_{\lambda \in \Lambda_p(i^\star(\hat{\theta}_{t-1}))} \sum_{k\in[K]}\omega^k \left(\KL_k(\hat{\theta}_{t-1},\lambda) + c_{t-1}^k\right).
\end{align}
That is, we simply replace a minimization over the whole alternative with one over the active pieces only.

\subsubsection{Game-Theoretic Approach with No-Regret Learners}\label{sec:no-regret}

The idea behind the game-theoretic approach of \cite{degenne2019non} is to avoid recomputing the full optimistic problem \eqref{eq:optimistic-ts-optimization} at each step, while solving it incrementally by means of no-regret learners. Given some online-learning algorithm $\cL$ working on the $K$-dimensional simplex, the sampling rule works as follows. At each time $t$, first $\cL$ outputs new proportions $\omega_t$. Then, we compute the closest alternative
\begin{align*}
\hat{\lambda}_t := \argmin_{\lambda \in \Lambda(i^\star(\hat{\theta}_{t-1}))}\sum_{k\in[K]}\omega_t^k \KL_k(\hat{\theta}_{t-1},\lambda).
\end{align*}
Finally, $\cL$ is updated with the (concave) gain function
\begin{align*}
g_t(\omega) := \sum_{k\in[K]}\omega^k \left(\KL_k(\hat{\theta}_{t-1},\hat{\lambda}_t) + c_{t-1}^k\right).
\end{align*}
Finally, the sampling rule uses tracking exactly as in \eqref{eq:tracking} to decide the next arm to pull. As before, we will only require the tracking rule to satisfy Assumption \ref{asm:tracking} without specifying an explicit form. Similarly, we will keep the learner implicit as far as it satisfies the following no-regret property.
\begin{assumption}\label{asm:no-regret}
The learner $\mathcal{L}$ is no regret: there exists a sub-linear (in $t$) function $C_{\mathcal{L}}(t)$ such that, for any $t\geq 1$ and any sequence of gains $\{g_s(\omega)\}_{s\leq t}$,
$$
\max_{w\in\Delta_K}\sum_{s=1}^t \big( g_s(w) -g_s(w_s) \big) \leq C_{\mathcal{L}}(t) \: .
$$
\end{assumption}

\paragraph{Combining with elimination}

For the game-theoretic approach, we only need to redefine the closest alternative used in the gains as
\begin{align*}
(\hat{p}_t, \hat{\lambda}_t) := \argmin_{p\in{\cP}_{t-1}^{\mathrm{smp}}(i^\star(\hat{\theta}_{t-1})), \lambda \in \Lambda_p(i^\star(\hat{\theta}_{t-1}))}\sum_{k\in[K]}\omega_t^k \KL_k(\hat{\theta}_{t-1},\lambda).
\end{align*}

\subsection{Assumption \ref{asm:sampling-rule} Holds}

We need to show that, for any time $t \geq 1$ where the good event $E_t$ (Equation \ref{eq:Et}) holds, the two sampling rules presented above satisfy
\begin{align*}
t H^\star(\theta) \leq \inf_{\lambda \in \Lambda(i^\star)} \sum_{k\in[K]}N_t^k\KL_k(\theta,\lambda) + R(\theta,t)
\end{align*}
for suitable choices of the function $R(\theta,t)$. We shall write a unified proof for these two sampling rules while explicitly mentioning where they differ. Let us suppose that their algorithmic components satisfy the assumptions stated above, i.e., Assumption \ref{asm:tracking} for tracking, Assumption \ref{asm:confidence intervals} for the confidence intervals, and Assumption \ref{asm:no-regret} for the no-regret learner.

Take any time step $t\geq 1$ and suppose that $E_t$ holds. Let $\lambda_t := \argmin_{\lambda \in \Lambda(i^\star)} \sum_{k\in[K]}N_t^k\KL_k(\theta,\lambda)$. Then, using the tracking property (Assumption \ref{asm:tracking}) together with Assumption \ref{asm:boundedness},
\begin{align*}
\inf_{\lambda \in \Lambda(i^\star)} \sum_{k\in[K]}N_t^k\KL_k(\theta,\lambda)
 &= \sum_{k\in[K]}N_t^k\KL_k(\theta,\lambda_t) 
 \\
 &\geq \sum_{k\in[K]}\sum_{s=1}^t\omega_s^k\KL_k(\theta,\lambda_t) - C_{\mathrm{track}}\sum_{k\in[K]}\KL_k(\theta,\lambda_t) 
 \\ &\geq \inf_{\lambda \in \Lambda(i^\star)}\sum_{k\in[K]}\sum_{s=1}^t\omega_s^k\KL_k(\theta,\lambda) - C_{\mathrm{track}}\sum_{k\in[K]}\KL_k(\theta,\lambda_t) 
  \\ &\geq \inf_{\lambda \in \Lambda(i^\star)}\sum_{k\in[K]}\sum_{s=1}^t\omega_s^k\KL_k(\theta,\lambda) - C_{\mathrm{track}}KB.
\end{align*}
We can now lower bound the first term as
\begin{align*}
\inf_{\lambda \in \Lambda(i^\star)}\sum_{k\in[K]}\sum_{s=1}^t\omega_s^k\KL_k(\theta,\lambda)
 &\stackrel{(a)}{\geq} \sum_{s=1}^t \inf_{\lambda \in \Lambda(i^\star)} \sum_{k\in[K]}  \omega_s^k\KL_k(\theta,\lambda)
 \\ &\stackrel{(b)}{\geq} \sum_{s=1}^t \inf_{\lambda \in \Lambda(i^\star(\hat{\theta}_{s-1}))} \sum_{k\in[K]} \omega_s^k\KL_k(\theta,\lambda)
 \\ &\stackrel{(c)}{\geq} \sum_{s=1}^t \inf_{\lambda \in \Lambda(i^\star(\hat{\theta}_{s-1}))} \sum_{k\in[K]} \omega_s^k \left( \KL_k(\hat{\theta}_{s-1},\lambda) - c_{s-1}^k \right)
 \\ &= \sum_{s=1}^t \inf_{\lambda \in \Lambda(i^\star(\hat{\theta}_{s-1}))} \sum_{k\in[K]} \omega_s^k \left( \KL_k(\hat{\theta}_{s-1},\lambda) + c_{s-1}^k \right) - 2\sum_{s=1}^t\sum_{k\in[K]}\omega_s^kc_{s-1}^k
 \\ &\stackrel{(d)}{\geq} \sum_{s=1}^t \inf_{\lambda \in \Lambda(i^\star(\hat{\theta}_{s-1}))} \sum_{k\in[K]} \omega_s^k \left( \KL_k(\hat{\theta}_{s-1},\lambda) + c_{s-1}^k \right) - 2C_{\mathrm{conf}}(t),
\end{align*}
where (a) is from the concavity of the infimum, (b) holds since either $\Lambda(i^\star(\hat{\theta}_{s-1})) = \Lambda(i^\star)$ or $\theta \in \Lambda(i^\star(\hat{\theta}_{s-1}))$ (in which case the infimum would be zero), (c) is from the validity of the confidence intervals under $E_t$, and (d) is from Assumption \ref{asm:confidence intervals}. Now note that, when applying the game-theoretic approach (Section \ref{sec:no-regret}), the first term on the right-hand side is exactly the sum of gains fed into the learner. Thus, using the no-regret property (Assumption \ref{asm:no-regret}),
\begin{align*}
&\inf_{\lambda \in \Lambda(i^\star)}\sum_{k\in[K]}\sum_{s=1}^t\omega_s^k\KL_k(\theta,\lambda)
\\
&\geq \max_{\omega\in\Delta_K}\sum_{s=1}^t \inf_{\lambda \in \Lambda(i^\star(\hat{\theta}_{s-1}))} \sum_{k\in[K]} \omega^k \left( \KL_k(\hat{\theta}_{s-1},\lambda) + c_{s-1}^k \right) - 2C_{\mathrm{conf}}(t) - C_{\cL}(t).
\end{align*}
If instead we are applying optimistic Track-and-Stop (Section \ref{sec:optimistic-ts}), the first term on the right-hand side is exactly the sum of optimal values of the objective functions maximized by the algorithm. Thus,
\begin{align*}
&\inf_{\lambda \in \Lambda(i^\star)}\sum_{k\in[K]}\sum_{s=1}^t\omega_s^k\KL_k(\theta,\lambda)
\\
&= \sum_{s=1}^t \max_{\omega\in\Delta_K}\inf_{\lambda \in \Lambda(i^\star(\hat{\theta}_{s-1}))} \sum_{k\in[K]} \omega^k \left( \KL_k(\hat{\theta}_{s-1},\lambda) + c_{s-1}^k \right) - 2C_{\mathrm{conf}}(t)
\\ &\geq \max_{\omega\in\Delta_K} \sum_{s=1}^t \inf_{\lambda \in \Lambda(i^\star(\hat{\theta}_{s-1}))} \sum_{k\in[K]} \omega^k \left( \KL_k(\hat{\theta}_{s-1},\lambda) + c_{s-1}^k \right) - 2C_{\mathrm{conf}}(t).
\end{align*}
Therefore, we only need to lower bound the first term above, which is common between the two considered algorithms. We have
\begin{align*}
&\max_{\omega\in\Delta_K} \sum_{s=1}^t \inf_{\lambda \in \Lambda(i^\star(\hat{\theta}_{s-1}))} \sum_{k\in[K]} \omega^k \left( \KL_k(\hat{\theta}_{s-1},\lambda) + c_{s-1}^k \right)
\\
&\stackrel{(e)}{\geq} \max_{\omega\in\Delta_K} \sum_{s=1}^t \inf_{\lambda \in \Lambda(i^\star)} \sum_{k\in[K]} \omega^k \left( \KL_k(\hat{\theta}_{s-1},\lambda) + c_{s-1}^k \right)
\\ &\stackrel{(f)}{\geq} \max_{\omega\in\Delta_K} \sum_{s=1}^t \inf_{\lambda \in \Lambda(i^\star)} \sum_{k\in[K]} \omega^k \KL_k({\theta},\lambda)
\\ &\stackrel{(g)}{\geq} t H^\star(\theta),
\end{align*}
where (e) follows from the same reasoning as step (b) above, (f) is from the fact that confidence intervals are valid under $E_t$, and (g) is from the definition of the game in the lower bound. Putting all together, we proved that optimistic Track-and-Stop satisfies Assumption \ref{asm:sampling-rule} with
\begin{align*}
R(\theta,t) = C_{\mathrm{track}}KB + 2C_{\mathrm{conf}}(t),
\end{align*}
while the game-theoretic approach satisfies it with
\begin{align*}
R(\theta,t) = C_{\mathrm{track}}KB + 2C_{\mathrm{conf}}(t) + C_{\cL}(t).
\end{align*}

\subsection{Assumption \ref{asm:sampling-rule-v2} Holds}

We now show that Assumption \ref{asm:sampling-rule-v2} holds for these sampling rules combined with elimination as formally explained above. The main steps are very similar as before, with the additional complications posed by eliminating pieces at sampling. 

Recall that we want to show that, for any time $t \geq 1$ where the good event $E_t$ (Equation \ref{eq:Et}) holds, the two sampling rules presented above satisfy
\begin{align*}
\max_{\omega\in\Delta_K}\sum_{s=1}^t \min_{p\in\cP_{s-1}^{\mathrm{smp}}(i^\star)}\inf_{\lambda \in \Lambda_p(i^\star)} \sum_{k\in[K]}\omega^k\KL_k(\theta,\lambda) \leq \min_{p\in\cP_{t}^{\mathrm{smp}}(i^\star)}\inf_{\lambda \in \Lambda_p(i^\star)} \sum_{k\in[K]}N_t^k\KL_k(\theta,\lambda) + R(\theta,t)
\end{align*}
for suitable choices of the function $R(\theta,t)$.

Take any time step $t\geq 1$ and suppose that $E_t$ holds. Let 
$$(p_t,\lambda_t) \in \argmin_{p\in\cP_{t}^{\mathrm{smp}}(i^\star), \lambda \in \Lambda_p(i^\star)} \sum_{k\in[K]}N_t^k\KL_k(\theta,\lambda)$$ be the closest piece and alternative at time $t$. Then, using the tracking property (Assumption \ref{asm:tracking}) and Assumption \ref{asm:boundedness},
\begin{align*}
\min_{p\in\cP_{t}^{\mathrm{smp}}(i^\star)}\inf_{\lambda \in \Lambda_p(i^\star)}\sum_{k\in[K]}N_t^k\KL_k(\theta,\lambda)
 &= \sum_{k\in[K]}N_t^k\KL_k(\theta,\lambda_t) 
\\ &\geq \sum_{k\in[K]}\sum_{s=1}^t\omega_s^k\KL_k(\theta,\lambda_t) - C_{\mathrm{track}}KB
 \\ &\geq \min_{p\in\cP_{t}^{\mathrm{smp}}(i^\star)}\inf_{\lambda \in \Lambda_p(i^\star)}\sum_{k\in[K]}\sum_{s=1}^t\omega_s^k\KL_k(\theta,\lambda) - C_{\mathrm{track}}KB.
\end{align*}
Recall that $\overline{t}_j := {\bar{t}_0}^{2^j}$ is the time step at which the $j$-th reset is performed and $j(t) := \lfloor \log_2\log_{\bar{t}_0} t \rfloor$ is the index of the last reset before $t$.  Let $\bar{t} := \overline{t}_{j(t)-1}$ be the time of the second-last reset before $t$. Note that
\begin{align*}
\bar{t} := {\bar{t}_0}^{2^{j(t)-1}} = \bar{t}_0^{\frac{1}{2}2^{\lfloor \log_{2}\log_{\bar{t}_0} t \rfloor}} = \sqrt{\bar{t}_0^{2^{\lfloor \log_{2}\log_{\bar{t}_0} t \rfloor}}} = \sqrt{\bar{t}_{j(t)}}\leq \sqrt{t}.
\end{align*}
We can now lower bound the first term as
\begin{align*}
&\min_{p\in\cP_{t}^{\mathrm{smp}}(i^\star)}\inf_{\lambda \in \Lambda_p(i^\star)}\sum_{k\in[K]}\sum_{s=1}^t\omega_s^k\KL_k(\theta,\lambda)
 \\ &\stackrel{(a)}{\geq} \sum_{s=1}^t \min_{p\in\cP_{t}^{\mathrm{smp}}(i^\star)}\inf_{\lambda \in \Lambda_p(i^\star)} \sum_{k\in[K]}  \omega_s^k\KL_k(\theta,\lambda)
 \\ &\stackrel{(b)}{\geq} \sum_{s=\bar{t}+1}^t \min_{p\in\cP_{t}^{\mathrm{smp}}(i^\star)}\inf_{\lambda \in \Lambda_p(i^\star)} \sum_{k\in[K]}  \omega_s^k\KL_k(\theta,\lambda)
 \\ &\stackrel{(c)}{\geq} \sum_{s=\bar{t}+1}^t \min_{p\in\cP_{t}^{\mathrm{smp}}(i^\star(\hat{\theta}_{s-1}))}\inf_{\lambda \in \Lambda_p(i^\star(\hat{\theta}_{s-1}))} \sum_{k\in[K]} \omega_s^k\KL_k(\theta,\lambda)
  \\ &\stackrel{(d)}{\geq} \sum_{s=1}^t \min_{p\in\cP_{s-1}^{\mathrm{smp}}(i^\star(\hat{\theta}_{s-1}))}\inf_{\lambda \in \Lambda_p(i^\star(\hat{\theta}_{s-1}))} \sum_{k\in[K]} \omega_s^k\KL_k(\theta,\lambda) - B\sqrt{t}
 \\ &\stackrel{(e)}{\geq} \sum_{s=1}^t \min_{p\in\cP_{s-1}^{\mathrm{smp}}(i^\star(\hat{\theta}_{s-1}))}\inf_{\lambda \in \Lambda_p(i^\star(\hat{\theta}_{s-1}))}  \sum_{k\in[K]} \omega_s^k \left( \KL_k(\hat{\theta}_{s-1},\lambda) - c_{s-1}^k \right) - B\sqrt{t}
 \\ &\stackrel{(f)}{\geq} \sum_{s=1}^t \min_{p\in\cP_{s-1}^{\mathrm{smp}}(i^\star(\hat{\theta}_{s-1}))}\inf_{\lambda \in \Lambda_p(i^\star(\hat{\theta}_{s-1}))}  \sum_{k\in[K]} \omega_s^k \left( \KL_k(\hat{\theta}_{s-1},\lambda) + c_{s-1}^k \right) - 2C_{\mathrm{conf}}(t) - B\sqrt{t},
\end{align*}
where (a) is from the concavity of the infimum, (b) drops the fist $\bar{t}$ rounds, (c) uses Lemma \ref{lem:change-inf-from-istar}, (d) uses that $\cP_t^{\mathrm{smp}}$ is contained in all active sets from $\bar{t}$ to $t$ and completes the sum with the first $\bar{t}$ rounds (while bounding $\bar{t} \leq \sqrt{t}$), (e) is from the validity of the confidence intervals under $E_t$, and (f) is from Assumption \ref{asm:confidence intervals}.

Now note that, when applying the game-theoretic approach (Section \ref{sec:no-regret}), the first term on the right-hand side is exactly the sum of gains fed into the learner. Thus, using the no-regret property (Assumption \ref{asm:no-regret}),
\begin{align*}
\min_{p\in\cP_{t}^{\mathrm{smp}}(i^\star)}&\inf_{\lambda \in \Lambda_p(i^\star)}\sum_{k\in[K]}\sum_{s=1}^t\omega_s^k\KL_k(\theta,\lambda) 
\\ &\geq \max_{\omega\in\Delta_K}\sum_{s=1}^t \min_{p\in\cP_{s-1}^{\mathrm{smp}}(i^\star(\hat{\theta}_{s-1}))}\inf_{\lambda \in \Lambda_p(i^\star(\hat{\theta}_{s-1}))}  \sum_{k\in[K]} \omega^k \left( \KL_k(\hat{\theta}_{s-1},\lambda) + c_{s-1}^k \right)
\\&\quad - 2C_{\mathrm{conf}}(t) - C_{\cL}(t)  - B\sqrt{t}.
\end{align*}
If instead we are applying optimistic Track-and-Stop (Section \ref{sec:optimistic-ts}), the first term on the right-hand side is exactly the sum of optimal values of the objective functions maximized by the algorithm. Thus,
\begin{align*}
&\min_{p\in\cP_{t}^{\mathrm{smp}}(i^\star)}\inf_{\lambda \in \Lambda_p(i^\star)}\sum_{k\in[K]}\sum_{s=1}^t\omega_s^k\KL_k(\theta,\lambda) 
\\ &\geq \sum_{s=1}^t \max_{\omega\in\Delta_K}\min_{p\in\cP_{s-1}^{\mathrm{smp}}(i^\star(\hat{\theta}_{s-1}))}\inf_{\lambda \in \Lambda_p(i^\star(\hat{\theta}_{s-1}))}  \sum_{k\in[K]} \omega^k \left( \KL_k(\hat{\theta}_{s-1},\lambda) + c_{s-1}^k \right) - 2C_{\mathrm{conf}}(t) - B\sqrt{t}
\\ &\geq \max_{\omega\in\Delta_K}\sum_{s=1}^t \min_{p\in\cP_{s-1}^{\mathrm{smp}}(i^\star(\hat{\theta}_{s-1}))}\inf_{\lambda \in \Lambda_p(i^\star(\hat{\theta}_{s-1}))}  \sum_{k\in[K]} \omega^k \left( \KL_k(\hat{\theta}_{s-1},\lambda) + c_{s-1}^k \right) - 2C_{\mathrm{conf}}(t) - B\sqrt{t}.
\end{align*}
Therefore, we only need to lower bound the first term above, which is common between the two considered algorithms. We have
\begin{align*}
&\max_{\omega\in\Delta_K} \sum_{s=1}^t \min_{p\in\cP_{s-1}^{\mathrm{smp}}(i^\star(\hat{\theta}_{s-1}))}\inf_{\lambda \in \Lambda_p(i^\star(\hat{\theta}_{s-1}))}  \sum_{k\in[K]} \omega^k \left( \KL_k(\hat{\theta}_{s-1},\lambda) + c_{s-1}^k \right)
\\ &\stackrel{(g)}{\geq}
\max_{\omega\in\Delta_K} \sum_{s=\bar{t}+1}^t \min_{p\in\cP_{s-1}^{\mathrm{smp}}(i^\star(\hat{\theta}_{s-1}))}\inf_{\lambda \in \Lambda_p(i^\star(\hat{\theta}_{s-1}))}  \sum_{k\in[K]} \omega^k \left( \KL_k(\hat{\theta}_{s-1},\lambda) + c_{s-1}^k \right)
\\ &\stackrel{(h)}{\geq} \max_{\omega\in\Delta_K} \sum_{s=\bar{t}+1}^t \min_{p\in\cP_{s-1}^{\mathrm{smp}}(i^\star)}\inf_{\lambda \in \Lambda_p(i^\star)}  \sum_{k\in[K]} \omega^k \left( \KL_k(\hat{\theta}_{s-1},\lambda) + c_{s-1}^k \right)
\\ &\stackrel{(i)}{\geq} \max_{\omega\in\Delta_K} \sum_{s=\bar{t}+1}^t \min_{p\in\cP_{s-1}^{\mathrm{smp}}(i^\star)}\inf_{\lambda \in \Lambda_p(i^\star)}  \sum_{k\in[K]} \omega^k  \KL_k({\theta},\lambda)
\\ &\stackrel{(j)}{\geq} \max_{\omega\in\Delta_K} \sum_{s=1}^t \min_{p\in\cP_{s-1}^{\mathrm{smp}}(i^\star)}\inf_{\lambda \in \Lambda_p(i^\star)}  \sum_{k\in[K]} \omega^k  \KL_k({\theta},\lambda) - B\sqrt{t},
\end{align*}
where (g) drops the first $\bar{t}$ rounds, (h) uses Lemma \ref{lem:change-inf-from-ihat}, (i) is from the fact that confidence intervals are valid under $E_t$, and (j) adds the missing first $\bar{t}$ rounds. Putting all together, we proved that optimistic Track-and-Stop satisfies Assumption \ref{asm:sampling-rule} with
\begin{align*}
R(\theta,t) = C_{\mathrm{track}}KB + 2C_{\mathrm{conf}}(t) + 2B\sqrt{t},
\end{align*}
while the game-theoretic approach satisfies it with
\begin{align*}
R(\theta,t) = C_{\mathrm{track}}KB + 2C_{\mathrm{conf}}(t) + C_{\cL}(t) + 2B\sqrt{t}.
\end{align*}

\begin{lemma}\label{lem:theta-not-elim-in-last-two-phases}
Under event $E_t$, for any $s\in\mathbb{N}$ with $\overline{t}_{j(t)-1} \leq s \leq t$,
\begin{align*}
L_s(\hat{\theta}_s,\theta) < \alpha_{s,\delta}.
\end{align*}
Moreover, for any $s,s'\in\mathbb{N}$ with $\overline{t}_{j(t)-1} \leq s' \leq s \leq t$,
\begin{align*}
L_{s'}(\hat{\theta}_{s'},\hat{\theta}_s) < \alpha_{s',\delta}.
\end{align*}
\end{lemma}
\begin{proof}
Using the good event $E_t$ followed by an application of Lemma \ref{lem:new_beta-diff-between-phases} together with $s \geq \overline{t}_{j(t)-1}$,
\begin{align*}
L_s(\hat{\theta}_s,\theta) \leq \beta_{t,1/t^2} \leq 4 \frac{c_2}{c_1}\beta_{\overline{t}_{j(t)-1},1/\overline{t}_{j(t)-1}^2} \leq 4\frac{c_2}{c_1}\beta_{s,1/s^2} < \alpha_{s,\delta}.
\end{align*}
This proves the first claim. To prove the second one, note that
\begin{align*}
\sqrt{L_{s'}(\hat{\theta}_{s'},\hat{\theta}_s)}
 \stackrel{(a)}{=} \sqrt{\sum_{k\in[K]}N_{s'}^k\KL_k(\hat{\theta}_{s'},\hat{\theta}_s)}
 &\stackrel{(b)}{\leq} \sqrt{\sum_{k\in[K]}N_{s'}^k\KL_k(\hat{\theta}_{s'},\theta)} + \sqrt{\sum_{k\in[K]}N_{s'}^k\KL_k(\hat{\theta}_s,\theta)}
 \\ &\stackrel{(c)}{\leq} \sqrt{\sum_{k\in[K]}N_{s'}^k\KL_k(\hat{\theta}_{s'},\theta)} + \sqrt{\sum_{k\in[K]}N_{s}^k\KL_k(\hat{\theta}_s,\theta)}
 \\ &\stackrel{(d)}{=} \sqrt{L_{s'}(\hat{\theta}_{s'},\theta)} + \sqrt{L_{s}(\hat{\theta}_{s},\theta)} \stackrel{(e)}{\leq} 2\sqrt{\beta_{t,1/t^2}},
\end{align*}
where (a) is from Corollary~\ref{cor:llr-lin-gauss}, (b) is from the triangle inequality (recall that the sum of KLs is a norm), (c) is from the fact that the pull counts are non-decreasing and $s\geq s'$, (d) is again from Corollary~\ref{cor:llr-lin-gauss}, and (e) is from event $E_t$. Using Lemma \ref{lem:new_beta-diff-between-phases} as before, we have $2\sqrt{\beta_{{t},1/{t}^2}} \leq 2\sqrt{\frac{c_2}{c_1}\beta_{s',1/s'^2}} < \sqrt{\alpha_{s',\delta}}$. This proves the second statement.
\end{proof}

\begin{lemma}\label{lem:change-inf-from-istar}
Under event $E_t$, for any $i\in\cI$ and $\omega\in\Delta_K$,
\begin{align*}
 \min_{p\in\cP_{t}^{\mathrm{smp}}(i^\star)}\inf_{\lambda \in \Lambda_p(i^\star)} \sum_{k\in[K]}  \omega^k\KL_k(\theta,\lambda)
 \geq \min_{p\in\cP_{t}^{\mathrm{smp}}(i)}\inf_{\lambda \in \Lambda_p(i)} \sum_{k\in[K]} \omega^k\KL_k(\theta,\lambda).
\end{align*}
\end{lemma}
\begin{proof}
The statement follows trivially if $i=i^\star$. So suppose $i\neq i^\star$. Since $i$ is not the answer of $\theta$, by the union property of the decomposition into pieces, there exists $p\in\cP(i)$ such that $\theta\in\Lambda_p(i)$. By Lemma \ref{lem:theta-not-elim-in-last-two-phases}, for any $\overline{t}_{j(t)-1} \leq s \leq t$,
\begin{align*}
\inf_{\lambda\in\Lambda_p(i)} L_{s}(\hat{\theta}_{s},\lambda) \leq L_{s}(\hat{\theta}_{s},\theta) < \alpha_{s,\delta}.
\end{align*}
This implies that $p\in\cP_{t}^{\mathrm{smp}}(i)$ since such set is defined as the intersection of all active sets from $\overline{t}_{j(t)-1}$ to $t$. Finally, we conclude that
\begin{align*}
\min_{p\in\cP_{t}^{\mathrm{smp}}(i)}\inf_{\lambda \in \Lambda_p(i)} \sum_{k\in[K]} \omega^k\KL_k(\theta,\lambda) \leq \sum_{k\in[K]} \omega^k\KL_k(\theta,\theta) = 0,
\end{align*}
and thus our result follows trivially.
\end{proof}

\begin{lemma}\label{lem:change-inf-from-ihat}
Under event $E_t$, for any $\overline{t}_{j(t)-1}  \leq s \leq t$, $i\in\cI$, and $\omega\in\Delta_K$,
\begin{align*}
 \min_{p\in\cP_{s}^{\mathrm{smp}}(i^\star(\hat{\theta}_{s}))}\inf_{\lambda \in \Lambda_p(i^\star(\hat{\theta}_{s}))}  \sum_{k\in[K]} \omega^k \KL_k(\hat{\theta}_{s},\lambda)
 \geq \min_{p\in\cP_{s}^{\mathrm{smp}}(i)}\inf_{\lambda \in \Lambda_p(i)} \sum_{k\in[K]} \omega^k\KL_k(\hat{\theta}_{s},\lambda).
\end{align*}
\end{lemma}
\begin{proof}
The proof is very similar to the one of Lemma \ref{lem:change-inf-from-istar}. The statement follows trivially if $i=i^\star(\hat{\theta}_{s})$. So suppose $i\neq i^\star(\hat{\theta}_{s})$. Since $i$ is not the answer of $\hat{\theta}_{s}$, by the union property of the decomposition into pieces, there exists $p\in\cP(i)$ such that $\hat{\theta}_{s}\in\Lambda_p(i)$. By Lemma \ref{lem:theta-not-elim-in-last-two-phases}, for any $\overline{t}_{j(t)-1} \leq s'\leq s \leq t$,
\begin{align*}
\inf_{\lambda\in\Lambda_p(i)} L_{s'}(\hat{\theta}_{s'},\lambda) \leq L_{s'}(\hat{\theta}_{s'},\hat{\theta}_s) < \alpha_{s',\delta}.
\end{align*}
This implies that $p\in\cP_{s}^{\mathrm{smp}}(i)$ since such set is defined as the intersection of all active sets from $\overline{t}_{j(t)-1}$ to $s$. Therefore, we conclude that
\begin{align*}
\min_{p\in\cP_{s}^{\mathrm{smp}}(i)}\inf_{\lambda \in \Lambda_p(i)} \sum_{k\in[K]} \omega^k\KL_k(\hat{\theta}_{s},\lambda) \leq \sum_{k\in[K]} \omega^k\KL_k(\hat{\theta}_{s},\hat{\theta}_{s}) = 0,
\end{align*}
and thus our result follows trivially.
\end{proof}

%% file: sections/appendix_experiments.tex

\section{Experiment Details and Additional Results}
\label{app:experiments}

\subsection{Reproducibility details}

We provide the main details to reproduce our experiments. For all the details, we refer the reader to our implementation at \url{https://github.com/AndreaTirinzoni/bandit-elimination}.

In all experiments, we used $\delta=0.01$ and a heuristic threshold $\beta_{t,\delta} = \log(1/\delta) + \log(1+t)$ for all elimination rules and LLR stopping. This is slightly larger than the heuristic threshold proposed by \cite{garivier2016optimal} and adopted in many recent works. We implemented the elimination rules as described in Appendix \ref{app:problems}. When using elimination at both stopping and sampling, we maintained only one set of active pieces (the one for stopping) instead of keeping a separate set with very lazy resets for sampling as suggested by theory. That set is shared by both sampling and stopping rules, and is never reset.

\paragraph{Computational infrastructure}

All experiments were run on a Dell XPS 13 laptop with an Intel Core i7-7560U (2.40GHz) CPU and 8GB of RAM.

\subsection{Bandit instances}\label{app:instances}

We provide details on how we generated the bandit instances considered in the experiments presented in the main paper and later in this section.

\paragraph{Linear instances (experiments of Figure \ref{fig:all} and Table \ref{tab:lin_elim})}

We set $K=50$ and $d=10$. The true parameter is $\theta = [1,1,\dots,1]^T$, while we generated the arm features randomly. The arm feature of the first arm is $\phi_1 = [1,0,0,\dots,0]^T$. Then, up to reaching 50 arms, we repeated the following procedure. First, we generated a 3-dimensional vector $v\in\mathbb{R}^3$ by drawing its elements uniformly in $[-1,1]^3$ and then normalizing to have unit norm. Then, we added 3 feature vectors $v_1 = [0,v,0,0,0,0,0,0]$, $v_2 = [0,0,0,0,v,0,0,0]$, and $v_3 = [0,0,0,0,0,0,0,v]$, but only if $v_1^T\theta \in [0, 0.8]$. In this way, we obtain linear instances with 50 arms where arm $1$ is optimal with value $\mu_1(\theta) = 1$, while all other arms have minimum sub-optimality of $0.2$ and maximum sub-optimality gap of $0.8$.

\paragraph{Linear instances (experiments of Table \ref{tab:lin_all})}

We set $K=50$ and $d=20$. The first 10 arms are set to the canonical basis of $\mathbb{R}^{10}$, i.e., $\phi_k = e_k$ for $k=1,\dots,10$. The generation of the true parameter $\theta$ and of the remaining 40 arms is slightly different from BAI/Top-m and OSI.

For BAI and Top-m, the true parameter $\theta$ has the first element equal to $1$, elements from the second to the fifth equal to $0.9$, and elements from the sixth to the tenth equal to $0.8$. The remaining 10 elements are uniformly drawn in $[-0.5,0.5]^{10}$. The remaining 40 arms are randomly generated as follows. First, we draw a vector $v$ uniformly in $[-1,1]^{20}$ and normalize it to have unit norm. Then, if $v^T\theta \leq 0.5$, we add $v$ to the set of arms. Otherwise, we reject the vector and keep repeating this procedure until we reach a total of 50 arms. In this way, we obtain random linear instances where the first arm is optimal with value $\mu_1(\theta) = 1$, the next 9 arms are hard to discriminate from it since they have small gap (either $0.1$ or $0.2$), and all remaining 40 arms have moderate to large gap (at least $0.5$) and are thus easy to eliminate.

For OSI, the true parameter $\theta$ has the first ten elements uniformly drawn in $([-0.2,-0.1]\cup[0.1,0.2])^{10}$ and the second ten elements uniformly drawn in $[-0.5,0.5]^{10}$. Similarly as before, to generate the remaining 40 arms we first draw a vector $v$ uniformly in $[-1,1]^{20}$ and normalize it to have unit norm. Then, if $|v^T\theta| \geq 0.5$, we add $v$ to the set of arms. Otherwise, we reject the vector and keep repeating this procedure until we reach a total of 50 arms. We thus obtain random linear instances where the first 10 arms are hard to learn since they have small gap (i.e., the absolute mean, which is between $0.1$ and $0.2$), and all remaining 40 arms have moderate to large gap (at least $0.5$) and are thus easy to eliminate.

\paragraph{Unstructured instances (experiments of Appendix \ref{app:uns_results})}

We used $K=40$ arms. For BAI and Top-m, the mean reward of the first 5 arms is $\mu_1 = 1$, $\mu_2 = 0.9$, $\mu_3 = 0.8$, $\mu_4 = 0.7$, and $\mu_5 = 0.6$. For all remaining arms the mean reward is uniformly drawn in $[0,0.5]$. For OSI, the mean reward of the first 4 arms is $\mu_1 = 0.1$, $\mu_2 = -0.2$, $\mu_3 = 0.3$, and $\mu_4 = -0.4$. For all remaining arms the mean reward is uniformly drawn in $[-1,-0.5]\cup[0.5,1]$.

\subsection{Additional Results}

\subsubsection{Full versus selective elimination}\label{app:full_vs_emp_results}

We report in Table \ref{tab:lin_elim} the full results of the experiment comparing elimination rules (full vs selective) from which we extracted Figure \ref{fig:all}\emph{(middle)}. We recall that the linear instances for this experiment were generated as explained in the first paragraph of Appendix \ref{app:instances}. We did not compare full and selective elimination rules on OSI since, as explained in Appendix \ref{app:problems}, they are actually equivalent in such a setting.

While we saw in Figure \ref{fig:all}\emph{(middle)} that the full elimination rule allows eliminating some arms earlier than the selective one, we notice from Table \ref{tab:lin_elim} that the former rule actually yields no advantage in terms of sample complexity. Moreover, its computational overhead makes it much slower than the selective elimination rule. Therefore, in practice we suggest using the selective elimination rule, which always yields reduced computation times and often improved sample complexity.

\begin{table*}[t!]
\centering
\small
\begin{tabular}{@{}clcccccc@{}} 
\toprule
 & & \multicolumn{2}{c}{No elim. (LLR)} & \multicolumn{2}{c}{Selective elim.} & \multicolumn{2}{c}{Full elim.} \\
\cmidrule(r){3-8}
& Algorithm & Samples & Time & Samples & Time & Samples & Time \\
\cmidrule{1-8}
\multirow{11}{*}{\rotatebox[origin=c]{90}{BAI}} 
& LinGapE & $4.51 \pm 1.3$ & $0.19$ & $4.49 \pm 1.3$ & $0.17$ & $4.49 \pm 1.3$ & $0.57$ \\
& LinGapE + elim & & & $4.16 \pm 1.4$ & $0.15$ & $4.17 \pm 1.4$ & $0.58$ \\
& LinGame & $5.28 \pm 1.7$ & $0.21$ & $5.09 \pm 1.8$ & $0.19$ & $5.09 \pm 1.8$ & $0.6$ \\
& LinGame + elim & & & $4.05 \pm 1.2$ & $0.17$ & $4.05 \pm 1.2$ & $0.65$ \\
& FWS & $4.68 \pm 4.2$ & $0.84$ & $4.68 \pm 4.2$ & $0.82$ & $4.68 \pm 4.2$ & $1.39$ \\
& FWS + elim & & & $4.21 \pm 1.4$ & $0.58$ & $4.21 \pm 1.4$ & $1.16$ \\
& Lazy TaS & $9.99 \pm 8.8$ & $0.45$ & $9.75 \pm 8.9$ & $0.45$ & $9.75 \pm 8.9$ & $0.78$ \\
& Lazy TaS + elim & & & $8.7 \pm 8.9$ & $0.38$ & $8.7 \pm 8.9$ & $0.73$ \\
& Oracle & $6.65 \pm 1.8$ & $0.04$ & $6.55 \pm 1.9$ & $0.02$ & $6.55 \pm 1.9$ & $0.31$ \\
& XY-Adaptive & & & & & $13.89 \pm 6.0$ & $2.23$ \\
& RAGE & & & & & $16.28 \pm 6.2$ & $0.02$ \\
\cmidrule{1-8}
\multirow{10}{*}{\rotatebox[origin=c]{90}{Top-m ($m=3$)}} 
& m-LinGapE & $6.26 \pm 1.2$ & $0.29$ & $6.21 \pm 1.2$ & $0.24$ & $6.21 \pm 1.2$ & $1.35$ \\
& m-LinGapE + elim & & & $5.77 \pm 1.2$ & $0.19$ & $5.77 \pm 1.2$ & $1.29$ \\
& MisLid & $7.06 \pm 1.4$ & $0.34$ & $6.81 \pm 1.5$ & $0.27$ & $6.81 \pm 1.5$ & $1.48$ \\
& MisLid + elim & & & $5.89 \pm 1.1$ & $0.22$ & $5.89 \pm 1.1$ & $1.42$ \\
& FWS & $5.91 \pm 1.7$ & $1.51$ & $5.9 \pm 1.7$ & $1.46$ & $5.9 \pm 1.7$ & $2.66$ \\
& FWS + elim & & & $5.84 \pm 1.7$ & $0.83$ & $5.84 \pm 1.7$ & $2.02$ \\
& Lazy TaS & $13.1 \pm 6.5$ & $0.71$ & $12.85 \pm 6.4$ & $0.67$ & $12.85 \pm 6.4$ & $1.57$ \\
& Lazy TaS + elim & & & $11.34 \pm 6.3$ & $0.56$ & $11.34 \pm 6.3$ & $1.47$ \\
& Oracle & $8.74 \pm 1.8$ & $0.1$ & $8.65 \pm 1.8$ & $0.04$ & $8.65 \pm 1.8$ & $1.02$ \\
& LinGIFA & $5.58 \pm 1.1$ & $1.8$ & $5.57 \pm 1.1$ & $1.75$ & $5.57 \pm 1.1$ & $2.68$ \\
\bottomrule
\end{tabular}
\caption{Experiments on linear instances with $K=50$ and $d=10$. The "Time" columns report average times per iteration in milliseconds (i.e., the total time the algorithm took divided by the number of samples). Each entry reports the mean across $100$ runs plus/minus standard deviation (which is omitted for compute times due to space constraints). The ``+ elim'' variant of some algorithms indicates that the corresponding sampling rule is combined with elimination. Samples are scaled down by a factor $10^3$.}\label{tab:lin_elim}
\end{table*}

\subsubsection{Unstructured instances}\label{app:uns_results}

We report the results on unstructured bandit instances (generated according to the procedure of Appendix \ref{app:instances}) in Table \ref{tab:uns_all}. The algorithm k-Learner is the unstructured variant of LinGame proposed by \cite{degenne2019non}. We note that the results are coherent with those for linear instances presented in the main paper. In particular, we observe a reduction in computation times when combining adaptive algorithms with selective elimation. The reduction is however less evident than in the linear case. This is expected since, in general, eliminations are easier in structured problems than in unstructured ones. We also note that combining sampling rules with selective elimination slightly improves the sample complexity of all algorithms.

\begin{table*}[t!]
\centering
\small
\begin{tabular}{@{}clcccccc@{}} 
\toprule
 & & \multicolumn{2}{c}{No elim. (LLR)} & \multicolumn{2}{c}{Elim. stopping} & \multicolumn{2}{c}{Elim. stopping + sampling} \\
\cmidrule(r){3-8}
& Algorithm & Samples & Time & Samples & Time & Samples & Time \\
\cmidrule{1-8}
\multirow{7}{*}{\rotatebox[origin=c]{90}{BAI}} 
& k-Learner & $18.76 \pm 6.5$ & $0.49$ & $18.12 \pm 6.6$ & $0.44$ & $14.82 \pm 4.6$ & $0.4$ \\
& FWS & $14.5 \pm 4.6$ & $1.25$ & $14.44 \pm 4.7$ & $1.24$ & $13.94 \pm 4.4$ & $1.16$ \\
& Lazy TaS & $26.18 \pm 8.0$ & $0.32$ & $24.78 \pm 7.7$ & $0.31$ & $20.66 \pm 7.4$ & $0.33$ \\
& Oracle & $27.49 \pm 3.7$ & $0.07$ & $27.0 \pm 3.7$ & $0.04$ & & \\
& LUCB & $14.2 \pm 5.2$ & $0.11$ & $14.18 \pm 5.2$ & $0.06$ & $13.57 \pm 4.7$ & $0.06$ \\
& UGapE & $15.13 \pm 5.0$ & $0.43$ & $15.13 \pm 5.0$ & $0.39$ & & \\
& Racing & & & & & $34.55 \pm 7.6$ & $0.01$ \\
\cmidrule{1-8}
\multirow{7}{*}{\rotatebox[origin=c]{90}{Top-m ($m=3$)}} 
& k-Learner & $25.84 \pm 6.2$ & $0.67$ & $25.06 \pm 6.3$ & $0.57$ & $17.65 \pm 4.9$ & $0.51$ \\
& FWS & $17.68 \pm 4.7$ & $2.52$ & $17.67 \pm 4.7$ & $2.48$ & $17.63 \pm 4.6$ & $2.0$ \\
& Lazy TaS & $38.89 \pm 10.4$ & $0.5$ & $37.84 \pm 10.7$ & $0.43$ & $27.74 \pm 6.4$ & $0.46$ \\
& Oracle & $34.17 \pm 4.5$ & $0.15$ & $33.68 \pm 4.9$ & $0.07$ & & \\
& LUCB & $17.61 \pm 4.6$ & $0.24$ & $17.58 \pm 4.5$ & $0.13$ & $17.15 \pm 5.4$ & $0.13$ \\
& UGapE & $17.87 \pm 4.3$ & $0.54$ & $17.87 \pm 4.3$ & $0.43$ & & \\
& Racing & & & & & $22.67 \pm 3.1$ & $0.01$ \\
\cmidrule{1-8}
\multirow{5}{*}{\rotatebox[origin=c]{90}{OSI}}
& k-Learner & $8.55 \pm 1.7$ & $0.61$ & $8.38 \pm 1.8$ & $0.55$ & $5.43 \pm 1.2$ & $0.52$ \\
& FWS & $5.54 \pm 1.3$ & $1.48$ & $5.53 \pm 1.3$ & $1.47$ & $5.47 \pm 1.4$ & $1.4$ \\
& Lazy TaS & $12.83 \pm 3.1$ & $0.72$ & $12.27 \pm 3.1$ & $0.7$ & $8.74 \pm 1.8$ & $0.79$ \\
& Oracle & $11.55 \pm 1.6$ & $0.07$ & $11.41 \pm 1.6$ & $0.04$ & & \\
& LUCB & $5.5 \pm 1.4$ & $0.11$ & $5.5 \pm 1.4$ & $0.1$ & $5.49 \pm 1.4$ & $0.1$ \\
\bottomrule
\end{tabular}
\caption{Experiments on unstructured instances with $K=40$. The "Time" columns report average times per iteration in milliseconds (i.e., the total time the algorithm took divided by the number of samples). Each entry reports the mean across $100$ runs plus/minus standard deviation (which is omitted for compute times due to space constraints). Algorithms for which the third column is missing cannot be combined with elimination at sampling, while algorithms for which the first two columns are missing are natively elimination-based. Samples are scaled down by a factor $10^3$.}\label{tab:uns_all}
\end{table*}

%% file: sections/appendix_revision.tex

\section{Examples}

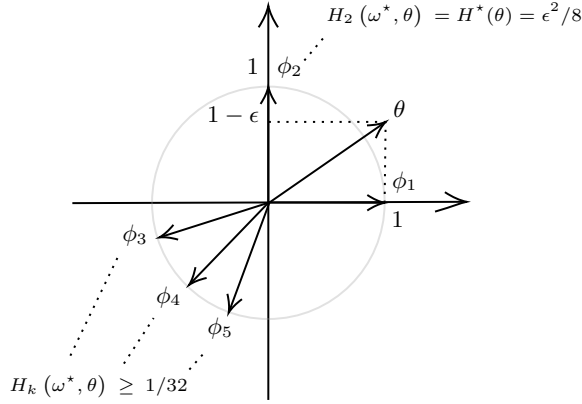
\begin{figure}[t]

\tikzset{every picture/.style={line width=0.75pt}} 
\centering
\begin{tikzpicture}[x=0.75pt,y=0.75pt,yscale=-1,xscale=1]

\draw    (211.5,131) -- (408.33,130.34) ;
\draw [shift={(410.33,130.33)}, rotate = 179.81] [color={rgb, 255:red, 0; green, 0; blue, 0 }  ][line width=0.75]    (15.93,-4.9) .. controls (6.95,-2.3) and (3.31,-0.67) .. (0,0) .. controls (3.31,0.67) and (6.95,2.3) .. (10.93,4.9)   ;
\draw    (310.33,230.33) -- (310.33,33.33) ;
\draw [shift={(310.33,31.33)}, rotate = 90] [color={rgb, 255:red, 0; green, 0; blue, 0 }  ][line width=0.75]    (10.93,-4.9) .. controls (6.95,-2.3) and (3.31,-0.67) .. (0,0) .. controls (3.31,0.67) and (6.95,2.3) .. (10.93,4.9)   ;
\draw  [color={rgb, 255:red, 155; green, 155; blue, 155 }  ,draw opacity=0.3 ][line width=0.75]  (251.69,130.83) .. controls (251.69,98.44) and (277.94,72.19) .. (310.33,72.19) .. controls (342.72,72.19) and (368.98,98.44) .. (368.98,130.83) .. controls (368.98,163.22) and (342.72,189.48) .. (310.33,189.48) .. controls (277.94,189.48) and (251.69,163.22) .. (251.69,130.83) -- cycle ;
\draw    (310.92,130.67) -- (366.98,130.83) ;
\draw [shift={(368.98,130.83)}, rotate = 180.16] [color={rgb, 255:red, 0; green, 0; blue, 0 }  ][line width=0.75]    (8.74,-3.92) .. controls (5.56,-1.84) and (2.65,-0.53) .. (0,0) .. controls (2.65,0.53) and (5.56,1.84) .. (8.74,3.92)   ;
\draw    (310.33,130.83) -- (310.33,74.19) ;
\draw [shift={(310.33,72.19)}, rotate = 90] [color={rgb, 255:red, 0; green, 0; blue, 0 }  ][line width=0.75]    (8.74,-3.92) .. controls (5.56,-1.84) and (2.65,-0.53) .. (0,0) .. controls (2.65,0.53) and (5.56,1.84) .. (8.74,3.92)   ;
\draw    (310.33,130.83) -- (257.24,147.73) ;
\draw [shift={(255.33,148.33)}, rotate = 342.35] [color={rgb, 255:red, 0; green, 0; blue, 0 }  ][line width=0.75]    (8.74,-3.92) .. controls (5.56,-1.84) and (2.65,-0.53) .. (0,0) .. controls (2.65,0.53) and (5.56,1.84) .. (8.74,3.92)   ;
\draw    (310.33,130.83) -- (271.72,170.89) ;
\draw [shift={(270.33,172.33)}, rotate = 313.95] [color={rgb, 255:red, 0; green, 0; blue, 0 }  ][line width=0.75]    (8.74,-3.92) .. controls (5.56,-1.84) and (2.65,-0.53) .. (0,0) .. controls (2.65,0.53) and (5.56,1.84) .. (8.74,3.92)   ;
\draw    (310.92,130.67) -- (291.03,184.46) ;
\draw [shift={(290.33,186.33)}, rotate = 290.29] [color={rgb, 255:red, 0; green, 0; blue, 0 }  ][line width=0.75]    (8.74,-3.92) .. controls (5.56,-1.84) and (2.65,-0.53) .. (0,0) .. controls (2.65,0.53) and (5.56,1.84) .. (8.74,3.92)   ;
\draw  [dash pattern={on 0.84pt off 2.51pt}]  (310.5,90) -- (369.5,90) ;
\draw  [dash pattern={on 0.84pt off 2.51pt}]  (368.98,130.83) -- (369.5,90) ;
\draw    (310.92,130.67) -- (367.86,91.14) ;
\draw [shift={(369.5,90)}, rotate = 145.23] [color={rgb, 255:red, 0; green, 0; blue, 0 }  ][line width=0.75]    (8.74,-3.92) .. controls (5.56,-1.84) and (2.65,-0.53) .. (0,0) .. controls (2.65,0.53) and (5.56,1.84) .. (8.74,3.92)   ;
\draw  [dash pattern={on 0.84pt off 2.51pt}]  (328.5,59) -- (339.5,46) ;
\draw  [dash pattern={on 0.84pt off 2.51pt}]  (210.5,208) -- (235.5,156) ;
\draw  [dash pattern={on 0.84pt off 2.51pt}]  (237.5,212) -- (252.5,189) ;
\draw  [dash pattern={on 0.84pt off 2.51pt}]  (270.5,215) -- (278.5,206) ;

\draw (370.98,133.83) node [anchor=north west][inner sep=0.75pt]  [font=\footnotesize] [align=left] {$\displaystyle 1$};
\draw (297.98,56.83) node [anchor=north west][inner sep=0.75pt]  [font=\footnotesize] [align=left] {$\displaystyle 1$};
\draw (277.98,80.83) node [anchor=north west][inner sep=0.75pt]  [font=\footnotesize] [align=left] {$\displaystyle 1-\epsilon $};
\draw (371.98,77.83) node [anchor=north west][inner sep=0.75pt]  [font=\footnotesize] [align=left] {$\displaystyle \theta $};
\draw (371.24,113.42) node [anchor=north west][inner sep=0.75pt]  [font=\footnotesize] [align=left] {$\displaystyle \phi _{1}$};
\draw (313.24,55.42) node [anchor=north west][inner sep=0.75pt]  [font=\footnotesize] [align=left] {$\displaystyle \phi _{2}$};
\draw (235.24,140.42) node [anchor=north west][inner sep=0.75pt]  [font=\footnotesize] [align=left] {$\displaystyle \phi _{3}$};
\draw (252.24,171.42) node [anchor=north west][inner sep=0.75pt]  [font=\footnotesize] [align=left] {$\displaystyle \phi _{4}$};
\draw (278.24,187.42) node [anchor=north west][inner sep=0.75pt]  [font=\footnotesize] [align=left] {$\displaystyle \phi _{5}$};
\draw (338.24,28.42) node [anchor=north west][inner sep=0.75pt]  [font=\scriptsize] [align=left] {$\displaystyle H_{2}\left( \omega ^{\star } ,\theta \right) \ =H^{\star }( \theta ) = \epsilon ^{2} /8$};
\draw (178.24,216.42) node [anchor=north west][inner sep=0.75pt]  [font=\scriptsize] [align=left] {$\displaystyle H_{k}\left( \omega ^{\star } ,\theta \right) \ \geq \ 1/32$};

\end{tikzpicture}
    \caption{BAI instance with $d=2$ and $K=5$. Arm $1$ is optimal with a mean reward of $1$, arm $2$ is $\epsilon$-suboptimal, while all other arms have a large sub-optimality gap. The optimal allocation $\omega^\star = (1/2,1/2,0,\dots)^T$ plays only arm 1 and arm 2 with the same proportions. The distance $H_2(\omega^{\star}, \theta )$ to the alternative piece of the second arm scales as $O(\epsilon^2)$ and fully controls the optimal sample complexity, which in turn is $\Omega(1/\epsilon^2)$. On the other hand, the distance $H_k(\omega^{\star},\theta)$ to the alternative piece of every other arm $k>2$ is large ($\Omega(1)$), which implies that such arms can be eliminated with sample complexity not scaling with $1/\epsilon^2$.}\label{fig:example}
\end{figure}

\subsection{Running example}\label{app:example}

Consider BAI in a Gaussian linear bandit instance with unit variance, $d=2$, and arbitrary number of arms $K \geq 3$ (see Figure \ref{fig:example} for an example with $K=5$). The arm features are $\phi_1 = (1,0)^T$, $\phi_2 = (0,1)^T$, and, for all $i = 3,\dots,K$, $\phi_i = (a_i,b_i)^T$ with $a_i,b_i$ arbitrary values in $(-1,0)$ such that $\|\phi_i\|_2=1$. The true parameter is $\theta = (1,1-\varepsilon)^T$, for $\varepsilon \in (0,1/2)$ a possibly very small value. Arm 1 is optimal with mean $\mu_1(\theta) = 1$, while arm 2 is sub-optimal with mean $\mu_2(\theta) = 1-\varepsilon$. For all other arms $i=3,\dots,K$, $\mu_i(\theta) \leq 0$.

Let $\omega \in \Delta_K$ be any allocation. Recall that in BAI each piece index is simply an arm, and $\cP(i^\star(\theta)) = \cP(1) = \{2,\dots,K\}$. Let $k\in\cP(1)$ be any sub-optimal arm. The distance to the $k$-th alternative piece $H_k(\omega, \theta)$ can be computed in closed form as
\begin{align*}
    H_k(\omega, \theta) = \frac{((\phi_1 - \phi_k)^T \theta)^2}{2\|\phi_1 - \phi_k\|_{V_\omega^{-1}}^2}.
\end{align*}
It can be determined that the optimal allocation, solution to $\argmax_\omega \min_k H_k(\omega, \theta)$, is $\omega^\star = (1/2,1/2,0,\dots,0)^\top$. We will prove this as a consequence of Lemma~\ref{lem:example-dist} below.

The intuition why this example is interesting is as follows. Any correct strategy is required to discriminate between arm 1 and 2 (i.e., to figure out that arm 1 is optimal), which requires roughly $O(1/\varepsilon^2)$ samples from both. An optimal strategy plays these two arms nearly with the same proportions. Since $\phi_1$ and $\phi_2$ form the canonical basis of $\mathbb{R}^2$, the samples collected by this strategy are informative for estimating the mean reward of \emph{every} arm, even those than are not played. Then, since arms $3,\dots,K$ have at least a sub-optimality gap of $1$, an elimination-based strategy discards them with a number of samples not scaling with $1/\varepsilon^2$. This means that a non-elimination strategy runs for $O(1/\varepsilon^2)$ steps over the original problem with $K$ arms, while an elimination-based one quickly reduces the problem to one with only 2 arms. The main impact is computational: since most algorithms need to compute some statistics for each active arm at each round (e.g., closest alternatives, confidence intervals, etc.) and that requires at least one loop over the set of active arms, the computational complexity of a non-elimination algorithm is at least $O(K/\varepsilon^2)$, while the one of an elimination-based variant is roughly $O(K + 1/\varepsilon^2)$, a potentially very large improvement.

\textbf{Remark. } \emph{The fact that certain arms are discarded very early is not only due to their large sub-optimality gap but also to the linear structure of the problem (i.e., to the fact that pulling certain arms provides information about others). In fact, if we consider an equivalent bandit problem with the same mean rewards but with $K$-dimensional features such that $\phi_k = e_k$ (the $k$-th vector of the canonical basis of $\mathbb{R}^K$) for all $k$, then pulling an arm provides no information about the others. In this unstructured case, it is known that an optimal allocation $\omega^\star$ makes all the half-spaces at the same distance. That is, $H_k(\omega, \theta)$ is the same for all $k$ and thus no arm is eliminated early.}

The following result formalizes the intuition that arms $3,\dots,K$ can be eliminated much earlier than arm $2$. It shows that, for any allocation $\omega$ and $k>2$, the distance $H_k(\omega, \theta)$ to the $k$-th alternative piece is at least a factor $1/\varepsilon^2$ larger than the distance $H_2(\omega, \theta)$ to the second alternative piece.

\begin{lemma}\label{lem:example-dist}
    For any $\omega$ and $k=3,\dots,K$, $H_k(\omega, \theta)
    \geq \frac{(1 - \varepsilon)^2}{\varepsilon^2} H_2(\omega, \theta)$.
\end{lemma}
\begin{proof}
For $k > 2$, $\Lambda_k(i^\star) \subseteq \{\eta \mid \eta_1 \le 0\} \cup \{\eta \mid \eta_2 \le 0\}$. Hence the distance from $\theta$ to $\Lambda_k(i^\star)$ in $V_\omega$-norm is greater than the minimal distance to each of those two half spaces. The closest points to $\theta$ in those two half spaces are $(0, 1-\varepsilon)^\top$ and $(1, 0)^\top$, respectively. We get
\begin{align*}
H_k(\omega, \theta)
&\ge \min \left\{ \frac{1}{2}\Vert \theta - (0, 1-\varepsilon)^\top \Vert_{V_\omega}^2, \frac{1}{2}\Vert \theta - (1, 0)^\top \Vert_{V_\omega}^2 \right\}
\\
&= \frac{1}{2}\min \{ \Vert \phi_1 \Vert_{V_\omega}^2, (1 - \varepsilon)^2\Vert \phi_2 \Vert_{V_\omega}^2 \}
\ge \frac{1}{2}(1 - \varepsilon)^2\min \{ \Vert \phi_1 \Vert_{V_\omega}^2, \Vert \phi_2 \Vert_{V_\omega}^2 \}.
\end{align*}

Since we have $2\times2$ matrices, let us get concrete and write $V_\omega = \left(\begin{array}{cc}a & b \\ b & d\end{array}\right)$. By our hypotheses on all $\phi_k$, $b \ge 0$. . We get that $d = 1 - a$ by computing the trace of $V_\omega$: $\mathrm{tr}(V_\omega) = \mathrm{tr}(\sum_k \omega^k \phi_k \phi_k^\top) = \sum_k \omega^k \mathrm{tr}(\phi_k \phi_k^\top) = \sum_k \omega^k \phi_k^\top\phi_k = 1$.

Then $V_\omega^{-1} = \frac{1}{a(1 - a)-b^2}\left(\begin{array}{cc}1 - a & -b \\ -b & a\end{array}\right)$, $\Vert \phi_1 \Vert_{V_\omega}^2 = a$, $\Vert \phi_2 \Vert_{V_\omega}^2 = 1 - a$, $\Vert \phi_1 - \phi_2\Vert^2_{V_\omega^{-1}} = \frac{1}{a(1 - a)-b^2}(1 + 2 b)$. Thus,
\begin{align*}
H_2(\omega, \theta)
= \frac{\varepsilon^2}{2\Vert \phi_1 - \phi_2\Vert^2_{V_\omega^{-1}}}
= \frac{\varepsilon^2 (a(1-a)-b^2)}{2(1 + 2 b)}
\le \frac{\varepsilon^2}{2} a(1-a),
\end{align*}
and
\begin{align*}
H_k(\omega, \theta)
&\ge \frac{1}{2}(1 - \varepsilon)^2\min \{ \Vert \phi_1 \Vert_{V_\omega}^2, \Vert \phi_2 \Vert_{V_\omega}^2 \}
= \frac{1}{2}(1 - \varepsilon)^2\min \{ 1 - a, a \}
\ge \frac{1}{2}(1 - \varepsilon)^2 a (1 - a).
\end{align*}

We finally obtain a comparison between $H_k(\omega,\theta)$ for $k > 2$ and $H_2(\omega, \theta)$:
\begin{align*}
H_k(\omega, \theta)
&\ge \frac{1}{2}(1 - \varepsilon)^2 a (1 - a)
\ge \frac{(1 - \varepsilon)^2}{\varepsilon^2} H_2(\omega, \theta)
\: .
\end{align*}

\end{proof}

\begin{corollary}
The optimal allocation for this problem is $\omega^\star = (1/2,1/2,0,\dots,0)^\top$.
\end{corollary}
\begin{proof}
By Lemma~\ref{lem:example-dist}, since $\varepsilon < 1/2$ the minimum over $H_k$ is attained by $H_2$ no matter what $\omega$ is, hence the optimal allocation belongs to $\argmax_\omega H_2(\omega, \theta)$. With the notation of that lemma, $H_2(\omega, \theta) = \frac{\varepsilon^2 (a(1-a)-b^2)}{2(1 + 2 b)}$ where $a \in [0,1], b \ge 0$ are coefficients of $V_\omega$. We get that (if that's attainable), the optimal value is reached for $a = 1/2, b = 0$. By construction of our arms $\phi_k$, $b=0$ implies that $\omega$ is supported on the two first coordinates. $a=1/2$ can then be attained by $(1/2,1/2,0,\dots,0)^\top$.
\end{proof}

\subsection{An oracle strategy with elimination at stopping provably reduces the computational complexity w.r.t. LLR stopping}

Consider an ``oracle'' strategy which tracks the optimal proportions $\omega^\star$ from the lower bound, i.e., the arm played at time $t$ is $k_t = \argmin_{k} (N_{t-1}^k - t\omega_k^\star)$. 

\begin{proposition}\label{prop:elim-time-oracle}
    For any $K \geq 3$ and $\varepsilon \in (0,1/2)$, the exists a Gaussian linear BAI instance with unit variance and $d=2$ such that, for any $\delta \in (0,1)$, the oracle strategy combined with LLR stopping satisfies
    \begin{align*}
        \mathbb{E}[\tau] \geq \Omega\left(\frac{\log(1/\delta)}{\varepsilon^2}\right).
    \end{align*}
    On the same instance, suppose we run the oracle strategy with elimination at stopping using a threshold $\beta_{t,\delta} = \log(1/\delta) + O(\log(t))$. Then, for all but 2 arms, the expected elimination time of the corresponding piece is, for full elimination,
    \begin{align*}
        \mathbb{E}[\tau_k] \leq \widetilde{O}(\log(1/\delta)),
    \end{align*}
    while for selective elimination,
    \begin{align*}
        \mathbb{E}[\tau_k] \leq \widetilde{O}\left(\log(1/\delta) + \frac{1}{\varepsilon^2}\right).
    \end{align*}
    Here $\widetilde{O}$ hides constant and logarithmic terms.
\end{proposition}
\begin{proof}
    Consider the problem defined in the example of Appendix \ref{app:example}. It is easy to see that, on such instance, the oracle sampling rule simply alternates between pulling arm $1$ and arm $2$. The first result is immediate from the lower bound $\mathbb{E}[\tau] \geq \log\left(\frac{1}{2.4\delta}\right) / H^\star(\theta)$. We only need to compute $H^\star(\theta) = \max_\omega \min_k H_k(\omega,\theta)$. By Lemma~\ref{lem:example-dist}, since $\varepsilon < 1/2$ the minimum is equal to $H_2(\omega,\theta)$. Since $\omega^\star = (1/2,1/2,0,\dots,0)^T$, using the closed-form expression of the infimum over each alternative piece as stated above,
    \begin{align*}
    H^\star(\theta) 
    = \max_{\omega} H_2(\omega, \theta)
    = \frac{1}{2}\frac{\varepsilon^2}{\frac{1}{\omega_1^\star} + \frac{1}{\omega_2^\star}}
    = \frac{\varepsilon^2}{8}
    \: .
    \end{align*}
    This proves the first result. Let us now prove the second one. 
    We use the upper bound on the elimination time of pieces from Theorem \ref{th:piece-elim}. Note that, for each time $t$ even, we have that $N_t/t = \omega^\star$, which implies that $R(\theta,t) = 0$. Therefore, for each $k=3,\dots,K$, $\mathbb{E}[\tau_k] \leq \bar{t}_k + 3$, where $\bar{t}_k$ is the first integer such that
    \begin{align}\label{eq:elimination_time}
    t \geq \max\left\{\frac{\left(\sqrt{\beta_{t,\delta}} + \sqrt{\beta_{t,1/t^2}}\right)^2}{H_k(\omega^\star, \theta)}, G(\theta,t) \right\},
    \end{align}
with $G(\theta,t) = 0$ for full elimination and $G(\theta,t) = \frac{ 4\beta_{t,1/t^2}}{H^\star(\theta)}$ for selective elimination.
    For all $k=3,\dots,K$, using Lemma~\ref{lem:example-dist},
    \begin{align*}
        H_k(\omega^\star, \theta)
        \ge \frac{(1 - \varepsilon)^2}{\varepsilon^2} H_2(\omega^\star, \theta)
        = \frac{(1 - \varepsilon)^2}{8}
        \ge \frac{1}{32} \: .
    \end{align*}

    Now, for full elimination, since $\beta_{t,\delta} = \log(1/\delta) + O(\log(t))$ and $H_k(\omega^\star, \theta) \geq 1/32$ as proved above, $\bar{t}_k$ is below the first time $t$ such that
\begin{align*}
    t \geq 32\log(1/\delta) + O(\log(t)),
\end{align*}
which yields $\bar{t}_k \leq \widetilde{O}(\log(1/\delta))$. The same reasoning can be done for selective elimination, which yields the stated bound.
\end{proof}

An immediate consequence of Proposition \ref{prop:elim-time-oracle} is that the total computational complexity of the oracle strategy can be significantly reduced with elimination, especially in the regime where $\delta$ and $\varepsilon$ are small.

\begin{proposition}
    On the same instance constructed in Proposition \ref{prop:elim-time-oracle}, the expected per-round computation time of the oracle strategy with LLR stopping is $\Omega(K)$, while it is at most $O(K^2\varepsilon^2)$ for full elimination and $O(K\varepsilon^2 + K/\log(1/\delta))$ for selective elimination.
\end{proposition}
\begin{proof}
    The per-round complexity of the oracle strategy with LLR stopping is $\Omega(K)$ since checking the stopping rule requires computing the infimum LLR for exactly $K-1$ alternative pieces at each step. Full elimination requires two nested loops over the set of active arms, which take at most $O(K^2)$ complexity (when no arm has been eliminated). By Proposition \ref{prop:elim-time-oracle}, this set is of size at most $K$ only until time $\widetilde{O}(\log(1/\delta))$, while it is of cardinality at most 2 after that. Since the total sample complexity is $\widetilde{O}(\log(1/\delta)/\varepsilon^2)$ (the same as LLR stopping), the total computation time is $O(K^2\log(1/\delta) + \log(1/\delta)/\varepsilon^2)$ and the average one is thus the stated quantity.
    Similarly, selective elimination requires a single loop over all active arms at each step, which takes at most $O(K)$ complexity. Again, by Proposition \ref{prop:elim-time-oracle}, this set is of size at most $K$ only until time $\widetilde{O}(\log(1/\delta) + 1/\varepsilon^2)$, while it is of cardinality at most 2 after that. Therefore, the total computation complexity is $O(K\log(1/\delta) + K/\varepsilon^2 + \log(1/\delta)/\varepsilon^2)$. Dividing by the total number of rounds $\log(1/\delta)/\varepsilon^2$ gives the stated bound.
\end{proof}

\subsection{Low information regret with elimination at stopping provably discards many arms early}

\begin{proposition}\label{prop:elim-time-lir}
    For any $K \geq 3$ and $\varepsilon \in (0,1/2)$, there exists a Gaussian linear BAI instance with unit variance and $d=2$ such that, for any $\delta \in (0,1)$, any low information regret sampling rule combined with LLR stopping satisfies
    \begin{align*}
        \mathbb{E}[\tau] \geq \Omega\left(\frac{\log(1/\delta)}{\varepsilon^2}\right).
    \end{align*}
    On the same instance, suppose we run the same strategy with elimination at stopping using a threshold $\beta_{t,\delta} = \log(1/\delta) + O(\log(t))$. Let $T_\varepsilon := \inf_{t}\{t : R(\theta,t)/t \leq \varepsilon^2/16\}$ (note that this quantity is constant in $\delta$). Then, for all but 2 arms, the expected elimination time of the corresponding piece is, for full elimination,
    \begin{align*}
        \mathbb{E}[\tau_k] \leq \widetilde{O}(\log(1/\delta) + T_\varepsilon),
    \end{align*}
    while for selective elimination,
    \begin{align*}
        \mathbb{E}[\tau_k] \leq \widetilde{O}\left(\log(1/\delta) + \frac{1}{\varepsilon^2} + T_\varepsilon\right).
    \end{align*}
    Here $\widetilde{O}$ hides constant and logarithmic terms.
\end{proposition}
\textbf{Remark}. \emph{On the instance constructed in Proposition \ref{prop:elim-time-lir}, we actually have $H^\star(\theta) = \varepsilon^2/8$ while, for any $\gamma \in (0,\varepsilon^2/16]$ and $k > 2$, $\min_{\omega\in\Omega_{\gamma}(\theta)} H_k(\omega,\theta) \geq 1/64$. This implies that, from Theorem \ref{th:piece-elim}, all but two arms are eliminated way before stopping by any low information regret sampling rule.}
\begin{proof}
    Consider the problem defined in the example of Appendix \ref{app:example}. The first result has already been proved in Proposition \ref{prop:elim-time-oracle}. Let us prove the second one. 
    
    We use the upper bound on the elimination time of pieces from Theorem \ref{th:piece-elim}. For arm $k > 2$, the denominator scales as
    \begin{align*}
        \min_{\omega\in\Omega_{R(\theta,t)/t}(\theta)} H_k(\omega,\theta).
    \end{align*}
    Suppose that $t \geq T_\varepsilon := \inf_{t}\{t : R(\theta,t)/t \leq \varepsilon^2/16\}$. Then, by definition of the set $\Omega$, any $\omega\in\Omega_{R(\theta,t)/t}(\theta)$ satisfies
    \begin{align*}
        H_2(\omega,\theta) \geq \min_{k\neq 1} H_k(\omega,\theta) \geq H^\star(\theta) - \frac{R(\theta,t)}{t} \geq \varepsilon^2/16.
    \end{align*}
    Combining this result with Lemma \ref{lem:example-dist},
    \begin{align*}
        \min_{\omega\in\Omega_{R(\theta,t)/t}(\theta)} H_k(\omega,\theta) \geq \frac{(1-\varepsilon)^2}{\varepsilon^2}\min_{\omega\in\Omega_{R(\theta,t)/t}(\theta)} H_2(\omega,\theta) \geq \frac{(1-\varepsilon)^2}{16} \geq 1/64.
    \end{align*}
    Then, from Theorem \ref{th:piece-elim} combined with the condition $t \geq T_\varepsilon$, $\mathbb{E}[\tau_k] \leq \bar{t}_k + 2 + T_\varepsilon$, where $\bar{t}_k$ is the first integer such that
    \begin{align}\label{eq:elimination_time}
    t \geq \max\left\{64\left(\sqrt{\beta_{t,\delta}} + \sqrt{\beta_{t,1/t^2}}\right)^2, G(\theta,t) \right\},
    \end{align}
with $G(\theta,t) = 0$ for full elimination and $G(\theta,t) = \frac{ 4\beta_{t,1/t^2}}{H^\star(\theta)}$ for selective elimination. We can now conclude as in the proof of Proposition \ref{prop:elim-time-oracle}.
\end{proof}


\newpage

\section{Additional details}

        \begin{algorithm}
            \small
            \begin{tabularx}{\textwidth}{*{2}{>{\centering\arraybackslash}X}}
            \begin{algorithmic}
            \vspace{-0.8em}
            \While{not stopped}
                \State \vphantom{Set ${\cP}_{t}^{\mathrm{stp}}(i^\star(\hat{\theta}_{t})) = {\cP}_{t-1}^{\mathrm{stp}}(i^\star(\hat{\theta}_{t}))$}
                \For{\textcolor{carmine}{$i \ne i^\star(\hat{\theta}_{t})$}} \Comment{stopping}
                    \State $L_{i,t} = \min_{\lambda \in \Lambda_i(i^\star(\hat{\theta}_{t}))} \Vert \hat{\theta}_t {-} \lambda \Vert_{V_{N_t}}^2$
                    \State \vphantom{\textbf{if} $L_{i,t} {>} \beta_{t,\delta}$ \textbf{then} delete $i$ from ${\cP}_{t}^{\mathrm{stp}}(i^\star(\hat{\theta}_{t}))$}
                \EndFor
                \State \textbf{if} $\min_{i \ne i^\star(\hat{\theta}_{t})} L_{i,t} > \beta_{t,\delta}$ \textbf{then} STOP
                \State Get $w_t$ from a learner $\mathcal L$
                \For{\textcolor{carmine}{$i \ne i^\star(\hat{\theta}_{t})$}} \Comment{sampling}
                    \State $\lambda_{i,t} = \argmin_{\lambda \in \Lambda_i(i_t)} \Vert \hat{\theta}_t {-} \lambda \Vert_{V_{w_t}}^2$
                \EndFor
                \State $\lambda_t = \argmin_{\textcolor{carmine}{i \ne i^\star(\hat{\theta}_{t})}} \Vert \hat{\theta}_t {-} \lambda_{i,t} \Vert_{V_{w_t}}^2$
                \State Update $\mathcal L$ with losses based on $\lambda_t$
                \State Pull arm $k_t = \argmin_k N_t^k - \sum_{s=1}^t w_s$
                \State \vphantom{Update $\cP_{t+1}^{\mathrm{smp}}(i^\star(\hat{\theta}_{t}))$ (Algorithm \ref{alg:update-pieces})}
            \EndWhile
            \vspace{-1em}
            \end{algorithmic}
            &
            \begin{algorithmic}
            \vspace{-0.8em}
            \While{not stopped}
                \State Set ${\cP}_{t}^{\mathrm{stp}}(i^\star(\hat{\theta}_{t})) = {\cP}_{t-1}^{\mathrm{stp}}(i^\star(\hat{\theta}_{t}))$
                \For{\textcolor{carmine}{$i \in {\cP}_{t-1}^{\mathrm{stp}}(i^\star(\hat{\theta}_{t}))$}} \Comment{stopping}
                    \State $L_{i,t} = \min_{\lambda \in \Lambda_i(i^\star(\hat{\theta}_{t}))} \Vert \hat{\theta}_t {-} \lambda \Vert_{V_{N_t}}^2$
                    \State \textbf{if} $L_{i,t} {>} \beta_{t,\delta}$ \textbf{delete} $i$ from ${\cP}_{t}^{\mathrm{stp}}(i^\star(\hat{\theta}_{t}))$
                \EndFor
                \State \textbf{if} ${\cP}_{t}^{\mathrm{stp}}(i^\star(\hat{\theta}_{t})) = \emptyset$ \textbf{then} STOP
                \State Get $w_t$ from a learner $\mathcal L$
                \For{\textcolor{carmine}{$i \in {\cP}_t^{\mathrm{smp}}(i^\star(\hat{\theta}_{t}))$}} \Comment{sampling}
                    \State $\lambda_{i,t} = \argmin_{\lambda \in \Lambda_i(i_t)} \Vert \hat{\theta}_t {-} \lambda \Vert_{V_{w_t}}^2$
                \EndFor
                \State $\lambda_t = \argmin_{\textcolor{carmine}{i \in {\cP}_t^{\mathrm{smp}}(i^\star(\hat{\theta}_{t}))}} \Vert \hat{\theta}_t {-} \lambda_{i,t} \Vert_{V_{w_t}}^2$
                \State Update $\mathcal L$ with losses based on $\lambda_t$
                \State Pull arm $k_t = \argmin_k N_t^k - \sum_{s=1}^t w_s$
                \State Update $\cP_{t+1}^{\mathrm{smp}}(i^\star(\hat{\theta}_{t}))$ (Algorithm \ref{alg:update-pieces})
            \EndWhile
            \vspace{-1em}
            \end{algorithmic}
            \end{tabularx}
            \caption{LinGame \cite{degenne2020gamification}: vanilla (left) and with selective elimination (right)}\label{alg:lingame}
            \end{algorithm}

    \begin{algorithm}
        \small
        \begin{algorithmic}
            \State \textbf{Input}: statistics at time $t$, answer $i$, active pieces $\cP_{t-1}^{\mathrm{smp}}(i)$
            \State \textbf{Stores}: current helper set $\tilde{\cP}_{t-1}^{\mathrm{smp}}(i)$ (initialized as $\tilde{\cP}_{0}^{\mathrm{smp}}(i) =  \cP(i)$), helper set at last reset $\tilde{\cP}_{\mathrm{last}}^{\mathrm{smp}}(i)$
            \State \textbf{Output}: updated active pieces $\cP_{t}^{\mathrm{smp}}(i)$  
            \State
            \If{$t = \bar{t}_0^{2^j}$ for some integer $j\geq 0$}
            \State Reset $\tilde{\cP}_{t}^{\mathrm{smp}}(i) = {\cP}(i)$
            and store $\tilde{\cP}_{\mathrm{last}}^{\mathrm{smp}}(i) = \tilde{\cP}_{t-1}^{\mathrm{smp}}(i)$
            \Else
            \State Set $\tilde{\cP}_{t}^{\mathrm{smp}}(i) = \tilde{\cP}_{t-1}^{\mathrm{smp}}(i)$
            \EndIf
            \For{$p \in \tilde{\cP}_{t-1}^{\mathrm{smp}}(i)$}
                \State \textbf{if} $\inf_{\lambda \in \Lambda_p(i)} L_t(\hat{\theta}_{t}, \lambda) > \alpha_{t,\delta}$ \textbf{eliminate} $p$ from $\tilde{\cP}_{t}^{\mathrm{smp}}(i)$
            \EndFor
            \State \textbf{Return} ${\cP}_t^{\mathrm{smp}}(i) = \tilde{\cP}_t^{\mathrm{smp}}(i) \cap \tilde{\cP}_{\mathrm{last}}^{\mathrm{smp}}(i)$
        \end{algorithmic}
        \caption{Update active pieces at sampling}\label{alg:update-pieces}
        \end{algorithm}